%% file: ms.tex
\documentclass[dvipsnames]{article}
\usepackage[margin=1.00in]{geometry}
\usepackage[citestyle=alphabetic, bibstyle=alphabetic, natbib=true, maxcitenames=2, maxbibnames=99, sorting=ynt]{biblatex}
\usepackage{amsmath,amssymb,amsfonts}
\usepackage{algorithmic}
\usepackage{graphicx}
\usepackage{textcomp}
\usepackage{subcaption}
\usepackage{placeins}  %
\usepackage{soul}      %

\usepackage{authblk}

\usepackage{siunitx}
\usepackage{balance}

\usepackage{titletoc}
\usepackage{longtable}

\def\externalize{1}

\newcommand{\keyword}[1]{\textit{#1}}
\usepackage{standalone}

\usepackage[shortcuts]{extdash}

\ifdefined\anonymize
  \newcommand{\revised}[1]{\colortext{OrangeRed}{#1}}
\else
  \newcommand{\revised}[1]{{#1}}
\fi

\newcommand{\colortext}[2]{{\color{#1} #2}}
\newcommand{\red}[1]{\colortext{red}{#1}}
\newcommand{\BrickRed}[1]{\colortext{BrickRed}{#1}}
\newcommand{\blue}[1]{\colortext{blue}{#1}}
\newcommand{\green}[1]{\colortext{ForestGreen}{#1}}

\ifdefined\comments
  \newcommand{\zayd}[1]{\red{ZH: #1}}
  \newcommand{\daniel}[1]{\blue{DL: #1}}
  \newcommand{\newmod}[1]{\colortext{red}{#1}}
\else
  \newcommand{\zayd}[1]{}
  \newcommand{\daniel}[1]{}
  \newcommand{\newmod}[1]{#1}
\fi

\usepackage{amsmath}

\DeclareMathOperator{\sign}{sgn}
\newcommand{\sgnp}[1]{{\sign \left( #1 \right) }}
\usepackage{amsfonts}  %
\usepackage{bbm}       %
\usepackage{mathrsfs}  %
\usepackage{relsize}   %

\usepackage{icomma}

\usepackage{mathtools} %
\usepackage{etoolbox}
\newcommand\swapifbranches[3]{#1{#3}{#2}}
\makeatletter
\MHInternalSyntaxOn
\patchcmd{\DeclarePairedDelimiter}{\@ifstar}{\swapifbranches\@ifstar}{}{}
\MHInternalSyntaxOff
\makeatother
\DeclarePairedDelimiter{\sbrack}{\lbrack}{\rbrack}

\DeclarePairedDelimiter{\floor}{\lfloor}{\rfloor}
\DeclarePairedDelimiter{\ceil}{\lceil}{\rceil}
\DeclarePairedDelimiter{\abs}{\lvert}{\rvert}

\usepackage{bm}
\DeclarePairedDelimiterX\set[1]\lbrace\rbrace{#1}
\newcommand{\setDynamic}[1]{{\left\{ #1 \right\}}}

\DeclarePairedDelimiterX\setbuild[2]\lbrace\rbrace{#1 \bm: #2}
\newcommand{\setbuildDynamic}[2]{\left\{#1 \bm: #2\right\}}

\newcommand{\setint}[1]{{\sbrack{#1}}}
\newcommand{\powerSet}[1]{2^{#1}}
\newcommand{\powerSetInt}[1]{\powerSet{\setint{#1}}}

\newcommand{\func}[3]{{#1:#2\rightarrow#3}}
\newcommand{\defeq}{\coloneqq}

\newcommand{\ind}[1]{{\mathbbm{1} \sbrack{#1}}}
\newcommand{\indSize}[2]{{\mathbbm{1} #2[ #1 #2]}}

\newcommand{\eqsmall}[1]{{\small #1}}

\newcommand{\intsPos}{\mathbb{Z}_{>0}}
\newcommand{\intsNN}{\mathbb{Z}_{+}}
\newcommand{\nats}{\mathbb{N}}
\newcommand{\real}{\mathbb{R}}

\newcommand{\bigO}[1]{{\mathcal{O}(#1)}}

\usepackage{array}  %
\usepackage{arydshln}  %
\usepackage{bigdelim}
\usepackage{booktabs}
\usepackage{multirow}
\usepackage{makecell}  %

\usepackage{amsthm}
\newtheorem{theorem}{Theorem}
\newtheorem{corollary}[theorem]{Corollary}  %
\newtheorem{definition}[theorem]{Def.}
\newtheorem{lemma}[theorem]{Lemma}
\newtheorem*{remark*}{Remark}
\newcommand{\defKeyword}[1]{{\textbf{#1}}}

\newenvironment{proofsketch}
 {\proof[Proof sketch]}
 {\endproof}

\usepackage[short,c2,nocomma]{optidef}   %

\usepackage{tikz}
\usetikzlibrary{arrows}
\usetikzlibrary{backgrounds}
\usetikzlibrary{calc}
\usetikzlibrary{decorations.markings}
\usetikzlibrary{matrix}
\usetikzlibrary{patterns}
\usetikzlibrary{positioning}
\usetikzlibrary{shadows}
\usetikzlibrary{shadows.blur}
\usetikzlibrary{shapes}
\usetikzlibrary{shapes.geometric}

\usepackage{pgfplots}
\pgfplotsset{compat=1.13}
\usepgfplotslibrary{colorbrewer}
\usepgfplotslibrary{fillbetween}
\usepgfplotslibrary{statistics}
\usepackage{pgfplotstable}
\usepackage{graphicx} %
\graphicspath{{img/}} %

\usepackage{xifthen}
\newcommand{\ifempty}[3]{%
  \ifthenelse{\isempty{#1}}{#2}{#3}%
}

\newcolumntype{H}{>{\setbox0=\hbox\bgroup}c<{\egroup}@{}}

\newcommand{\citepos}[1]{\citeauthor{#1}'s \citep{#1}}

\newcommand{\eqcomment}[1]{\triangleright~#1}
\ifdefined\anonymize
  \newcommand{\sourceCodeUrl}{https://github.com/AnonSaTML/certified_regression}
\else
  \newcommand{\sourceCodeUrl}{https://github.com/ZaydH/certified-regression}
\fi

\newcommand{\predSym}{q}

\newcommand{\certBound}{R}
\newcommand{\greedyBound}{G}

\newcommand{\swapBound}{\Delta}

\newcommand{\threshold}{\xi}

\newcommand{\lowSub}[1]{#1_{\textnormal{l}}}
\newcommand{\lThreshold}{\lowSub{\threshold}}

\newcommand{\upperSub}[1]{#1_{\textnormal{u}}}
\newcommand{\uThreshold}{\upperSub{\threshold}}

\newcommand{\harmonic}[1]{{H(#1)}}

\newcommand{\median}{\textnormal{med}}
\newcommand{\medFunc}[1]{{\median \,#1}}

\newcommand{\bigOn}{$\bigO{\nTr}$}
\newcommand{\bigOnPlT}{$\bigO{\nTr + \nModel}$}

\newcommand{\bigOT}{$\bigO{\nModel}$}

\newcommand{\setScalarSym}{\nu}
\newcommand{\setScalarI}[1][\modIdx]{\modSub[#1]{\setScalarSym}}
\newcommand{\setScalarOne}{\setScalarI[1]}
\newcommand{\setScalarFin}{\setScalarI[\nModel]}

\newcommand{\setVals}{\mathcal{V}}
\newcommand{\setValsMod}{\widetilde{\setVals}}
\newcommand{\setValsZO}{\setVals_{{\pm}1}}

\newcommand{\lowerSetVals}{\lowSub{\setVals}}
\newcommand{\nLower}{\abs{\lowerSetVals}}
\newcommand{\upperSetValsTransform}[1]{\widetilde{#1}}
\newcommand{\upperSetValsPerturb}{\upperSub{\upperSetValsTransform{\setVals}}}
\newcommand{\upperSetVals}{\upperSub{\setVals}}

\newcommand{\covSetSym}{\mathcal{R}}
\newcommand{\lowerCovVals}{\lowSub{\covSetSym}}
\newcommand{\smallestCovValsMinusOne}{\covSetSym_{\!\swapBound}}
\newcommand{\smallestCovVals}{\lowSub{\widetilde{\covSetSym}}}

\newcommand{\trStr}{\text{tr}}
\newcommand{\teStr}{\text{te}}

\newcommand{\nTr}{n}
\newcommand{\trainSet}{\mathcal{S}}
\newcommand{\trainSetPerturb}{\widetilde{\trainSet}}

\newcommand{\blockIdx}{j}
\newcommand{\trSub}{S}
\newcommand{\dsSuper}[2][\blockIdx]{#2^{(#1)}}
\newcommand{\blockI}[1][\blockIdx]{\dsSuper[#1]{\trSub}}
\newcommand{\blockBoldI}[1]{\mathbf{\trSub}^{(\mathbf{#1})}}
\newcommand{\blockOne}{\blockI[1]}
\newcommand{\blockT}{\blockI[\nModel]}
\newcommand{\blockFin}{\blockI[\nBlocks]}

\newcommand{\nBlocks}{m}

\newcommand{\trIdx}{i}

\newcommand{\X}{x}
\newcommand{\xI}{\X_{\trIdx}}
\newcommand{\domainX}{\mathcal{X}}
\newcommand{\dimX}{d}
\newcommand{\xTe}{\X_{\teStr}}

\newcommand{\y}{y}
\newcommand{\yI}{\y_{\trIdx}}
\newcommand{\domainY}{\mathcal{Y}}
\newcommand{\yTe}{\y_{\teStr}}

\newcommand{\z}{z}
\newcommand{\domainZ}{\mathcal{Z}}

\newcommand{\hashSym}{h}
\newcommand{\hashTrain}{\hashSym_{\trStr}}
\newcommand{\hashTrainFunc}[1]{{\hashTrain(#1)}}

\newcommand{\hashModel}{\hashSym_{\dec}}
\newcommand{\hashModelFunc}[1]{{\hashModel(#1)}}

\newcommand{\dec}{f}
\newcommand{\modSub}[2][\modIdx]{#2_{#1}}
\newcommand{\decI}[1][\modIdx]{\modSub[#1]{\dec}}
\newcommand{\decOne}{\modSub[1]{\dec}}
\newcommand{\decFin}{\modSub[\nModel]{\dec}}

\newcommand{\decSubset}[1]{\decI[#1]}
\newcommand{\decSubsetFunc}[2]{\decFunc[\decSubset{#1}]{#2}}

\newcommand{\decFunc}[2][\dec]{{#1(#2)}}
\newcommand{\decTe}{\decFunc{\xTe}}
\newcommand{\decFuncI}[2][\modIdx]{{\decI[#1](#2)}}
\newcommand{\decFuncOne}[1]{{\decFuncI[1]{#1}}}
\newcommand{\decFuncFin}[1]{{\decFuncI[\nModel]{#1}}}

\newcommand{\nModel}{T}
\newcommand{\lowerSubMods}{\lowSub{\mathcal{T}}}

\newcommand{\modIdx}{t}
\newcommand{\modIdxAlt}{\modIdx'}

\newcommand{\modTrainSetSym}{\mathcal{D}}
\newcommand{\modTrainSetI}[1][\modIdx]{\modSub[#1]{\modTrainSetSym}}
\newcommand{\modTrainSetOne}{\modTrainSetI[1]}

\newcommand{\modTrainSetIAlt}{\modSub[\modIdxAlt]{\modTrainSetSym}}

\newcommand{\halfModels}{\ceil{\frac{\nModel}{2}}}

\newcommand{\covSym}{r}
\newcommand{\covI}[1][\modIdx]{\modSub[#1]{\covSym}}
\newcommand{\covOne}{\covI[1]}
\newcommand{\covFin}{\covI[\nModel]}

\newcommand{\covMax}{\covI[\max]}

\newcommand{\covSet}{\covSetSym}

\newcommand{\kNeigh}{k}
\newcommand{\knn}{$\kNeigh$\textsc{NN}}
\newcommand{\knnMSuffix}{\=/m}
\newcommand{\knnM}{\knn{}\knnMSuffix}

\newcommand{\rnn}{$r$NN}

\newcommand{\neighSym}{\mathcal{N}}
\newcommand{\neighborhood}[1]{{\neighSym(#1)}}
\newcommand{\neighTe}{\neighborhood{\xTe}}

\newcommand{\lowNeigh}{\lowerSetVals}

\newcommand{\modelVarSym}{\delta}
\newcommand{\modelVarI}{\modSub{\modelVarSym}}

\newcommand{\dsBlockVarSym}{\omega}
\newcommand{\dsBlockVarI}{\dsSuper{\dsBlockVarSym}}

\newcommand{\isMultiVar}{\sigma}

\newcommand{\datasetDiv}{q}

\newcommand{\spreadDegreeSym}{d}
\newcommand{\spreadDegI}[1][\blockIdx]{\dsSuper[#1]{\spreadDegreeSym}}
\newcommand{\spreadDegOne}{\spreadDegI[1]}
\newcommand{\spreadDegFin}{\spreadDegI[\nBlocks]}

\newcommand{\spreadDegMax}{\spreadDegreeSym_{\max}}

\newcommand{\disjointSingleMethod}{\textsc{PCR}}
\newcommand{\overlapSingleMethod}{\textsc{OCR}}

\newcommand{\bothDisjoint}{(\multiPrefix)\disjointSingleMethod}
\newcommand{\bothOverlap}{(\multiPrefix)\overlapSingleMethod}

\newcommand{\multiPrefix}{\textsc{W}\=/}
\newcommand{\disjointMultiMethod}{\multiPrefix\disjointSingleMethod}
\newcommand{\overlapMultiMethod}{\multiPrefix\overlapSingleMethod}

\newcommand{\knnMethod}{\knn{}\=/\textsc{CR}}

\newcommand{\xgb}{XGBoost}

\newcommand{\medIdx}{\alpha}
\newcommand{\thresholdIdx}{\nLower}

\newcommand{\groundSet}{U}
\newcommand{\setCollection}{\mathbf{Q}}
\newcommand{\subsetSym}{\mathcal{Q}}
\newcommand{\nSubset}{\nBlocks}
\newcommand{\subsetI}[1][\blockIdx]{\subsetSym_{#1}}
\newcommand{\subsetOne}{\subsetI[1]}
\newcommand{\subsetFin}{\subsetI[\nSubset]}

\newcommand{\subcover}{\mathcal{F}}

\newcommand{\hyperRidgeWD}{\lambda}
\newcommand{\hyperRidgeTol}{\varepsilon}
\newcommand{\hyperRidgeItr}{\text{\#~Itr.}}

\newcommand{\hyperXgbNumEst}{\tau}
\newcommand{\hyperXgbMaxDepth}{h}
\newcommand{\hyperXgbMetric}{\mathcal{L}}
\newcommand{\hyperXgbWD}{\lambda}
\newcommand{\hyperXgbLR}{\eta}
\newcommand{\hyperXgbMinLoss}{\gamma}

\newcommand{\SupplementaryMaterialsTitle}{%
  \vbox{
    \hrule height 4pt
    ~\vspace{0.10in}

    \vskip -\parskip%
    \begin{center}
      {\Large\bf \titleTextBreak{} \par}

      \vspace{8pt}
      {\Large Supplemental Materials \par}
    \end{center}
    \vskip 0.15in
    \vskip -\parskip
    \hrule height 1pt
    \vskip 0.09in%
  }
}

\newcommand{\equationSize}{\small}

\newcommand{\tableSize}{\small}
\usepackage{graphicx}
\newsavebox\CBox
\def\textBF#1{\sbox\CBox{#1}\resizebox{\wd\CBox}{\ht\CBox}{\textbf{#1}}}

\providecommand{\keywords}[1]
{
  \small
  \noindent
  \textbf{\textit{Keywords}}:~#1
}

\newcommand\blfootnote[1]{%
  \begingroup
  \renewcommand\thefootnote{}\footnote{#1}%
  \addtocounter{footnote}{-1}%
  \endgroup
}

\ifdefined\externalize%
  \usepgfplotslibrary{external}%
  \newcommand{\tikzExternalDir}{build}%
\fi%

\newcommand{\pgfPlotsLegendRef}[1]{%
  \ifdefined\externalize%
    \scalebox{0.5}[1.0]{%
      \tikzexternaldisable#1\tikzexternalenable
    }%
  \else%
    \raisebox{0.5ex}%
             {%
               \scalebox{0.5}[1.0]{%
                 \begin{tikzpicture}%
                   #1%
                 \end{tikzpicture}%
               }%
             }%
  \fi%
}

\newcommand{\setValOne}{2}
\newcommand{\setValTwo}{3}
\newcommand{\setValThree}{4}
\newcommand{\setValFour}{5}
\newcommand{\setValFive}{6}

\newcommand{\setCovOne}{3}
\newcommand{\setCovTwo}{4}
\newcommand{\setCovThree}{5}
\newcommand{\setCovFour}{6}
\newcommand{\setCovFive}{7}

\newcommand{\covValDiffVal}{\number\numexpr\setCovOne-\numexpr\setValOne}
\newcommand{\covValDiff}{%
  \ifnumless{\covValDiffVal}{0}%
            {- \covValDiffVal}%
            {+ \covValDiffVal}%
}

\newcommand{\setCovBaseBound}{\setCovOne}
\newcommand{\setCovFullBound}{{\number\numexpr\setCovOne+\numexpr\setCovTwo-1\relax}}

\newcommand{\setValThreshold}{5.4}

\usepackage{etoolbox}
\newcommand{\decInt}[1]{{\number\numexpr#1-1\relax}}

\newcommand{\setCovTwoMOne}{\decInt{\setCovTwo}}

\newcommand{\medianPerturbFontStyle}{\scriptsize\bfseries}
\newcommand{\medianPerturbNodeMinSize}{0.45cm}
\newcommand{\medianPerturbNodeDist}{1.150cm}
\newcommand{\medianLabelDist}{0.65cm}

\newcommand{\belowMedianFill}{blue!05!white}
\newcommand{\aboveMedianFill}{red!10!white}
\newcommand{\perturbFill}{green!10!white}

\tikzset{
    vertex/.style = {%
      circle,
      draw,
      node distance=\medianPerturbNodeDist,
      minimum size=\medianPerturbNodeMinSize,
      inner sep=0pt,  %
      drop shadow,
      font=\medianPerturbFontStyle,
    },
    below median/.style = {%
      blue,
      text=black,
      fill=\belowMedianFill,
    },
    above median/.style = {%
      red,
      text=black,
      fill=\aboveMedianFill,
    },
    square/.style = {regular polygon,regular polygon sides=4},
    added/.style = {%
      square,
      minimum size=\medianPerturbNodeMinSize,
      inner sep=0pt,
      text=black,
      fill=\perturbFill,
    },
    empty vertex/.style = {%
      circle,
      fill=white,
      node distance=\medianPerturbNodeDist,
      minimum size=\medianPerturbNodeMinSize,
      text=white,
      font=\medianPerturbFontStyle,
    },
    median label/.style = {%
      node distance = \medianLabelDist,
      font=\scriptsize,
      inner ysep=2pt,
    },
    label line/.style = {%
      black!75,
      thick,
    },
    vals label/.style = {%
      node distance={\valsLabelNodeDist},
      font={\valsLabelFontSize},
    },
    hide text/.style = {%
      text opacity=0,
    },
}

\newcommand{\valsLabelNodeDist}{0.7cm}
\newcommand{\valsLabelFontSize}{\scriptsize}
\newcommand{\ValLabel}[1]{#1:}

\tikzset{%
  dot hidden/.style={},
  line hidden/.style={},
  dot color/.style={dot hidden/.append style={color=#1}},
  dot color/.default=black,
  dot perturb/.style={draw=black},
  dot unperturb/.style={draw=black,fill=white},
  line color/.style={line hidden/.append style={color=#1}},
  line color/.default=black,
}%

\newcommand{\DiceTopY}{0.20}
\newcommand{\DiceMidY}{0.50}
\newcommand{\DiceBotY}{0.80}

\newcommand{\dieDim}{0.5cm}
\newcommand{\dotRadius}{0.100}

\usepackage{xparse}
\NewDocumentCommand{\drawdie}{O{}m}{%
    \begin{tikzpicture}[
        x=\dieDim,
        y=\dieDim,
        radius=\dotRadius,
        #1
      ]
      \draw[%
        rounded corners=1,
        line hidden,
        drop shadow,
      ] (0,0) rectangle (1,1);
      \ifnum#2=2
        \fill[dot perturb] (\DiceTopY,\DiceTopY) circle;%
        \fill[dot perturb] (\DiceBotY,\DiceBotY) circle;%
      \fi
      \ifnum#2=3
        \fill[dot perturb] (\DiceMidY,\DiceMidY) circle;%
        \fill[dot perturb] (\DiceTopY,\DiceTopY) circle;%
        \fill[dot unperturb] (\DiceBotY,\DiceBotY) circle;%
      \fi
      \ifnum#2>3
        \fill[dot unperturb] (\DiceTopY,\DiceTopY) circle;%
        \fill[dot unperturb] (\DiceBotY,\DiceBotY) circle;%
        \fill[dot unperturb] (\DiceTopY,\DiceBotY) circle;%
        \fill[dot unperturb] (\DiceBotY,\DiceTopY) circle;%
      \fi
      \ifnum#2=5
        \fill[dot unperturb] (\DiceMidY,\DiceMidY) circle;%
      \fi
      \ifnum#2=6
        \fill[dot unperturb] (\DiceBotY,\DiceMidY) circle;%
        \fill[dot unperturb] (\DiceTopY,\DiceMidY) circle;%
      \fi
    \end{tikzpicture}%
}%

\newcommand{\knnTikzScale}{1.00}

\newcommand{\knnTeLabelSize}[1]{{\scriptsize #1}}
\newcommand{\knnValueSize}[1]{{\scriptsize #1}}

\tikzset{%
    cmark/.style={%
      append after command={plot[only marks,mark=#1] coordinates {(\tikzlastnode)}}%
    },
    knn vertex/.style = {%
    },
    xTe vertex/.style = {%
      knn vertex,
      cmark={*,mark options={black,fill=black!50!white}},
    },
    below med knn/.style = {%
      knn vertex,
      cmark={*,mark options={blue,fill=blue!18!white}},
    },
    below med knn inline/.style = {%
      knn vertex,
      cmark={*,mark options={blue,fill=blue!18!white},mark size=2.5pt},
    },
    above med knn/.style = {%
      knn vertex,
      cmark={*,mark options={red,fill=red!18!white}},
    },
    other knn/.style = {%
      knn vertex,
      cmark={*,mark options={black,fill=white}},
    },
    added knn/.style = {%
      knn vertex,
      cmark={square*,mark options={ForestGreen,fill=green!20!white}},
    },
    added knn inline/.style = {%
      knn vertex,
      cmark={square*,mark options={ForestGreen,fill=green!20!white},mark size=2.5pt},
    },
    knn base radius/.style = {%
      black!40!white,
      dashed,
      thick,
    }
}

\DeclareRobustCommand\showAddedKnnSquare{\tikz[baseline=-0.60ex]\node[added knn inline,mark size=0.7ex]{};}
\DeclareRobustCommand\showBelowMedianKnn{\tikz[baseline=-0.60ex]\node[below med knn inline,mark size=0.7ex]{};}

\pgfdeclarelayer{back}
\pgfsetlayers{back,main}
\pgfkeys{%
  /tikz/on layer/.code={
    \pgfonlayer{#1}\begingroup
    \aftergroup\endpgfonlayer
    \aftergroup\endgroup
  },
  /tikz/node on layer/.code={
    \pgfonlayer{#1}\begingroup
    \expandafter\def\expandafter\tikz@node@finish\expandafter{\expandafter\endgroup\expandafter\endpgfonlayer\tikz@node@finish}%
  },
}

\definecolor{lavenderblue}{rgb}{0.8, 0.8, 1.0}
\definecolor{wisteria}{rgb}{0.79, 0.63, 0.86}
\definecolor{wildwatermelon}{rgb}{0.99, 0.42, 0.52}
\definecolor{salmon}{rgb}{1.0, 0.55, 0.41}
\definecolor{lightkhaki}{rgb}{0.94, 0.9, 0.55}
\definecolor{pistachio}{rgb}{0.58, 0.77, 0.45}
\definecolor{turquoise}{rgb}{0.19, 0.84, 0.78}

\definecolor{bluegray}{rgb}{0.40, 0.60, 0.80}
\definecolor{tuftyblue}{rgb}{0.25, 0.49, 0.76}
\definecolor{ashgrey}{rgb}{0.70, 0.75, 0.71}
\definecolor{persianorange}{rgb}{0.85, 0.56, 0.53}
\definecolor{junebud}{rgb}{0.74, 0.85, 0.34}
\definecolor{asparagus}{rgb}{0.53, 0.66, 0.42}
\definecolor{citrine}{rgb}{0.89, 0.82, 0.04}

\definecolor{dsPart1Color}{rgb}{0.76, 0.7, 0.5}
\definecolor{dsPart2Color}{rgb}{0.6, 0.73, 0.45}
\definecolor{dsPart3Color}{rgb}{0.29, 0.59, 0.82}
\definecolor{dsPart4Color}{rgb}{0.8, 0.6, 0.8}
\definecolor{dsPart5Color}{rgb}{0.94, 0.5, 0.5}
\definecolor{dsPart6Color}{rgb}{0.98, 0.81, 0.69}
\definecolor{dsPart7Color}{rgb}{0.98, 0.85, 0.37}

\tikzset{%
  dataset block base/.style = {%
    cylinder,
    shape border rotate=90,
    draw,
    aspect=0.22,
    minimum height=0.0cm,
    minimum width=1.5cm,
    inner sep=2.75pt,
    drop shadow={gray!40, opacity=1.0, on layer=back, shadow xshift=0.5ex, shadow yshift=-0.5ex, shadow scale=1.03},
  },
  block01/.style = {%
    dataset block base,
    left color=dsPart1Color!80!black,
    right color=dsPart1Color!80!white,
    middle color=dsPart1Color, %
  },%
  block02/.style = {%
    dataset block base,
    left color=dsPart2Color!80!black,
    right color=dsPart2Color!80!white,
    middle color=dsPart2Color, %
  },
  block03/.style = {%
    dataset block base,
    left color=dsPart3Color!80!black,
    right color=dsPart3Color!80!white,
    middle color=dsPart3Color, %
  },
  block04/.style = {%
    dataset block base,
    left color=dsPart4Color!80!black,
    right color=dsPart4Color!80!white,
    middle color=dsPart4Color, %
  },
  block05/.style = {%
    dataset block base,
    left color=dsPart5Color!80!black,
    right color=dsPart5Color!80!white,
    middle color=dsPart5Color, %
  },
  block06/.style = {%
    dataset block base,
    left color=dsPart6Color!80!black,
    right color=dsPart6Color!80!white,
    middle color=dsPart6Color, %
  },
  block07/.style = {%
    dataset block base,
    left color=dsPart7Color!80!black,
    right color=dsPart7Color!80!white,
    middle color=dsPart7Color, %
  },
  dataset line/.style = {%
    ->,
    thick
  },%
  submodel/.style = {%
    drop shadow,
    rounded corners,
    square,
    draw,
    line width=0.3mm,
    inner sep=0.130cm,
    text=black,
    fill=white,
  },%
  prediction base/.style = {%
    circle,
    draw,
    drop shadow,
  },%
  prediction below/.style = {%
    prediction base,
    below median,
  },%
  prediction above/.style = {%
    prediction base,
    above median,
  },%
  prediction line/.style = {%
    ->,
    thick
  },%
  description box/.style = {%
    rectangle,
    minimum width=1in,
    minimum height=1in,
    draw,
    rounded corners,
    gray,
    dashed,
    thick,
    on layer=back,
  },%
}

\newcommand{\boundLineWidth}{0.5pt}
\newcommand{\boundLineWidthLegend}{\boundLineWidth}

\newcommand{\boundFontSize}{\scriptsize}
\newcommand{\boundYTickFontSize}{\scriptsize}

\tikzset{%
  line fill base/.style = {%
    draw=none,
  },%
  disjoint line fill/.style = {
    line fill base,
    fill=blue!20,
  },
  overlap line fill/.style = {
    line fill base,
    fill=red!10,
  },
  overlap base line/.style = {%
    line width=\boundLineWidth,
  },
  single disjoint line/.style = {
    overlap base line,
    blue,
    dashed,
  },
  single overlap line/.style = {
    overlap base line,
    red,
    dashed,
  },
  multi disjoint line/.style = {
    overlap base line,
    blue,
  },
  multi overlap line/.style = {
    overlap base line,
    red,
  },
  knn trend line/.style = {
    overlap base line,
    dashdotted,
    ForestGreen,
    line width=1.5*\boundLineWidth,
  },
  knn line fill/.style = {
    line fill base,
    fill=ForestGreen!25,
    opacity=0.40,
  },
  baseline trend line/.style = {
    overlap base line,
    dotted,
    line cap=round,
    black,
    line width=\boundLineWidth,
  },
}

\tikzset{%
  knn alt-main trend line/.style = {
    overlap base line,
    red,
    line width=1.5*\boundLineWidth,
  },%
  knn alt-main line fill/.style = {
    overlap base line,
    dash pattern=on 6pt off 1pt,
    red!30,
    line width=1.5*\boundLineWidth,
  },%
  knn alt01 trend line/.style = {
    overlap base line,
    dash pattern=on 6pt off 1pt,
    Goldenrod!90!gray,
    line width=\boundLineWidth,
  },%
  knn alt01 line fill/.style = {
    line fill base,
    fill=Goldenrod!80,
    opacity=0.40,
  },%
  knn alt02 trend line/.style = {
    overlap base line,
    dash pattern=on 5pt off 2pt,
    Aquamarine,
    line width=\boundLineWidth,
  },%
  knn alt02 line fill/.style = {
    line fill base,
    fill=Aquamarine!35,
    opacity=0.40,
  },%
  knn alt03 trend line/.style = {
    overlap base line,
    dash pattern=on 4pt off 3pt,
    RedOrange,
    line width=1.5*\boundLineWidth,
  },%
  knn alt03 line fill/.style = {
    line fill base,
    fill=RedOrange!25,
    opacity=0.40,
  },%
  knn alt04 trend line/.style = {
    overlap base line,
    dash pattern=on 3pt off 4pt,
    BlueViolet,
    line width=1.5*\boundLineWidth,
  },%
  knn alt04 line fill/.style = {
    line fill base,
    fill=BlueViolet!25,
    opacity=0.40,
  },%
  knn alt05 trend line/.style = {
    overlap base line,
    dash pattern=on 2pt off 5pt,
    Gray,
    line width=1.5*\boundLineWidth,
  },%
  knn alt05 line fill/.style = {
    line fill base,
    fill=Gray!35,
    opacity=0.40,
  },%
  knn alt06 trend line/.style = {
    overlap base line,
    dotted,
    VioletRed,
    line width=1.5*\boundLineWidth,
  },%
  knn alt06 line fill/.style = {
    line fill base,
    fill=VioletRed!35,
    opacity=0.40,
  },%
}

\newcommand{\ReductionSymbolFontSize}{\small}

\newcommand{\VarTypeDescriptorSpacer}{0.80cm}

\newcommand{\BranchNodeFrac}{0.75}

\newcommand{\InputVarMainSpacer}{5.85cm}
\newcommand{\MedBoxMainSpacer}{1.45cm}
\newcommand{\DecFuncMainSpacer}{2.30cm}

\newcommand{\SignumBoxSpacer}{1.80cm}

\newcommand{\CertifiedReductLineColor}{black!85}

\tikzset{%
  certified box/.style = {%
      drop shadow,
      rounded corners,
      rectangle,
      draw,
      line width=0.3mm,
      inner sep=0.2cm,
      text=black,
      fill=white,
      align=center,
      font={\ReductionSymbolFontSize},
  },
  signum box/.style = {%
      rectangle,
      draw,
      line width=0.3mm,
      inner sep=0.15cm,
      text=black,
      fill=white,
      align=center,
      node distance={\SignumBoxSpacer},
      rounded corners=1,
      drop shadow,
  },
  certified reduction branch node/.style = {%
      circle,
      \CertifiedReductLineColor,
      fill=\CertifiedReductLineColor,
      inner sep=0cm,
      minimum size=0.15cm,
      font={\ReductionSymbolFontSize},
  },
  certified reduction var descriptor/.style = {%
    node distance={\VarTypeDescriptorSpacer},
    font={\ReductionSymbolFontSize},
  },
  certified reduction input label/.style = {%
      font={\ReductionSymbolFontSize},
  },
  certified reduction output label/.style = {%
      text width={0.95cm},
      anchor=west,
      font={\ReductionSymbolFontSize},
  },
  certified reduction subtraction/.style = {%
      circle,
      draw,
      inner sep=01pt,  %
      font={\bfseries},
      drop shadow,
      fill=white,
  },
  certified reduction surround box/.style = {%
      draw,
      rounded corners=2,
      line width=1.05pt,
      black!45,
  },%
  white label/.style = {
      fill=white,
      inner sep=0.5mm,
      text=black,
      font={\ReductionSymbolFontSize},
  },
  certified reduct line/.style = {%
      \CertifiedReductLineColor,
      thick,
  },%
}%

\newcommand{\CovLabelDist}{0.6cm}
\newcommand{\CovBaseWeightDist}{0.3cm}

\tikzset{%
  vertex cost/.style = {%
      node distance={\CovLabelDist},
      font={\medianPerturbFontStyle},
  },
  cov label/.style = {
      vals label,
      node distance={\CovLabelDist},
  },
  base weight text/.style = {
      vals label,
      node distance={\CovBaseWeightDist},
  },
}%

\usepackage[colorlinks=false]{hyperref}
\newcommand{\titleText}{Reducing Certified Regression to Certified Classification for General Poisoning Attacks}
\newcommand{\titleTextNoBreak}{\titleText}
\newcommand{\titleTextBreak}{Reducing Certified Regression to Certified Classification \\ for General Poisoning Attacks}
\newcommand{\pdfKeywords}{%
  Robust regression,
  Certified classifier,
  Data poisoning,
  Partial set cover,
  Partial set multicover
}

\author{Zayd Hammoudeh\footnote{Correspondence to \href{mailto:zayd@cs.uoregon.edu}{zayd@cs.uoregon.edu}.}}
\author{Daniel Lowd}
\affil{University of Oregon}
\date{}

\ifdefined\anonymize
\else
  \RequirePackage{graphicx, xcolor}
  \definecolor[named]{ACMBlue}{cmyk}{1,0.1,0,0.1}
  \definecolor[named]{ACMYellow}{cmyk}{0,0.16,1,0}
  \definecolor[named]{ACMOrange}{cmyk}{0,0.42,1,0.01}
  \definecolor[named]{ACMRed}{cmyk}{0,0.90,0.86,0}
  \definecolor[named]{ACMLightBlue}{cmyk}{0.49,0.01,0,0}
  \definecolor[named]{ACMGreen}{cmyk}{0.20,0,1,0.19}
  \definecolor[named]{ACMPurple}{cmyk}{0.55,1,0,0.15}
  \definecolor[named]{ACMDarkBlue}{cmyk}{1,0.58,0,0.21}
  \hypersetup{colorlinks,
    linkcolor=ACMRed,
    citecolor=ACMPurple,
    urlcolor=ACMDarkBlue,
    filecolor=ACMDarkBlue,
    pdftitle={\titleTextNoBreak},
    pdfauthor={Zayd Hammoudeh and Daniel Lowd},
    pdfkeywords={\pdfKeywords},
  }
\fi

\makeatletter%
\begin{document}
\title{\textbf{\titleText}%
}

\maketitle

\begin{abstract}
Adversarial training instances can severely distort a model's behavior.
This work investigates
\keyword{certified regression defenses}, which provide guaranteed limits on how much a regressor's prediction may change under a poisoning attack.
Our key insight is that certified regression reduces to voting-based certified classification when using median as a model's primary decision function. %
Coupling our reduction with existing certified classifiers, we propose six new regressors provably-robust to poisoning attacks.
To the extent of our knowledge, this is the first work that certifies the robustness of individual regression predictions without any assumptions about the data distribution and model architecture.
We also show that the assumptions made by existing state-of-the-art certified classifiers are often overly pessimistic.
We introduce a tighter analysis of model robustness, which in many cases results in significantly improved certified guarantees.
Lastly, we empirically demonstrate our approaches' effectiveness on both regression and classification data, where the accuracy of up to 50\% of test predictions can be guaranteed under 1\%~training set corruption and up to 30\% of predictions under 4\%~corruption.
Our source code is available at
\mbox{\url{\sourceCodeUrl}}.%
\end{abstract}


\blfootnote{%
  This paper appeared at the
  first IEEE Conference on Secure and Trustworthy Machine Learning (SaTML).
  The definitive, peer-reviewed version is published in the proceedings of {SatML'23}~\citep{Hammoudeh:2023:CertifiedRegression}.%
}

\section{Introduction}%
\label{sec:Intro}

In a \keyword{poisoning attack}, an adversary inserts malicious instances into a model's training set to manipulate one or more target predictions~\citep{Biggio:2012:Poisoning,Chen:2017:Targeted,Shafahi:2018:PoisonFrogs,Huang:2020:MetaPoison,Wallace:2021}.
\citepos{Kumar:2020} recent survey of large corporate and governmental organizations found that poisoning attacks were their top ML security concern due to previous successful attacks~\citep{Lee:2016:Learning}.
\citeauthor{Kumar:2020} specifically note that most defenses against these attacks lack ``fundamental security rigor'' and acknowledge that most adversarial ML defenses are like ``crypto pre-Shannon''~\citep{Carlini:2019:Youtube}.

\citeauthor{Kumar:2020}'s concerns primarily arise because most ML defenses (including those for poisoning attacks) are \textit{empirical}~\citep{Tran:2018,Peri:2020:DeepKNN,Zhu:2021:CLEAR,Hammoudeh:2022:GAS};
such defenses derive from insights into the underlying mechanisms of specific attacks and in turn provide strategies to mitigate the associated vulnerability.
The fatal weakness of empirical defenses is that they provide no guarantees of their effectiveness, and attacks can be adapted to bypass them -- often with minimal effort~\citep{Gao:2019:STRIP,Kumar:2020}.

In contrast, \textit{certified defenses}~\citep{Lai:2016:AgnosticMeanCovariance,Steinhardt:2017,Wang:2020,Weber:2023:RAB} provide a quantifiable guarantee of a prediction's robustness, albeit under specific assumptions.
There has been significant recent progress towards lifting the assumptions necessary to certify a classifier's prediction.
Today, non-trivial guarantees of the \keyword{pointwise robustness} of individual classification predictions are possible without making assumptions about the underlying data distribution or model architecture.

Similar progress has not yet been made for regression.
Existing certified regressors generally make strong assumptions about the \textit{data distribution} (e.g., linearity~\citep{Liu:2020:RobustSparseRegression}, sparsity~\citep{Liu:2017:RobustLinearRegression}) that rarely hold in practice. %
When those assumptions fail to hold, these methods' ``guarantees'' are not guarantees at all.
Another common requirement of existing certified regressors is that the \textit{model architecture} is linear~\citep{Jagielski:2018:Trim}, despite other architectures often performing far better~\citep{Chen:2016:XGBoost,Prokhorenkova:2018:CatBoost}.

\newcommand{\probA}{Q}
\newcommand{\probB}{\probA'}
  A problem~$\probA$ is \keyword{reducible} to a different problem~$\probB$ if an efficient algorithm to solve~$\probB$ can also efficiently solve $\probA$~\citep{Dasgupta:2008:Algorithms}.
Our \textit{key insight} is that certified regression is reducible to voting-based certified classification.
Mapping certified regression to certified classification requires only minimal changes to the certified classifier's architecture, with the robustness certification function identical.
Given reducibility,
an important takeaway is that certified regression can be viewed as \ul{no harder} than certified classification.%

Coupling our reduction with existing certified classifiers~\citep{Jia:2022:CertifiedKNN,Levine:2021:DPA,Wang:2022:DeterministicAggregation}, we propose six new certifiably-robust regressors.
To the extent of our knowledge, our methods are the first to provide pointwise regression robustness guarantees against poisoning without both distributional and model assumptions.

Our primary contributions are enumerated below.
\begin{enumerate}
  \item We formalize three paradigms based on median perturbation to map certified regression to certified classification.  All of our certified regressors apply one of these paradigms.
  \item We propose two provably-robust instance-based regressors -- one based on $\kNeigh$\=/nearest neighbors and %
  the other based on all training instances within a feature-space region.
  \item We separately propose four ensemble-based certified regressors, where one pair of regressors trains submodels on disjoint data while the other pair allows submodels to be trained on overlapping data.
  \item We significantly improve the certification performance of our ensemble-based regressors \textit{and} existing certified classifiers via a tighter analysis of submodel prediction stability.
  \item We demonstrate our methods' effectiveness on both regression and classification datasets, where we certify significant fractions of the training set and even outperform state-of-the-art certified classifiers on binary classification.
\end{enumerate}
Note that all proofs are in the supplemental materials.
\newcommand{\prelimParagraph}[1]{%
  \vspace{5pt}%
  \noindent%
  \textbf{#1}~
}

\section{Preliminaries}%
\label{sec:Preliminaries}

We begin with a summary of our notation followed by a formalization of our threat model and objective.%
\footnote{%
  \revised{%
  Supplemental Sec.~\ref{sec:App:Nomenclature} provides a full nomenclature reference.%
  }%
}

\prelimParagraph{Notation}
Let \eqsmall{$\setint{k}$} denote the set of integers \eqsmall{$\{1, \ldots, k\}$}, and denote the corresponding \keyword{power set}~\eqsmall{$\powerSetInt{k}$}.
\eqsmall{$\ind{\predSym}$}~is the \keyword{indicator function}, which equals~1 if predicate~$\predSym$ is true and 0~otherwise.
Let \eqsmall{${\harmonic{k} = \sum_{i=1}^{k} \frac{1}{i}}$} denote the \keyword{$k$\=/th harmonic number}.

In this work, the term ``set'' refers to a \keyword{multiset} where multiple elements may have the same value.
Denote set $A$'s \keyword{median} $\medFunc{A}$.
In cases where $A$'s cardinality is even, the median is the midpoint between $A$'s
\eqsmall{$\frac{\abs{A}}{2}$}\=/th
and
\eqsmall{${(\frac{\abs{A}}{2} + 1)}$}\=/th
largest values.
Finding the median requires only linear time,
meaning median is asymptotically as fast as mean~\citep{Blum:1973:TimeBoundsSelection}.

${\X \in \domainX}$ is an \keyword{independent variable} (e.g.,~\keyword{feature vector}) of dimension~$\dimX$ and
{${\y \in \domainY \subseteq \real}$} a \keyword{response variable} (e.g.,~\keyword{target}).
Let \eqsmall{${\domainZ \defeq \domainX \times \domainY}$} denote the \keyword{instance space}.
\keyword{Training set}~$\trainSet$ consists of $\nTr$~training examples. %
For \eqsmall{${\nBlocks \in \nats}$} where \eqsmall{${\nBlocks \leq \nTr}$},
deterministic function \eqsmall{$\func{\hashTrain}{\domainZ}{\setint{\nBlocks}}$} partitions the instance space, and by consequence~\eqsmall{$\trainSet$}.
Let
\eqsmall{${\blockOne, \ldots, \blockFin}$}
be the $\nBlocks$~disjoint training set \keyword{blocks} defined by \eqsmall{$\hashTrain$}
where \eqsmall{${\trainSet = \sqcup_{\blockIdx = 1}^{\nBlocks} \blockI}$}.

\keyword{Model}~\eqsmall{$\func{\dec}{\domainX}{\real}$} is trained on full set~$\trainSet$.
$\dec$~may be a single \keyword{decision function} or the fusion of an \keyword{ensemble} of \eqsmall{$\nModel$}~submodels, where for simplicity \eqsmall{$\nModel$} is chosen to be odd. %
Let \eqsmall{$\decI$} denote a submodel where \eqsmall{${\modIdx \in \setint{\nModel}}$}.
Each \eqsmall{$\decI$}~has its own training data \eqsmall{${\modTrainSetI \subset \trainSet}$}.
Submodel training is \keyword{deterministic}, meaning training on fixed~\eqsmall{$\modTrainSetSym$} always yields the same model.
We separately consider both \keyword{partitioned} (\eqsmall{${\forall_{\modIdx, \modIdxAlt} \, \modTrainSetI \cap \modTrainSetIAlt = \emptyset}$})
and \keyword{overlapping}
(\eqsmall{${\exists_{\modIdx, \modIdxAlt} \, \modTrainSetI \cap \modTrainSetIAlt \ne \emptyset}$})
submodel training data.

\prelimParagraph{Threat Model}%
\label{sec:Preliminaries:ThreatModel}
For arbitrary \keyword{test instance}~\eqsmall{${(\xTe, \yTe)}$}, the adversary's objective is to alter the model so that the \keyword{prediction error} \eqsmall{${\abs{\decFunc{\xTe} - \yTe}}$} is as large as possible.
Our primary threat model considers an adversary that can insert arbitrary instances into training set~\eqsmall{$\trainSet$} and arbitrarily delete instances from~\eqsmall{$\trainSet$}.%
\footnote{%
  Sec.~\ref{sec:BeyondUnitCost} considers a somewhat restricted threat model where attackers only make arbitrary deletions but no insertions.
  This allows us to empirically evaluate our method despite few base models fully utilizing our threat model.%
}
The attacker has perfect knowledge of the learner and our method.
We make \textit{no assumptions about the underlying data distribution or adversarial training instances}.

\prelimParagraph{Our Objective}%
\label{sec:Preliminaries:OurObjective}
Determine \keyword{certified robustness}~$\certBound$ -- a guarantee on the number of training instances that can be inserted into or deleted from training set~$\trainSet$ without the model prediction ever violating the requirement that \eqsmall{${\lThreshold \leq \decFunc{\xTe} \leq \uThreshold}$}, where \eqsmall{${\lThreshold, \uThreshold \in \real}$} are user-specified and application dependent. %
Note that robustness~$\certBound$ is \keyword{pointwise}, meaning each prediction \eqsmall{$\decFunc{\xTe}$} is certified individually.

\subsection{One-Sided vs.\ Two-Sided Certification Bounds}%
\label{sec:Preliminaries:OneSidedTwoSided}
For simplicity, the remaining sections exclusively describe how to certify a \textit{one-sided upper bound}, \eqsmall{${\decFunc{\X} \leq \threshold}$}, since all other bounds reduce to this base case.
For example, certifying a \keyword{one-sided lower bound} reduces to certifying an upper bound via negation as \eqsmall{${\decFunc{\X} \geq \threshold \Leftrightarrow -\decFunc{\X} \leq -\threshold}$}.
Likewise, a \keyword{two-sided bound} is equivalent to the worst one-sided robustness as %

{%
  \equationSize%
  \begin{equation}
    \begin{aligned}
      \lThreshold \leq \decFunc{\X} \leq \uThreshold &\Leftrightarrow \big(\decFunc{\X} \geq \lThreshold\big) \wedge \big( \decFunc{\X} \leq \uThreshold \big) \\
                                                     &\Leftrightarrow \big( {-}\decFunc{\X} \leq {-}\lThreshold \big) \wedge \big( \decFunc{\X} \leq \uThreshold \big)
      \text{.}
    \end{aligned}
  \end{equation}%
}%
\noindent

\subsection{Relating Regression and Binary Classification}%
\label{sec:Preliminaries:RelatingRegressionClassification}

Binary classification can be viewed as a simple form of regression where \eqsmall{${\domainY = \set{\pm1}}$}. %
The model's decision function becomes \eqsmall{$\sign{\decFunc{\xTe}}$} where \eqsmall{${\sign a = +1}$} if \eqsmall{${a > 0}$} and \eqsmall{${-1}$}~otherwise.
While our primary focus is regression, our methods also achieve state-of-the-art results for binary classification. %
\newcommand{\relatedWorkParagraph}[1]{%
  \vspace{6pt}%
  \noindent%
  \textbf{#1}
}

\section{Related Work}%
\label{sec:RelatedWork}

Techniques to mitigate (regression) models' implicit fragility have been studied for over half a century.
Below we partition these methods into three categories with progressively stronger robustness guarantees.

\relatedWorkParagraph{Resilient Regression}
Early methods were rooted in \keyword{robust statistics} and focused on mitigating the effect of training set outliers.
For example, various trimmed loss functions (e.g., Huber~\citep{Huber:1964}, Tukey~\citep{Beaton:1974:TukeyBiweight}) cap an outlier's influence on a model~\citep{Dennis:1978:WelschLoss,Leclerc:1989}.
Methods like RANSAC~\citep{Fischler:1981:RANSAC}
employ another common robustness strategy of detecting and removing training set outliers~\citep{Torr:2000:MLESAC,Rousseeuw:2011}.

\relatedWorkParagraph{Adversarially-Robust Regression}
The above methods primarily target random noise/outliers.
Adversarial training instances can be much more insidious since they are crafted to avoid detection by appearing uninfluential and may only affect a very small fraction of test predictions~\citep{Chen:2017:Targeted,Wallace:2021}.
These factors can combine to make adversarial training instances difficult for resilient methods to fully detect and correct~\citep{Li:2022:Survey}.

Some existing adversarial regression defenses do provide pointwise robustness guarantees, albeit under strong assumptions about the underlying data distribution~\citep{Klivans:2018:OutlierRobustRegression}.
For example, some work assumes that the training set follows a linear data distribution with arbitrary white, Gaussian noise~\citep{Chen:2013:RobustSparseRegression,Liu:2020:RobustSparseRegression}.
Others assume the data distribution's feature matrix is low rank~\citep{Liu:2017:RobustLinearRegression}.
Conditioning a guarantee on a specific data distribution is inherently precarious -- in particular if the strong distributional assumption rarely holds and cannot be easily verified.
If the distributional assumption does not hold, any guarantee is no guarantee at all.

Note that there are adversarially-robust regression defenses that provide guarantees without making distributional assumptions~\citep{Jagielski:2018:Trim,Klivans:2018:OutlierRobustRegression}.
However, their robustness guarantees are themselves distributional.
For example, \citet{Jagielski:2018:Trim} bound the clean training data's mean error but provide no pointwise guarantees.
In other words, such methods do not provide insight into each prediction's robustness.

A major strength of our approach to certified regression is that we
\textit{provide pointwise guarantees without any assumptions}.
Lastly, many existing adversarially-robust regressors consider exclusively linear models~\citep{Chen:2013:RobustSparseRegression,Liu:2017:RobustLinearRegression,Jagielski:2018:Trim,Liu:2020:RobustSparseRegression}.
However, state-of-the-art regressors increasingly leverage non-linear methods~\citep{Chen:2016:XGBoost,Prokhorenkova:2018:CatBoost,Arik:2021:TabNet,Brophy:2022:TreeInfluence}.
In contrast, our ensemble-based certified regressors support any submodel architecture. %
\relatedWorkParagraph{Certified Classification}
Recent work has proposed numerous \textit{classifiers} provably robust to poisoning and backdoor attacks~\citep{Steinhardt:2017,Wang:2020,Rosenfeld:2020:CertifiedLabelFlipping,Jia:2021:CertifiedBaggingRobustness,Weber:2023:RAB}.
These state-of-the-art certified classifiers all rely on \keyword{majority voting-based} methods and derive their guarantees by (lower) bounding the number of training set modifications
needed for the label with the second-most votes to overtake the plurality label.
\citepos{Jia:2022:CertifiedKNN} certified $\kNeigh$\=/nearest neighbor~(\knn) classifier is the simplest such method,  %
where the set of ``votes'' is the training labels from the test instance's neighborhood.
Due to space, we defer a detailed discussion and analysis of \citet{Jia:2022:CertifiedKNN}'s method to suppl.\ Sec.~\ref{sec:App:JiaVsKnnCR}.
\citet{Levine:2021:DPA} propose \keyword{deep partition aggregation}~(DPA), a general, ensemble-based certified classifier.
Suppl.\ Sec.~\ref{sec:App:PcrVsDpa} describes DPA in detail, but briefly, DPA's deterministic submodels are fully-independent since they are trained on disjoint data.
Given a test instance, each submodel predicts a label, and the overall prediction is the ensemble's plurality label.
To turn these labels (i.e.,~votes) into a robustness guarantee, DPA needs to certify each submodel's robustness.
However, DPA sidesteps this by always assuming \textit{worst-case} submodel robustness, which we formalize below.

\begin{definition}
  \label{def:RelatedWork:UnitCostAssumption}
  \defKeyword{Unit-Cost Assumption}: Any modification to a submodel's training set changes the submodel arbitrarily.
\end{definition}

In practice, there are limits to how much a single training set modification will alter a submodel and its predictions -- in particular for models with strong inductive biases.
Therefore, the unit-cost assumption's pessimism can cause methods like DPA to underestimate a prediction's true robustness.
Nonetheless, this assumption greatly simplifies ensemble robustness certification by reducing the task to just submodel vote counting.
Most recently, \citet{Wang:2022:DeterministicAggregation} modify DPA's ensemble so that submodels can be trained on overlapping data, which (slightly) improves the ensemble's certification bounds.
While voting-based methods work well for classification with nominal~$\domainY$, these ideas have not yet been adapted to regression where $\domainY$ is continuous.
This work fills in that gap by providing a reduction that adapts certified classifiers to certify regression.
We specifically detail the reduction for the certified classifiers proposed by \citet{Jia:2022:CertifiedKNN}, \citet{Levine:2021:DPA}, and \citet{Wang:2022:DeterministicAggregation}.
By building on these methods, we inherit their property of not needing to make assumptions about the data distribution or model architecture.
The fundamental challenge of our reduction is making a continuous output space behave like a robust, nominal label space.
We describe the solution to this challenge next.%

\section{Warmup: Perturbing a Set's Median}%
\label{sec:WarmUp}

Traditional center statistics such as mean have a \keyword{breakdown point} of~0, i.e., altering a single value in a set can shift the mean arbitrarily.
In contrast, median has maximum robustness, i.e.,~a breakdown of~50\%.
A high breakdown point entails that a statistic is stable and resistant to change.
We formalize changes to median below.

\begin{definition}
  \defKeyword{Median Perturbation}: The task of altering a set's contents so that its median exceeds some specified~\eqsmall{${\threshold \in \real}$}.
\end{definition}

Throughout this work, determining pointwise robustness~\eqsmall{$\certBound$} simplifies to quantifying the number of changes that can be made to a set without perturbing its median.
To better foster intuitions, we first formalize robustness~\eqsmall{$\certBound$} w.r.t.\ simply perturbing a multiset's median and unrelated to any model.
Later sections apply these ideas to link certified regression and certified classification.

Formally, let \eqsmall{$\setVals$} be a multiset of cardinality~\eqsmall{${\nModel \defeq \abs{\setVals}}$}.
Denote the subset of elements in~\eqsmall{$\setVals$} that are at most~\eqsmall{$\threshold$} as
\eqsmall{${%
    \lowerSetVals
      \defeq
        \setbuild{\setScalarI \in \setVals}
                 {\setScalarI \leq \threshold}
}$} and denote its %
complement \eqsmall{${\upperSetVals \defeq \setVals \setminus \lowerSetVals}$}.

Below we define three different paradigms that constrain how \eqsmall{$\setVals$} is modified. Figure~\ref{fig:WarmUp:MedianPerturbation:General} visualizes our first two unweighted paradigms.
Note that Fig.~\ref{fig:WarmUp:MedianPerturbation:General}'s values are repeatedly used throughout this paper, including in Fig.~\ref{fig:WarmUp:MedianPerturbation:NonUniform} for our third median perturbation paradigm and later in Figs.~\ref{fig:CertifiedKNN} and~\ref{fig:CertifiedOverlap:Ensemble}.
In all cases below, consider when \eqsmall{${\medFunc{\setVals} \leq \threshold}$} since the degenerate case of \eqsmall{${\medFunc{\setVals} > \threshold}$} is by definition non-robust.

\subsection{Unweighted Swap Paradigm}%
\label{sec:WarmUp:Swap}

Here, set~\eqsmall{$\setVals$} has fixed, odd-valued%
\footnote{%
  Fixing $\nModel$ as odd simplifies the overall formulation and presentation since it ensures that $\setVals$'s median is always an element in $\setVals$.
  In all cases here where $\nModel$ is fixed as odd, $\nModel$ is always a user-selected hyperparameter.
  Extending our formulation to consider even $\nModel$ is not challenging but is verbose.%
}
cardinality~\eqsmall{${\nModel}$}.
All modifications to~\eqsmall{$\setVals$} take the form of ``swaps'' where a single value in~\eqsmall{$\setVals$} is replaced with any real number.
Fig.~\ref{fig:WarmUp:MedianPerturbation:Swap} visualizes the unweighted swap paradigm on a simple set \eqsmall{${\setVals = \set{\setValOne, \ldots, \setValFive}}$} of \eqsmall{${\nModel = 5}$} values.
Lemma~\ref{lem:Warmup:Swap} tightly bounds the number of arbitrary swaps~\eqsmall{$\certBound$} that can be made to~\eqsmall{$\setVals$} without perturbing its median.

\begin{figure}[t]
    \centering
    \begin{minipage}[t]{0.49\textwidth}
\newcommand{\minipageWidth}{\columnwidth}
\newcommand{\figSpacer}{\vspace{6pt}}

\centering
\begin{subfigure}{\minipageWidth}
  \centering
\begin{tikzpicture}[
    node distance = \medianPerturbNodeDist,
  ]

  \node (B1) [vertex, below median] {\setValOne};
  \node (B2) [vertex, below median, right of=B1] {\setValTwo};
  \node (B3) [vertex, below median, right of=B2] {\setValThree};
  \node (B4) [vertex, below median, right of=B3] {\setValFour};
  \node (A5) [vertex, above median, right of=B4] {\setValFive};
  \node (A6) [empty vertex, right of=A5] {$\infty$};
  \node (A7) [empty vertex, right of=A6] {$\infty$};
  \coordinate (Median) at ($(B3)$);
  \node (MedianLabel) [median label, above of=Median] {Initial Median\strut};
  \draw [->, label line] (MedianLabel) -- (B3)  {};

  \coordinate (Threshold) at ($(B4)!0.40!(A5)$);
  \node (ThresholdLabel) [median label, above of=Threshold] {$\threshold$\strut};
  \draw [->, label line] (ThresholdLabel) -- (Threshold)  {};
\end{tikzpicture}
   \caption{Initial set~${\setVals \defeq \lowerSetVals \sqcup \upperSetVals}$}%
  \label{fig:WarmUp:MedianPerturbation:Initial}
\end{subfigure}

\figSpacer%
\begin{subfigure}{\minipageWidth}
  \centering%
\begin{tikzpicture}[
    node distance = \medianPerturbNodeDist,
  ]

  \node (B1) [empty vertex] {\setValOne};

  \node (B2) [vertex, below median, right of=B1] {\setValTwo};

  \node (B3) [vertex, below median, right of=B2] {\setValThree};
  \node (B4) [vertex, below median, right of=B3] {\setValFour};
  \node (A5) [vertex, above median, right of=B4] {\setValFive};
  \node (A6) [vertex, added, right of=A5] {$\bm{\infty}$};

  \node (A7) [empty vertex, right of=A6] {$\bm{\infty}$};

  \coordinate (Median) at ($(B4)$);
  \node (MedianLabel) [median label, below of=Median] {New Median};
  \draw [->, label line] (MedianLabel) -- (B4)  {};
\end{tikzpicture}
   \caption{Unweighted swap paradigm with ${\certBound = 1}$}%
  \label{fig:WarmUp:MedianPerturbation:Swap}
\end{subfigure}

\figSpacer%
\begin{subfigure}{\minipageWidth}
  \centering%
\begin{tikzpicture}[
    node distance = \medianPerturbNodeDist,
  ]

  \node (B1) [vertex, below median] {\setValOne};
  \node (B2) [vertex, below median, right of=B1] {\setValTwo};
  \node (B3) [vertex, below median, right of=B2] {\setValThree};
  \node (B4) [vertex, below median, right of=B3] {\setValFour};
  \node (A5) [vertex, above median, right of=B4] {\setValFive};
  \node (A6) [vertex, added, right of=A5] {$\bm{\infty}$};
  \node (A7) [vertex, added, right of=A6] {$\bm{\infty}$};

  \coordinate (Median) at ($(B4)$);

  \node (MedianLabel) [median label, below of=Median] {New Median};

  \draw [->, label line] (MedianLabel) -- (B4)  {};
\end{tikzpicture}
   \caption{Insertion only with ${\certBound = 2}$}
  \label{fig:WarmUp:MedianPerturbation:Add}
\end{subfigure}

\figSpacer%
\begin{subfigure}{\minipageWidth}
  \centering%
\begin{tikzpicture}[
    node distance = \medianPerturbNodeDist,
  ]

  \node (B1) [empty vertex] {\setValOne};
  \node (B2) [empty vertex, right of=B1] {\setValTwo};

  \node (B3) [vertex, below median, right of=B2] {\setValThree};

  \node (B4) [vertex, below median, right of=B3] {\setValFour};
  \node (A5) [vertex, above median, right of=B4] {\setValFive};
  \node (A6) [empty vertex, right of=A5] {$\infty$};
  \node (A7) [empty vertex, right of=A6] {$\infty$};

  \coordinate (Median) at ($(B4)$);

  \node (MedianLabel) [median label, below of=Median] {New Median};

  \draw [->, label line] (MedianLabel) -- (B4)  {};
\end{tikzpicture}
   \caption{Deletion only with ${\certBound = 2}$}
  \label{fig:WarmUp:MedianPerturbation:Delete}
\end{subfigure}
 
  \caption{%
    \textbf{Unweighted Median Perturbation}:
    (\ref{fig:WarmUp:MedianPerturbation:Initial})
    \blue{Blue} denotes elements in subset~\eqsmall{$\lowerSetVals$}, i.e.,~elements in \eqsmall{$\setVals$} with value at most~${\threshold = \setValThreshold}$.
    \eqsmall{$\upperSetVals$}'s values %
    are \BrickRed{red}.
    Each ``swap'' (\ref{fig:WarmUp:MedianPerturbation:Swap}) switches a \blue{value} in~\eqsmall{$\lowerSetVals$} with an arbitrarily large \green{replacement}.
    Deletions (\ref{fig:WarmUp:MedianPerturbation:Delete}) and \green{insertions} (\ref{fig:WarmUp:MedianPerturbation:Add}) are interchangeable (suppl.\ Lemma~\ref{lem:App:Proofs:AddDeleteSame}), with both yielding the same median value in the same number of modifications made to~\eqsmall{$\setVals$}.
    In Figs.~\ref{fig:WarmUp:MedianPerturbation:Swap} to~\ref{fig:WarmUp:MedianPerturbation:Delete} above, any additional modifications to the set would perturb the median.
  }%
  \label{fig:WarmUp:MedianPerturbation:General}
     \end{minipage}
    \hfill
    \begin{minipage}[t]{0.48\textwidth}
\newcommand{\minipageWidth}{\columnwidth}
\newcommand{\figSpacer}{\vspace{4pt}}
\centering

\begin{subfigure}{\minipageWidth}
  \centering
\newcommand{\unusedWeightVal}{}
\newcommand{\unusedWeightSep}{}

\begin{tikzpicture}[
    node distance = \medianPerturbNodeDist,
  ]

  \node (B1) [vertex, below median] {\setValOne};
  \node (B2) [vertex, below median, right of=B1] {\setValTwo};
  \node (B3) [vertex, below median, right of=B2] {\setValThree};
  \node (B4) [vertex, below median, right of=B3] {\setValFour};
  \node (A5) [vertex, above median, right of=B4] {\setValFive};
  \node (A6) [empty vertex, right of=A5] {$\infty$};
  \node (A7) [empty vertex, right of=A6] {$\infty$};

  \node (B1L) [vertex cost, below of=B1] {$\unusedWeightVal\unusedWeightSep\setCovOne$};
  \node (B2L) [vertex cost, below of=B2] {$\unusedWeightVal\unusedWeightSep\setCovTwo$};
  \node (B3L) [vertex cost, below of=B3] {$\unusedWeightVal\unusedWeightSep\setCovThree$};
  \node (B4L) [vertex cost, below of=B4] {$\unusedWeightVal\unusedWeightSep\setCovFour$};
  \node (A5L) [vertex cost, below of=A5] {$\unusedWeightVal\unusedWeightSep\setCovFive$};

  \node (VLabel) [vals label, left of=B1] {\ValLabel{$\setVals$}};
  \node (RLabel) [cov label, below of=VLabel] {\ValLabel{$\covSet$}};

  \coordinate (Median) at ($(B3)$);
  \node (MedianLabel) [median label, above of=Median] {Initial Median\strut};
  \draw [->, label line] (MedianLabel) -- (B3)  {};

  \coordinate (Threshold) at ($(B4)!0.40!(A5)$);
  \node (ThresholdLabel) [median label, above of=Threshold] {$\threshold$\strut};
  \draw [->, label line] (ThresholdLabel) -- (Threshold)  {};
\end{tikzpicture}

  \caption{Initial sets where ${\setVals = \set{\setValOne, \ldots, \setValFive}}$ and ${\covSet = \set{\setCovOne, \ldots, \setCovFive}}$}
  \label{fig:WarmUp:MedianPerturbation:NonUniform:Initial}
\end{subfigure}

\figSpacer
\begin{subfigure}{\minipageWidth}
  \centering
\newcommand{\UsedCostDef}[2]{$#1$}
\newcommand{\BaseAmountDef}[1]{\tiny(\textnormal{of} $#1$)}

\def\includeCosts{1}

\begin{tikzpicture}[
    node distance = \medianPerturbNodeDist,
  ]

  \node (B1) [empty vertex] {\setValOne};
  \node (B2) [vertex, below median, right of=B1] {\setValTwo};
  \node (B3) [vertex, below median, right of=B2] {\setValThree};
  \node (B4) [vertex, below median, right of=B3] {\setValFour};
  \node (A5) [vertex, above median, right of=B4] {\setValFive};
  \node (A6) [vertex, added, right of=A5] {$\bm{\infty}$};
  \node (A7) [empty vertex, right of=A6] {$\infty$};
  \node (B2L) [vertex cost, below of=B2] {\UsedCostDef{\setCovTwoMOne}{\setCovTwo}};
  \node (B3L) [vertex cost, below of=B3] {\UsedCostDef{0}{\setCovThree}};
  \node (B4L) [vertex cost, below of=B4] {\UsedCostDef{0}{\setCovFour}};
  \node (A5L) [vertex cost, below of=A5] {\UsedCostDef{0}{\setCovFive}};
  \node (A6L) [vertex cost, below of=A6] {\UsedCostDef{\setCovOne}{\setCovOne}};

  \ifdefined\includeCosts
    \node (B2Weight) [base weight text, below of=B2L] {\BaseAmountDef{\setCovTwo}};
    \node (B3Weight) [base weight text, below of=B3L] {\BaseAmountDef{\setCovThree}};
    \node (B4Weight) [base weight text, below of=B4L] {\BaseAmountDef{\setCovFour}};
    \node (A5Weight) [base weight text, below of=A5L] {\BaseAmountDef{\setCovFive}};
    \node (A6Weight) [base weight text, below of=A6L] {\BaseAmountDef{\setCovOne}};
  \fi

  \coordinate (B2B3Mid) at ($(B2L)!0.5!(B3L)$);
  \node (B2B3Text) [vals label] at (B2B3Mid) {$+$};
  \coordinate (B3B4Mid) at ($(B3L)!0.5!(B4L)$);
  \node (B3B4Text) [vals label] at (B3B4Mid) {$+$};
  \coordinate (B4A5Mid) at ($(B4L)!0.5!(A5L)$);
  \node (B4A5Text) [vals label] at (B4A5Mid) {$+$};
  \coordinate (A5A6Mid) at ($(A5L)!0.5!(A6L)$);
  \node (A5A6Text) [vals label] at (A5A6Mid) {$+$};

  \node (VLabel) [vals label, hide text, left of=B1] {\ValLabel{$\setVals$}};
  \node (RLabel) [cov label, below of=VLabel] {${\certBound=}$};

  \coordinate (Median) at ($(B3L)$);
  \ifdefined\includeCosts
    \node (MedianLabel) [median label, below of=B4Weight, node distance=0.7cm] {New Median};
    \draw [->, label line] (MedianLabel) -- (B4Weight)  {};
  \else
    \node (MedianLabel) [median label, below of=B4L, node distance=0.7cm] {New Median};
    \draw [->, label line] (MedianLabel) -- (B4L)  {};
  \fi

\end{tikzpicture}
 
  \caption{Weighted swap paradigm with ${\certBound = \setCovFullBound}$}
  \label{fig:WarmUp:MedianPerturbation:NonUniform:Swap}
\end{subfigure}
 
  \caption{%
    \textbf{Weighted Swap Paradigm}:
    Extension of Fig.~\ref{fig:WarmUp:MedianPerturbation:General} to weighted costs.
  For simplicity and w.l.o.g., let \eqsmall{${\covSet = \set{\setCovOne, \ldots, \setCovFive}}$}, i.e.,~${\forall_{\modIdx} \, \covI = \setScalarI \covValDiff}$
    Fig.~\ref{fig:WarmUp:MedianPerturbation:NonUniform:Initial} is identical to Fig.~\ref{fig:WarmUp:MedianPerturbation:Initial} except below each element~\eqsmall{$\setScalarI$} is its corresponding weight~\eqsmall{$\covI$}.
    Observe \eqsmall{${\swapBound = 1}$} and \eqsmall{${\smallestCovVals = \set{\setCovOne, \setCovTwo}}$}.
    Fig.~\ref{fig:WarmUp:MedianPerturbation:NonUniform:Swap} shows that for \eqsmall{${\certBound = \setCovFullBound}$} (visualized below each element), it is impossible to perturb the median, and any additional weight would be sufficient to swap out \eqsmall{${\setScalarI[2] = \setValTwo}$}.
  }%
  \label{fig:WarmUp:MedianPerturbation:NonUniform}
     \end{minipage}
\end{figure}

\begin{lemma}%
  \label{lem:Warmup:Swap}
  For
  \eqsmall{${\threshold \in \real}$},
  real multiset \eqsmall{$\setVals$} where \eqsmall{${\medFunc{\setVals} \leq \threshold}$} with \eqsmall{${\nModel \defeq \abs{\setVals}}$} odd,
  and
  \eqsmall{${\lowerSetVals \defeq \setbuild{\setScalarI \in \setVals}{\setScalarI \leq \threshold}}$},
  let \eqsmall{$\setValsMod$} be a multiset formed from~\eqsmall{$\setVals$} where elements have been arbitrarily replaced.
  If the number of elements replaced in~\eqsmall{$\setValsMod$} does not exceed %

  {%
    \equationSize%
    \begin{equation}
      \label{eq:Warmup:Swap:Lemma:Bound}
      \certBound
        =
        \nLower
        -
        \halfModels %
      \text{,}
    \end{equation}%
  }%
  \noindent%
  it is guaranteed that \eqsmall{${\medFunc{\setValsMod} \leq \threshold}$}.
\end{lemma}

\begin{proofsketch}
  For a set of odd cardinality~\eqsmall{$\nModel$}, the median is always the set's \eqsmall{$\halfModels$}\=/th largest value.
  For \eqsmall{$\setVals$}'s median to be at most~\eqsmall{$\threshold$}, at least \eqsmall{$\halfModels$}~items in~\eqsmall{$\setVals$} cannot exceed~\eqsmall{$\threshold$}.
  Each swap reduces the number of elements not exceeding~\eqsmall{$\threshold$} by at most one.
  If there are \eqsmall{$\nLower$}~elements less than or equal to~\eqsmall{$\threshold$} in~\eqsmall{$\setVals$} and there must be at least \eqsmall{$\halfModels$} such elements to avoid perturbing the median, then at most \eqsmall{${\nLower - \halfModels}$}~swaps can be performed.
\end{proofsketch}

\subsection{Insertion/Deletion Paradigm}%
\label{sec:WarmUp:InsertOrDelete}

For the second paradigm,
\eqsmall{$\setVals$} is no longer fixed cardinality (it may expand or contract), and \eqsmall{$\nModel$}~may be even or odd.
Each modification of~\eqsmall{$\setVals$} takes the form of either a single deletion or insertion but not both.
Figs.~\ref{fig:WarmUp:MedianPerturbation:Add} and~\ref{fig:WarmUp:MedianPerturbation:Delete} visualize median perturbation under insertions and deletions resp.\ with certified robustness~\eqsmall{$\certBound$} following
Lem.~\ref{lem:Warmup:AddDelete}.
Suppl.\ Sec.~\ref{sec:App:AdditionalLemma} proves
that worst-case insertions and deletions perturb a set's median in exactly the same way and thus are interchangeable.
That is why
Figs.~\ref{fig:WarmUp:MedianPerturbation:Add} and~\ref{fig:WarmUp:MedianPerturbation:Delete} have identical certified robustness (\eqsmall{${\certBound = 2}$}).

\begin{lemma}%
  \label{lem:Warmup:AddDelete}
  \newcommand{\setScalarIM}{\setScalarI_{\nModUpper}}
  For \eqsmall{${\threshold \in \real}$} and
  real multiset \eqsmall{$\setVals$} where \eqsmall{${\medFunc{\setVals} \leq \threshold}$},
  define
  \eqsmall{${\nModel \defeq \abs{\setVals}}$}
  and
  \eqsmall{${\lowerSetVals \defeq \setbuild{\setScalarI \in \setVals}{\setScalarI \leq \threshold}}$}.
  Let \eqsmall{$\setValsMod$} be any multiset formed from \eqsmall{$\setVals$} where elements have been arbitrarily deleted and/or inserted.
  Then, if the total number of inserted and deleted elements in~\eqsmall{$\setValsMod$} does not exceed

  {%
    \equationSize%
    \begin{equation}
      \label{eq:Warmup:AddDelete:Lemma:Bound}
      \certBound
      =
        2\nLower
        -
        \nModel
        -
        1
      \text{,}
    \end{equation}%
  }%
  \noindent%
  it is guaranteed that \eqsmall{${\medFunc{\setValsMod} \leq \threshold}$}.
\end{lemma}

\noindent
Eq.~\eqref{eq:Warmup:Swap:Lemma:Bound}'s bound may be non-tight by~1.
We did this for consistency with other ideas.
See suppl.\ Sec.~\ref{sec:App:Tightness:RegionBasedIBL} for details.

Comparing Eqs.~\eqref{eq:Warmup:Swap:Lemma:Bound} \&~\eqref{eq:Warmup:AddDelete:Lemma:Bound},
the insertion/deletion paradigm's robustness~\eqsmall{$\certBound$} is about twice that of the unweighted swap paradigm.
Intuitively, this is because one swap entails two separate operations -- both an insertion and a deletion.

\subsection{Weighted Swap Paradigm}%
\label{sec:Warmup:NonUniformSwap}

The two median-perturbation paradigms above assume that each modification to~\eqsmall{$\setVals$} has equivalent cost.
Consider a generalized swap paradigm where each value \eqsmall{${\setScalarI \in \setVals}$} has an associated weight/cost \eqsmall{${\covI \in \nats}$}.
We seek to tightly bound the budget an attacker could spend with it remaining guaranteed that \eqsmall{${\decFunc{\xTe} \leq \threshold}$}; we still denote this budget~\eqsmall{$\certBound$}.

Given \eqsmall{$\lowerSetVals$} as above, \eqsmall{${\lowerCovVals \defeq \setbuild{\covI}{\setScalarI \in \lowerSetVals}}$} contains \eqsmall{$\lowerSetVals$}'s corresponding weights/costs.
Define \eqsmall{${\swapBound \defeq \nLower - \halfModels}$},
and let multiset \eqsmall{$\smallestCovValsMinusOne$} be the \eqsmall{$\swapBound$} smallest values in~\eqsmall{$\lowerCovVals$} (i.e.,~\eqsmall{${\abs{\smallestCovValsMinusOne} = \swapBound}$}).
Directly applying Lem.~\ref{lem:Warmup:Swap}, an obvious but non-optimal bound is

{%
  \equationSize%
  \begin{equation} %
    \label{eq:WarmUp:NonUniformSwap:Naive}%
    \certBound %
      \geq %
      \sum_{\covSym \in \smallestCovValsMinusOne}%
        \covSym%
      \text{.}
  \end{equation}%
}%

Recall Fig.~\ref{fig:WarmUp:MedianPerturbation:Initial} where
\eqsmall{${\setVals = \set{\setValOne, \ldots, \setValFive}}$}
and
\eqsmall{${\threshold = \setValThreshold}$}.
Consider its weighted extension where for simplicity and w.l.o.g.\ \eqsmall{${\covSet = \set{\setCovOne, \ldots, \setCovFive}}$}, i.e., \eqsmall{${\forall_{\modIdx} \, \covI = \setScalarI \covValDiff}$}.
Eq.~\eqref{eq:WarmUp:NonUniformSwap:Naive} certifies robustness \eqsmall{${\certBound = \setCovBaseBound}$} for this example. %
However, Fig.~\ref{fig:WarmUp:MedianPerturbation:NonUniform:Swap} shows \eqsmall{${\certBound = \setCovFullBound}$} since the budget of the second
(i.e.,
\eqsmall{${(\swapBound + 1)}$}\=/th)
largest value in~\eqsmall{$\lowerCovVals$}
can be partially used. %
Lemma~\ref{lem:Warmup:NonUniformSwap} formalizes this insight into a tight bound for median perturbation under weighted swaps.

\begin{lemma}%
  \label{lem:Warmup:NonUniformSwap}%
  For \eqsmall{${\threshold \in \real}$}
  and real multiset \eqsmall{$\setVals$} where \eqsmall{${\medFunc{\setVals} \leq \threshold}$},
  let \eqsmall{$\covSet$} be \eqsmall{$\setVals$}'s corresponding integral weight multiset where \eqsmall{${\nModel \defeq \abs{\setVals} = \abs{\covSet}}$} is fixed and odd.
  Define
  \eqsmall{${\lowerCovVals \defeq \setbuild{\covI \in \covSet}{\setScalarI \leq \threshold}}$},
  and let
  \eqsmall{$\smallestCovVals$}
  be the smallest \eqsmall{${(\abs{\setVals} - \halfModels + 1)}$} values in \eqsmall{$\lowerCovVals$}.
  Then the cost to perturb \eqsmall{$\setVals$}'s median exceeds

  {%
    \equationSize%
    \begin{equation}%
      \label{eq:WarmUp:NonUniformSwap:Bound}%
      \certBound
      =
        \sum_{\covSym \in \smallestCovVals}
          \covSym
        -
        1
      \text{.}
    \end{equation}%
  }%
\end{lemma}

\newcommand{\setValsClass}{\setValsZO}

\section{Reducing Regression to \revised{Voting-Based} Binary Classification}%
\label{sec:RegressionToClassification}

We now show how methods used to certify binary classification can be adapted to certify regression.
During inference, all voting-based certified methods (both classifiers and regressors)
follow the same basic procedure.

First, the model generates a multiset of votes, which for binary classification we denote \eqsmall{$\setValsClass$}.
Certified classifiers only differ in how~\eqsmall{$\setValsClass$} is constructed and in the consequences that construction has on certifying~\eqsmall{$\certBound$}.
For example, \eqsmall{$\setValsClass$} could be a \knn{} neighborhood or the submodel predictions in an ensemble.
Nonetheless,
for binary classification, \eqsmall{$\setValsClass$} contains at most two unique values (\eqsmall{${+1}$} and \eqsmall{${-1}$}), meaning \eqsmall{$\setValsClass$}'s majority label is also its median.
In other words, \eqsmall{${\decTe = \medFunc{\setValsZO}}$}.

To certify robustness~\eqsmall{$\certBound$}, existing methods rely on a function we term the \keyword{robustness certifier}.
The function's inputs are votes \eqsmall{$\setValsZO$} and optionally weights/costs \eqsmall{$\covSet$}.
Implicitly, the certifier knows how the votes were generated and how changes to training set~\eqsmall{$\trainSet$} could affect \eqsmall{$\setValsZO$}.
Generally, a simple procedure to construct \eqsmall{$\setValsZO$} entails a simple certifier, and complex construction implies a complex certifier.
Fundamentally,
for voting-based, binary classification, robustness certification always reduces to the same core idea.
If \eqsmall{${\decTe = \medFunc{\setValsZO}}$}, then for the runner-up label to overtake the majority label, \eqsmall{$\setValsZO$}'s median must be perturbed.
Therefore, \textit{certifying voting-based, binary classification is simply certifying median perturbation}.

To generalize a voting-based, certified classifier to certify regression, two primary modifications are required; we visualize our regression to classification reduction in Fig.~\ref{fig:RegressionToClassification}.
\begin{figure}[t]
  \centering

\begin{tikzpicture}[
  ]

  \node (MedBox) [certified box] {Median};

  \node (SetValsNode) [
    certified reduction input label,
    left of=MedBox,
    node distance={\InputVarMainSpacer},
  ] {$\setVals$~};
  \node (SetValsBranchNode) [
    certified reduction branch node,
  ] at ($(MedBox)!\BranchNodeFrac!(SetValsNode)$) {};

  \node (CertifiedBox) [
    below of=MedBox,
    node distance={\MedBoxMainSpacer},
    certified box,
    align=center,
  ] {Robustness \\ Certifier};

  \draw [
    certified reduction surround box,
  ]  ($(SetValsBranchNode)+(-0.70cm,0.50cm)$) rectangle ($(CertifiedBox.south east)+(0.4cm,-0.19cm)$) {};

  \node (DecFuncLabel) [
    certified reduction output label,
    right of=MedBox,
    node distance={\DecFuncMainSpacer},
  ] {$\decFunc{\xTe}$};
  \draw [->, certified reduct line] (MedBox) -- (DecFuncLabel)  {};

  \node (RobustnessLabel) [
    certified reduction output label,
    below of=DecFuncLabel,
    node distance={\MedBoxMainSpacer},
  ]
  {$\certBound$};

  \draw [->, certified reduct line] (CertifiedBox) -- (RobustnessLabel)  {};
  \coordinate (CertClassIn1) at ($(CertifiedBox.north west)!0.20!(CertifiedBox.south west)$);
  \coordinate (CertClassIn2) at ($(CertifiedBox.south west)!0.20!(CertifiedBox.north west)$);

  \draw [->, certified reduct line] (SetValsNode) -- (MedBox)  {};

  \node (WeightsVals) [
    certified reduction input label,
  ] at (SetValsNode |- CertClassIn2) {$\covSet$~};
  \draw [->, dashed, certified reduct line] (WeightsVals) -- (CertClassIn2)  {};

  \node (SignumBox) [
    signum box,
    left of={CertClassIn1},
  ] {$\sign(\cdot)$};
  \draw [->, certified reduct line] (SignumBox) -- (CertClassIn1) node[pos=0.45, white label, yshift=7pt] {$\setValsZO$};

  \node (SubtractionNode) [
    certified reduction subtraction,
  ]  at (SetValsBranchNode |- SignumBox.west) {$\boldsymbol{-}$};
  \draw [->, certified reduct line] (SetValsBranchNode) -- (SubtractionNode)  {};
  \draw [->, certified reduct line] (SubtractionNode) -- (SignumBox.west)  {};

  \node (ThresholdVal) [
    certified reduction input label,
  ] at (SetValsNode |- SubtractionNode) {$\threshold$~};
  \draw [->, certified reduct line] (ThresholdVal) -- (SubtractionNode)  {};

  \node (InputVarLabel) [
    certified reduction var descriptor,
    above of=SetValsNode,
    xshift=7.5pt,
  ] {Inputs};

  \node (OutputVarLabel) [
    certified reduction var descriptor,
    above of=DecFuncLabel,
    xshift=-5.0pt,
  ] {Outputs};

  \node (TopDescription)[
    certified reduction var descriptor,
    font={\ReductionSymbolFontSize\bfseries},
  ] at ($(InputVarLabel.west)!0.47!(OutputVarLabel.east)$) {Certified Regressor} ;
\end{tikzpicture}
   \caption{%
    \textbf{Certified Regression to Certified Classification Reduction}:
    For \eqsmall{${\xTe \in \domainX}$}, the decision function is \eqsmall{${\decFunc{\xTe} \defeq \medFunc{\setVals}}$} -- just like \revised{voting-based} certified classification.
    Certified regression binarizes \eqsmall{$\setVals$} into \eqsmall{$\setValsZO$}, which is used by the robustness certifier
    (optionally with weights~\eqsmall{$\covSet$}) to determine~\eqsmall{$\certBound$}.%
  }
  \label{fig:RegressionToClassification}
\end{figure}
 
First, the model is modified from generating binary votes~\eqsmall{$\setValsZO$} to generating real-valued ones denoted~\eqsmall{$\setVals$}.
The changes necessary to make this switch are specific to the underlying certified classifier.
In some cases, no change is required~\citep{Jia:2022:CertifiedKNN}; for others, ensemble submodel classifiers are simply replaced with submodel regressors~\citep{Levine:2021:DPA,Wang:2022:DeterministicAggregation}.

The second modification is more subtle.
If \eqsmall{$\setVals$} is real-valued, a robustness certifier expecting binary votes cannot be directly applied.
That is where \eqsmall{${\threshold \in \real}$} fits in; it partitions \eqsmall{$\setVals$} into two subsets:
\eqsmall{${\lowerSetVals}$} containing all ``votes'' at most~\eqsmall{$\threshold$} and \eqsmall{$\upperSetVals$} containing all ``votes'' exceeding~\eqsmall{$\threshold$}.
We can think
of these subsets as two different classes where if \eqsmall{${\decFunc{\xTe} \leq \threshold}$}, \eqsmall{$\lowerSetVals$} is the majority class and \eqsmall{$\upperSetVals$} runner-up.
For any prediction
\eqsmall{${\decFunc{\xTe} \defeq \medFunc{\setVals}}$},
the robustness certifier's output~\eqsmall{$\certBound$} equals the number of training set modifications that can be made without ever perturbing \eqsmall{$\medFunc{\setVals}$} beyond~\eqsmall{$\threshold$}.

Lemma~\ref{lem:RegressionToClassification} formalizes the connection between real-valued and binarized robustness.
This symmetry in robustness derives from both tasks' (implicit) shared reliance on median.

\begin{lemma}
  \label{lem:RegressionToClassification}
  \newcommand{\thresholdAlt}{\threshold'}
  \newcommand{\setValsAlt}{\setVals'}
  \newcommand{\covSetAlt}{\covSet}
  For \eqsmall{${\thresholdAlt \in \real}$}
  and
  real multiset \eqsmall{$\setValsAlt$} where \eqsmall{${\medFunc{\setValsAlt} \leq \thresholdAlt}$},
  let \eqsmall{${\setValsZO \defeq \setbuild{\sgnp{\setScalarI - \thresholdAlt}}{\setScalarI \in \setValsAlt}}$}.
  Let \eqsmall{$\covSetAlt$} be the corresponding integral weight multiset of \eqsmall{$\setValsAlt$} where \eqsmall{${\abs{\setValsAlt} = \abs{\covSetAlt}}$}.
  Then, under the (un)weighted swap and insertion/deletion paradigms,
  both
    \eqsmall{${\setVals = \setValsAlt}$} with \eqsmall{${\threshold = \thresholdAlt}$}
    and
    \eqsmall{${\setVals=\setValsZO}$} with \eqsmall{${\threshold = 0}$}
   have equivalent robustness~\eqsmall{$\certBound$}.
\end{lemma}

\noindent
By binarizing \eqsmall{$\setVals$},
Lem.~\ref{lem:RegressionToClassification} enables us to directly \textit{reuse robustness certifiers from binary classification to certify regression}.

Our reduction to certified classification entails two primary benefits.
First, it allows us to repurpose for regression the diverse set of voting-based, certified classifiers that already exist~\citep{Jia:2022:CertifiedKNN,Levine:2021:DPA,Wang:2022:DeterministicAggregation}.
Moreover,
as new voting-based, certified classifiers are proposed in the future, these yet undiscovered methods can also be reformulated as certified regressors.

\revised{%
  Although this work focuses on certified poisoning defenses, other types of certified defenses also rely on voting-based schemes, including randomized smoothing methods for evasion attacks~\citep{Levine:2020:RandomizedAblation,Jia:2022:AlmostTightL0}.
  Our certified regression to certified classification reduction can also be applied to these other types of voting-based defenses as well.
}

As mentioned above, the procedure to construct the set of votes and to certify robustness is unique to each classifier.
The next three sections describe how to certify regression using progressively more complex models, with each method based on a reduction to an existing \revised{voting-based} certified classifier.%

\section{Certified Instance-Based Regression}%
\label{sec:CertifiedKNN}

For the first method, recall from Sec.~\ref{sec:RelatedWork} that \citet{Jia:2022:CertifiedKNN} propose a state-of-the-art certified classifier based on \knn{}.
Nearest-neighbor methods are a specific type of
\keyword{instance-based learner}~(IBL), where predictions are made using memorized training instances~\citep{Aha:1991:InstanceBasedLearning}.
IBLs generally rely on the intuition that instances close together in \keyword{feature space}~(\eqsmall{$\domainX$}) have similar target values~(\eqsmall{$\domainY$}).
Specifically, IBLs search for stored training instances most similar to $\xTe$ and derive the model prediction from these retrieved \keyword{neighbors}.
We partition IBLs into two subcategories:
\begin{itemize}
  \item \keyword{Fixed-population neighborhood} methods specify the exact number of ``neighbors'' when making a prediction.
  \item \keyword{Region-based neighborhood} methods define a neighborhood as all training instances in a specific feature-space region.
\end{itemize}
\noindent
These two subcategories calculate certified robustness differently and are discussed separately below.

All IBLs considered here use the same decision rule.
Formally, given \eqsmall{${\xTe \in \domainX}$} and real multiset neighborhood~\eqsmall{$\neighTe$} returned by the IBL, the model's prediction is the neighborhood's median, %
i.e., \eqsmall{${\decFunc{\xTe} \defeq \medFunc{\neighTe}}$}.
Recall that our goal is to certify that if at most~$\certBound$ arbitrary insertions or deletions are made to~$\trainSet$, it is guaranteed that %
\eqsmall{${\decFunc{\xTe} \leq \threshold}$}.

\subsection{Fixed-Population Neighborhood}%
\label{sec:CertifiedKNN:FixedPopulation}

As the name indicates, fixed-population neighborhood IBLs make predictions using a fixed number of training instances, i.e.,~\eqsmall{${\forall_{\xTe} \, \nModel = \abs{\neighTe}}$}.
$\kNeigh$\=/nearest neighbors is perhaps the best-known fixed-population method.
Traditionally, \knn{} returns the neighborhood's mean value.
For clarity, we will refer to the version of \knn{} that uses the neighborhood's median value as \keyword{$\kNeigh$\=/Nearest Neighbors Median}, or simply \knnM{}.

Our threat model allows the adversary to insert arbitrary training instances and/or delete any existing instances.
Fig.~\ref{fig:CertifiedKNN:SwapPerturbed} visualizes an example attack on a \knnM{} regressor.
Since $\kNeigh$~is fixed, inserting a new instance~(\showAddedKnnSquare) into the neighborhood causes one neighborhood instance to be ejected; in other words, insertions are simply instance swaps.
As a worst-case, we assume that the ejected element equals at most threshold~$\threshold$, meaning each insertion always maximally increases the neighborhood's median. %
Under this simplifying assumption, adversarial insertions are always at least as harmful as deletions for fixed-population neighborhood IBLs.

Neighborhood size~$\kNeigh$ is a user-specified hyperparameter so let $\kNeigh$ be odd-valued.
Therefore, these fixed-population neighborhood IBL regressors satisfy all of the criteria of median perturbation under the unweighted swap paradigm where ${\nModel = \kNeigh}$.
Theorem~\ref{thm:CertifiedKNN:FixedSize} then follows directly from Lemma~\ref{lem:Warmup:Swap}.

\begin{theorem}
  \label{thm:CertifiedKNN:FixedSize}
  Let \eqsmall{$\dec$}~be an instance-based regressor trained on set~\eqsmall{$\trainSet$}.
  Given \eqsmall{${\threshold \in \real}$}
  and \eqsmall{${\xTe \in \domainX}$},
  let
  real multiset \eqsmall{$\neighTe$}
  be \eqsmall{$\xTe$}'s neighborhood under \eqsmall{$\dec$}
  with fixed, odd-valued cardinality
  \eqsmall{${\nModel \defeq \abs{\neighTe}}$}.
  Define \eqsmall{${\lowNeigh \defeq \setbuild{\y \in \neighTe}{\y \leq \threshold}}$}.
  Given \eqsmall{${\decFunc{\xTe} \defeq \medFunc{\neighTe} \leq \threshold}$},
  then if model~$\dec$ is trained on a modified \eqsmall{$\trainSet$} where the total number of inserted and deleted training instances does not exceed

  {%
    \equationSize%
    \begin{equation}%
      \label{eq:CertifiedKNN:FixedSize}
      \certBound
      =
        \abs{\lowNeigh}
        -
        \halfModels
      \text{,}%
    \end{equation}%
  }%
  \noindent%
  it is guaranteed that \eqsmall{${\decFunc{\xTe} \leq \threshold}$}.
\end{theorem}

We denote \keyword{\knnM{} certified regression} as \knnMethod{}.
Due to space, we defer to suppl.\ Sec.~\ref{sec:App:JiaVsKnnCR}
the proof that when under binary classification, \knnMethod{} and \citepos{Jia:2022:CertifiedKNN} \knn{} classifier yield identical robustness guarantees.%

\subsection{Region-Based Neighborhood}%
\label{sec:CertifiedKNN:RegionBased}

Neighborhood membership does not need to be tied to the number of neighbors.
Rather, a neighborhood can be defined by specific criteria, with all stored training instances satisfying those criteria included in the neighborhood.
For instance, \keyword{radius nearest neighbors}~(\rnn{}) defines $\xTe$'s neighborhood as all training instances within a given distance of $\xTe$~\citep{Bentley:1975:RadiusNearestNeighbor}.
Alternatively, \keyword{fully-random decision trees} recursively partition the feature space into disjoint regions, and a neighborhood is defined as all instances within the same feature region~\citep{Geurts:2006:ExtremelyRandomizedTrees}.%

Fig.~\ref{fig:CertifiedKNN:AddDeletePerturbed} visualizes an attack on an \rnn{}\=/median learner,
where the adversary inserts malicious instances~(\showAddedKnnSquare) to perturb the median prediction.
Unlike fixed-population neighborhoods, the inserted instances do not cause any existing training instances to be ejected.
Rather, inserting and deleting training instances are distinct operations.

It is easy to see that region-based IBLs with median as the decision operator follow Sec.~\ref{sec:WarmUp:InsertOrDelete}'s insertion/deletion paradigm.
Theorem~\ref{thm:CertifiedKNN:DynamicSize} then follows directly from Lemma~\ref{lem:Warmup:AddDelete}.

\begin{theorem}
  \label{thm:CertifiedKNN:DynamicSize}
  Let \eqsmall{$\dec$} be an instance-based regressor trained on~\eqsmall{$\trainSet$} that partitions~\eqsmall{$\domainX$} into disjoint regions.
  Given \eqsmall{${\xTe \in \domainX}$},
  let real multiset \eqsmall{$\neighTe$}
  be \eqsmall{$\xTe$}'s neighborhood under~\eqsmall{$\dec$} where
  \eqsmall{${\nModel \defeq \abs{\neighTe}}$}.
  For \eqsmall{${\threshold \in \real}$},
  define
  \eqsmall{${\lowNeigh \defeq \setbuild{\y \in \neighTe}{\y \leq \threshold}}$}.
  If model~\eqsmall{$\dec$} is trained on a modified \eqsmall{$\trainSet$} where the total number of inserted and deleted training instances does not exceed

  {%
    \equationSize%
    \begin{equation}%
      \label{eq:CertifiedKNN:DynamicSize}
      \certBound
      =
        2\lvert{\lowNeigh}\rvert
        -
        \nModel
        -
        1
      \text{,}%
    \end{equation}%
  }%
  \noindent%
  it is guaranteed that \eqsmall{${\decFunc{\xTe} \leq \threshold}$}.
\end{theorem}

\citeauthor{Jia:2022:CertifiedKNN} propose an \rnn{}\=/based certified classifier, with the robustness certifier identical to their \knn{} method.
By using our insertion/deletion paradigm for the robustness certifier instead of \citeauthor{Jia:2022:CertifiedKNN}'s approach, Eq.~\eqref{eq:CertifiedKNN:DynamicSize}'s \eqsmall{$\certBound$} roughly doubles.

\begin{figure}[t]
  \newcommand{\OutsideSpacer}{\hfill}
  \newcommand{\InsideSpacer}{\hspace{00pt}}
  \newcommand{\knnMiniWidth}{0.25\columnwidth}

  \centering
  \OutsideSpacer{}
  \begin{subfigure}[t]{\knnMiniWidth}
    \centering
    \captionsetup{justification=centering}
\begin{tikzpicture}[
    node distance = \medianPerturbNodeDist,
    test vertex/.style = {
    },
    scale=\knnTikzScale,
  ]

  \coordinate[
    xTe vertex,
    label=below:\knnTeLabelSize{$\xTe$}
  ] (Center) at (0,0);
  \draw[knn base radius] (Center) circle (1.42);
  \draw[knn base radius, white] (Center) circle (1.08);

  \coordinate[
    below med knn,
    label=below:\knnValueSize{$\setValOne$},
  ] (B1) at (-1.1,+0.90);

  \coordinate[
    below med knn,
    label=below:\knnValueSize{$\setValTwo$},
  ] (B2) at (0.89,0.62);

  \coordinate[
    below med knn,
    label=below:\knnValueSize{$\setValThree$},
  ] (B3) at (0.40,-0.50);

  \coordinate[
    below med knn,
    label=below:\knnValueSize{$\setValFour$},
  ] (B4) at (-0.48,-0.4);

  \coordinate[
    above med knn,
    label=below:\knnValueSize{$\setValFive$},
  ] (A5) at (0.55,0.15);

  \coordinate[other knn] (O1) at (-1.5,-1.0);
  \coordinate[other knn] (O2) at (-1.2,-1.2);
  \coordinate[other knn] (O3) at (-1.25,1.3);
  \coordinate[other knn] (O4) at (1.3,1.0);
  \coordinate[other knn] (O5) at (1.5,-0.7);
  \coordinate[other knn] (O6) at (1.2,-1.2);
  \coordinate[other knn] (O7) at (-1.6,0);
\end{tikzpicture}
     \caption{Unperturbed}%
    \label{fig:CertifiedKNN:Unperturbed}
  \end{subfigure}
  \InsideSpacer{}
  \begin{subfigure}[t]{\knnMiniWidth}
    \centering
    \captionsetup{justification=centering}
\begin{tikzpicture}[
    node distance = \medianPerturbNodeDist,
    test vertex/.style = {
    },
    scale=\knnTikzScale,
  ]

  \coordinate[
    xTe vertex,
    label=below:\knnTeLabelSize{$\xTe$}
  ] (Center) at (0,0);
  \draw[knn base radius, white] (Center) circle (1.42);
  \draw[knn base radius] (Center) circle (1.08);

  \coordinate[
    below med knn,
    label=below:\knnValueSize{$\setValOne$},
  ] (B1) at (-1.1,+0.90);

  \coordinate[
    below med knn,
    label=below:\knnValueSize{$\setValTwo$},
  ] (B2) at (0.89,0.62);

  \coordinate[
    below med knn,
    label=below:\knnValueSize{$\setValThree$},
  ] (B3) at (0.40,-0.50);

  \coordinate[
    below med knn,
    label=below:\knnValueSize{$\setValFour$},
  ] (B4) at (-0.48,-0.4);

  \coordinate[
    above med knn,
    label=below:\knnValueSize{$\setValFive$},
  ] (A5) at (0.55,0.15);

  \coordinate[
    added knn,
    label=above:\knnValueSize{$\infty$},
  ] (A7) at (-0.35,0.45);

  \coordinate[other knn] (O1) at (-1.5,-1.0);
  \coordinate[other knn] (O2) at (-1.2,-1.2);
  \coordinate[other knn] (O3) at (-1.25,1.3);
  \coordinate[other knn] (O4) at (1.3,1.0);
  \coordinate[other knn] (O5) at (1.5,-0.7);
  \coordinate[other knn] (O6) at (1.2,-1.2);
  \coordinate[other knn] (O7) at (-1.6,0);
\end{tikzpicture}
    \caption{Fixed-population \\ ${\certBound = 1}$}%
    \label{fig:CertifiedKNN:SwapPerturbed}
  \end{subfigure}
  \InsideSpacer{}
  \begin{subfigure}[t]{\knnMiniWidth}
    \centering
    \captionsetup{justification=centering}
\begin{tikzpicture}[
    node distance = \medianPerturbNodeDist,
    test vertex/.style = {
    },
    scale=\knnTikzScale,
  ]

  \coordinate[
    xTe vertex,
    label=below:\knnTeLabelSize{$\xTe$}
  ] (Center) at (0,0);
  \draw[knn base radius] (Center) circle (1.42);
  \draw[knn base radius, white] (Center) circle (1.08);

  \coordinate[
    below med knn,
    label=below:\knnValueSize{$\setValOne$},
  ] (B1) at (-1.1,+0.90);

  \coordinate[
    below med knn,
    label=below:\knnValueSize{$\setValTwo$},
  ] (B2) at (0.89,0.62);

  \coordinate[
    below med knn,
    label=below:\knnValueSize{$\setValThree$},
  ] (B3) at (0.40,-0.50);

  \coordinate[
    below med knn,
    label=below:\knnValueSize{$\setValFour$},
  ] (B4) at (-0.48,-0.4);

  \coordinate[
    above med knn,
    label=below:\knnValueSize{$\setValFive$},
  ] (A5) at (0.55,0.15);

  \coordinate[
    added knn,
    label=above:\knnValueSize{$\infty$},
  ] (A6) at (0.35,0.45);

  \coordinate[
    added knn,
    label=above:\knnValueSize{$\infty$},
  ] (A7) at (-0.35,0.45);

  \coordinate[other knn] (O1) at (-1.5,-1.0);
  \coordinate[other knn] (O2) at (-1.2,-1.2);
  \coordinate[other knn] (O3) at (-1.25,1.3);
  \coordinate[other knn] (O4) at (1.3,1.0);
  \coordinate[other knn] (O5) at (1.5,-0.7);
  \coordinate[other knn] (O6) at (1.2,-1.2);
  \coordinate[other knn] (O7) at (-1.6,0);
\end{tikzpicture}
     \caption{Region-based \\ ${\certBound = 2}$}%
    \label{fig:CertifiedKNN:AddDeletePerturbed}
  \end{subfigure}
  \OutsideSpacer{}

  \caption{%
    \textbf{Certified Instance-Based Regression}:
    Fig.~\ref{fig:CertifiedKNN:Unperturbed} visualizes an unperturbed IBL model.
    Test instance $\xTe$'s neighborhood is visualized as a dashed line with neighborhood \eqsmall{$\neighTe$} identical to \eqsmall{$\setVals$} in Fig.~\ref{fig:WarmUp:MedianPerturbation:Initial}.
    Fig.~\ref{fig:CertifiedKNN:SwapPerturbed} shows an attack on a \knnM{} model where the neighborhood's cardinality (\eqsmall{${\nModel = 5}$}) is fixed, and the one attack instance~(\showAddedKnnSquare) replaces one instance in \eqsmall{$\lowerSetVals$}~(\showBelowMedianKnn) (source Fig.~\ref{fig:WarmUp:MedianPerturbation:Swap}).
    A \rnn{}\=/median model is shown in Fig.~\ref{fig:CertifiedKNN:AddDeletePerturbed}, where the two inserted instances~(\showAddedKnnSquare) do not change the neighborhood's radius (source Fig.~\ref{fig:WarmUp:MedianPerturbation:Add}).%
  }%
  \label{fig:CertifiedKNN}
\end{figure}
 
\subsection{Computational Complexity}

Eqs.~\eqref{eq:CertifiedKNN:FixedSize} and~\eqref{eq:CertifiedKNN:DynamicSize} require determining \eqsmall{$\lowNeigh$}'s cardinality,
which has complexity \eqsmall{\bigOT}. %
However, constructing neighborhood~\eqsmall{$\neighTe$} can require scanning the entire training set and has complexity~\eqsmall{\bigOn{}}.
Therefore, certifying each IBL regression prediction's robustness is in~\eqsmall{\bigOn{}} -- the same as \citeauthor{Jia:2022:CertifiedKNN}'s certified \knn{} and \rnn{} classifiers.
\section{Certified Regression for General Models}%
\label{sec:CertifiedGeneral}

Instance-based learners lend themselves to robustness certification.
However, there are many applications where IBLs perform poorly.
This section explores reducing certified regression to a second certified classifier, which will now allow us to use whichever model architecture has the best performance.

Recall %
that \citepos{Levine:2021:DPA} certified classifier, DPA, uses an ensemble trained on partitioned training data.
In this section, we first reduce certified regression to certified classification using DPA.
We then improve the certification performance of DPA and by extension our certified regressor by using tighter, weighted analysis.
All certified regression ensembles we consider have \eqsmall{$\nModel$}~submodels denoted \eqsmall{${\decOne, \ldots, \decFin}$}, and
the ensemble decision function uses median, i.e.,~
\eqsmall{%
  ${%
    \decFunc{\xTe}
      \defeq
      \medFunc{%
        \set{%
          \decFuncOne{\xTe},
          \ldots,
          \decFuncFin{\xTe}
        }%
      }%
  }$
}. %

Since ensemble size~\eqsmall{$\nModel$} is always a user-specified hyperparameter, select odd~\eqsmall{$\nModel$}. %
For arbitrary \eqsmall{${\xTe \in \domainX}$}, let
\eqsmall{${\setVals \defeq \setbuild{\decFuncI{\xTe}}{\modIdx \in \setint{\nModel}}}$}.
Our goal remains to determine~\eqsmall{$\certBound$} -- a pointwise guarantee on the total number of training set modifications where it remains guaranteed that \eqsmall{${\decFunc{\xTe} \leq \threshold}$}.

\subsection{Partitioned Certified Regression}%
\label{sec:CertifiedGeneral:PCR}

Here, the \eqsmall{$\nModel$}~submodel regressors are \keyword{fully-independent}, meaning their training sets are disjoint, and each submodel prediction provides no direct insight into any other submodel's behavior.  %
This simple framework makes \textit{no assumptions about the submodel architecture}; the submodels may be non-parametric or parametric, deep or shallow, etc.
The only requirement is that each submodel returns a deterministic prediction given its training set and feature vector~$\xTe$.

\citeauthor{Levine:2021:DPA} enforce disjoint submodel training sets by using deterministic function~\eqsmall{$\hashTrain$} to partition training set~\eqsmall{$\trainSet$} into \eqsmall{$\nModel$}~disjoint blocks, \eqsmall{${\blockOne, \ldots, \blockT}$}.
Formally, for all \eqsmall{${\modIdx \in \setint{\nModel}}$}, submodel~\eqsmall{$\decI$}'s training set is \eqsmall{${\modTrainSetI = \blockI[\modIdx]}$}.
\newmod{Since each training instance is assigned to exactly one submodel, any training set modification can only affect one submodel.
Under the \keyword{unit-cost assumption} (Def.~\ref{def:RelatedWork:UnitCostAssumption}), each training set modification changes the corresponding submodel's prediction from \eqsmall{$\decFuncI{\xTe}$} to $\infty$ in the worst case.
Thus, perturbing a partitioned ensemble's median prediction follows Sec.~\ref{sec:WarmUp:Swap}'s unweighted swap paradigm where, as explained above, \textit{each perturbed submodel entails one training set modification}.}

Via reduction to DPA, Theorem~\ref{thm:CertifiedGeneral:PCR} directly applies Lemma~\ref{lem:Warmup:Swap} to certify unit-cost, partitioned regression's robustness under arbitrary training set insertions and deletions.

\begin{theorem}%
  \label{thm:CertifiedGeneral:PCR}
  For \eqsmall{${\xTe \in \domainX}$},
  \eqsmall{${\threshold \in \real}$},
  and deterministic function \eqsmall{$\hashTrain$} that partitions set~\eqsmall{$\trainSet$} into disjoint blocks~\eqsmall{${\blockI[1], \ldots, \blockI[\nModel]}$},
  let $\dec$ be an ensemble of \eqsmall{$\nModel$} submodels where \eqsmall{$\nModel$}~is odd, and
  each deterministic submodel \eqsmall{$\decI$} is trained on block~\eqsmall{$\blockI[\modIdx]$}.
  Define \eqsmall{${\lowerSetVals \defeq \setbuild{\decFuncI{\xTe}}{\decFuncI{\xTe} \leq \threshold}}$}.
  Given
  \eqsmall{${\decFunc{\xTe} \defeq \medFunc{\set{\decFuncI[1]{\xTe}, \ldots, \decFuncI[\nModel]{\xTe}}} \leq \threshold}$},
  if model~$\dec$ is trained on a modified \eqsmall{$\trainSet$} where the total number of inserted and deleted training instances does not exceed

  {%
    \equationSize%
    \begin{equation}%
      \label{eq:CertifiedGeneral:PCRBound}%
      \certBound
      =
        \nLower
        -
        \ceil{\frac{\nModel}{2}}
      \text{,}%
    \end{equation}%
  }%
  \noindent%
  it is guaranteed that \eqsmall{${\decFunc{\xTe} \leq \threshold}$}.
\end{theorem}

We denote this disjoint ensemble regressor as \keyword{partitioned certified regression}~(\disjointSingleMethod).
Suppl.\ Sec.~\ref{sec:App:PcrVsDpa} %
proves that when regression is used for binary classification, \disjointSingleMethod{} and DPA yield identical robustness guarantees~(\eqsmall{$\certBound$}).

\subsection{Weighted Partitioned Certified Regression}%
\label{sec:CertifiedGeneral:NuPCR}

\citeauthor{Levine:2021:DPA} only consider the maximally pessimistic unit-cost assumption.
For a feature vector~\eqsmall{$\xTe$}, it may take multiple training set insertions/deletions %
to corrupt a submodel's prediction.
For example, Theorems~\ref{thm:CertifiedKNN:FixedSize} and~\ref{thm:CertifiedKNN:DynamicSize} prove that IBL predictions are robust to multiple training set modifications.

Fixing the regressor's overall architecture, one obvious way to improve certified robustness~\eqsmall{$\certBound$} is to improve the robustness certifier.
Below, we introduce tighter analysis of each \disjointSingleMethod{} submodel's pointwise robustness so as to move beyond unit cost.
Let \eqsmall{${\covI \in \nats}$}~denote the minimum number of insertions/deletions required to change\footnote{%
  Certified robustness~$\certBound$ is the total number of training set modifications that can be made with it remaining guaranteed that ${\decTe \leq \threshold}$. In contrast, $\covI$~is minimum the number of modifications needed to \textit{perturb} submodel~$\modIdx$'s prediction enough that ${\decFuncI{\xTe} > \threshold}$.
  If $\certBound_{\modIdx}$ were the certified robustness of just \textit{submodel}~$\modIdx$, then ${\covI = \certBound_{\modIdx} + 1}$.
  $\covI$'s definition here follows related work~\citep{Ran:2021:PartialSetCoverRmax}.%
}
the submodel enough where \eqsmall{${\decFuncI{\xTe} > \threshold}$}.
By definition, if \eqsmall{${\decFuncI{\xTe} > \threshold}$} without any training set modifications, \eqsmall{${\covI = 0}$}.
When \eqsmall{${\exists_{\modIdx} \, \covI > 1}$}, better certified guarantees are possible through a weighted framework. %
Theorem~\ref{thm:CertifiedGeneral:NonUniformPCR} directly applies Lemma~\ref{lem:Warmup:NonUniformSwap}'s weighted swap paradigm to adapt \disjointSingleMethod{} (and DPA) to weighted perturbation costs.
We denote this extension \keyword{weighted partitioned certified regression}~(\disjointMultiMethod{}).

\begin{theorem}%
  \label{thm:CertifiedGeneral:NonUniformPCR}
  For \eqsmall{${\xTe \in \domainX}$},
  \eqsmall{${\threshold \in \real}$},
  and function \eqsmall{$\hashTrain$} that partitions set~\eqsmall{$\trainSet$} into disjoint blocks~\eqsmall{${\blockI[1], \ldots, \blockI[\nModel]}$},
  let $\dec$ be an ensemble of \eqsmall{$\nModel$} submodels where \eqsmall{$\nModel$}~is odd.
  Each deterministic submodel \eqsmall{$\decI$} is trained on block~\eqsmall{$\blockI[\modIdx]$} and requires at least ${\covI \in \intsNN}$~modifications to \eqsmall{$\blockI[\modIdx]$} for \eqsmall{${\decFuncI{\xTe} > \threshold}$}.
  For \eqsmall{${\covSet \defeq \setbuild{\covI}{\decFuncI{\xTe} \leq \threshold}}$},
  let \eqsmall{$\smallestCovVals$} be \eqsmall{$\covSet$}'s smallest \eqsmall{${\abs{\covSet} - \halfModels + 1}$} values.
  Given
  \eqsmall{${\decFunc{\xTe} \defeq \medFunc{\set{\decFuncI[1]{\xTe}, \ldots, \decFuncI[\nModel]{\xTe}}} \leq \threshold}$},
  if model~$\dec$ is trained on a modified \eqsmall{$\trainSet$} where the total number of inserted and deleted training instances does not exceed

  {%
    \equationSize%
    \begin{equation}%
      \label{eq:CertifiedGeneral:Thm:NonUniformPCRBound}%
      \certBound
      =
        \sum_{\covSym \in \smallestCovVals}
          \covSym
        -
        1%
      \text{,}
    \end{equation}%
  }%
  \noindent%
  it is guaranteed that \eqsmall{${\decFunc{\xTe} \leq \threshold}$}.
\end{theorem}

It can be easily shown that \disjointMultiMethod{} always yields certified robustness at least as good as \disjointSingleMethod{}.
Although proposed in the context of regression, our weighted formulation also notably improves certified classification as shown in Sec.~\ref{sec:ExpRes:Results}.

\subsection{Computational Complexity}%
\label{sec:CertifiedGeneral:Complexity}

Both \disjointSingleMethod{} and \disjointMultiMethod{} require training \eqsmall{\bigOT{}}~models.
As established by Lemmas~\ref{lem:Warmup:Swap} and~\ref{lem:Warmup:NonUniformSwap},
the computational complexity of \disjointSingleMethod{} and \disjointMultiMethod{} (resp.)\ to certify each ensemble prediction is~\eqsmall{\bigOT{}}~\citep{Blum:1973:TimeBoundsSelection} -- the same complexity as DPA.%
\footnote{Not included in \disjointMultiMethod{}'s complexity is the time to determine {${\covOne, \ldots, \covFin}$}.}
\newcommand{\overlapParagraph}[1]{%
  \vspace{4pt}%
  \noindent%
  \textbf{#1}
}
\begin{figure*}[t]
  \centering
\centering
\newcommand{\mainFontSize}{\footnotesize}

\newcommand{\blockIdStr}[1]{{\scriptsize $\blockBoldI{#1}$}}
\newcommand{\modelIdStr}[1]{$\decI[#1]$}

\newcommand{\predictionStr}[1]{\textbf{#1}}

\newcommand{\descBoxHeight}{6.75cm}

\newcommand{\modelDatasetSpacer}{0.86cm}  %
\newcommand{\trainSetSpacer}{4.00cm}

\newcommand{\mappingDescSpacer}{2.90cm}
\newcommand{\mappingDescWidth}{2.40cm}
\newcommand{\mappingDescHeight}{5.00cm}

\newcommand{\trainSetMidPoint}{0.47}
\newcommand{\blockMidPoint}{0.35}

\newcommand{\trainSetArrowOutSpacer}{1.00cm}

\newcommand{\blockBoxWidth}{2.2cm}
\newcommand{\blockArrowInSpacer}{0.85cm}
\newcommand{\blockArrowOutSpacer}{0.90cm}

\newcommand{\modelSpacer}{3.85cm}
\newcommand{\modelArrowInSpacer}{0.60cm}
\newcommand{\modelArrowOutSpacer}{0.60cm}
\newcommand{\modelBoxWidth}{1.7cm}

\newcommand{\predictionSpacer}{3.50cm}
\newcommand{\predArrowInSpacer}{0.45cm}
\newcommand{\predBoxWidth}{1.5cm}

\newcommand{\thresholdLabelSpacer}{1.50cm}

\newcommand{\predLabelDescDistance}{50pt}

\newcommand{\labelDistanceSpacer}{0pt}

\begin{tikzpicture}[
  ]

  \node (block5B) [block06] {\blockIdStr{6}};
  \node (block5A) [block05, above=0pt of block5B.before top, anchor=after bottom] {\blockIdStr{5}};
  \coordinate (Block5Mid) at ($(block5A)!\blockMidPoint!(block5B)$);  %
  \coordinate[left of=Block5Mid, node distance=\blockArrowInSpacer] (Block5ArrowIn);
  \coordinate[right of=Block5Mid, node distance=\blockArrowOutSpacer] (Block5ArrowOut);

  \node (block4B) [block07, above=\modelDatasetSpacer of block5B.before top, anchor=after bottom] {\blockIdStr{7}};
  \node (block4A) [block01, above=0pt of block4B.before top, anchor=after bottom] {\blockIdStr{1}};
  \coordinate (Block4Mid) at ($(block4A)!\blockMidPoint!(block4B)$);  %
  \coordinate[left of=Block4Mid, node distance=\blockArrowInSpacer] (Block4ArrowIn);
  \coordinate[right of=Block4Mid, node distance=\blockArrowOutSpacer] (Block4ArrowOut);

  \node (block3B) [block05, above=\modelDatasetSpacer of block4B.before top, anchor=after bottom] {\blockIdStr{5}};
  \node (block3A) [block04, above=0pt of block3B.before top, anchor=after bottom] {\blockIdStr{4}};
  \coordinate (Block3Mid) at ($(block3A)!\blockMidPoint!(block3B)$);  %
  \coordinate[left of=Block3Mid, node distance=\blockArrowInSpacer] (Block3ArrowIn);
  \coordinate[right of=Block3Mid, node distance=\blockArrowOutSpacer] (Block3ArrowOut);

  \node (block2B) [block06, above=\modelDatasetSpacer of block3B.before top, anchor=after bottom] {\blockIdStr{6}};
  \node (block2A) [block03, above=0pt of block2B.before top, anchor=after bottom] {\blockIdStr{3}};
  \coordinate (Block2Mid) at ($(block2A)!\blockMidPoint!(block2B)$);  %
  \coordinate[left of=Block2Mid, node distance=\blockArrowInSpacer] (Block2ArrowIn);
  \coordinate[right of=Block2Mid, node distance=\blockArrowOutSpacer] (Block2ArrowOut);

  \node (block1B) [block04, above=\modelDatasetSpacer of block2B.before top, anchor=after bottom] {\blockIdStr{4}};
  \node (block1A) [block02, above=0pt of block1B.before top, anchor=after bottom] {\blockIdStr{2}};
  \coordinate (Block1Mid) at ($(block1A)!\blockMidPoint!(block1B)$);  %
  \coordinate[left of=Block1Mid, node distance=\blockArrowInSpacer] (Block1ArrowIn);
  \coordinate[right of=Block1Mid, node distance=\blockArrowOutSpacer] (Block1ArrowOut);

  \coordinate (BlockDatasetMid) at ($(Block1Mid)!0.505!(Block5Mid)$);

  \node [%
    description box,
    minimum width=\blockBoxWidth,
    minimum height=\descBoxHeight,
    label={[align=center, label distance=\labelDistanceSpacer, font=\mainFontSize] Submodel \\ Training Sets}
  ] at (BlockDatasetMid) {};

  \node (block07) [
    block07,
    left=\trainSetSpacer of block4B,
  ] {\blockIdStr{7}};

  \node (block06) [block06, above=0pt of block07.before top, anchor=after bottom] {\blockIdStr{6}};
  \node (block05) [block05, above=0pt of block06.before top, anchor=after bottom] {\blockIdStr{5}};
  \node (block04) [block04, above=0pt of block05.before top, anchor=after bottom] {\blockIdStr{4}};
  \node (block03) [block03, above=0pt of block04.before top, anchor=after bottom] {\blockIdStr{3}};
  \node (block02) [block02, above=0pt of block03.before top, anchor=after bottom] {\blockIdStr{2}};

  \node (block01) [
    block01,
    above=0pt of block02.before top,
    anchor=after bottom,
    label={[align=center, label distance=4pt, font=\mainFontSize]Training set~$\trainSet$ \\ partitioned by $\hashTrain$},
  ] {\blockIdStr{1}};

  \coordinate (TrainSetMid) at ($(block01)!\trainSetMidPoint!(block07)$);  %
  \coordinate[right of=TrainSetMid, node distance=\trainSetArrowOutSpacer] (TrainSetArrowOut) ;

  \coordinate[left=\mappingDescSpacer of BlockDatasetMid]  (mappingCenter) ;
  \node [%
    minimum width=\mappingDescWidth,
    minimum height=\mappingDescHeight,
    label={[align=center, label distance=\labelDistanceSpacer, font=\mainFontSize] Block Mapping \\ by $\hashModel$}
  ] at (mappingCenter) {};

  \draw [dataset line] (TrainSetArrowOut) -- (Block5ArrowIn)  {};
  \draw [dataset line] (TrainSetArrowOut) -- (Block4ArrowIn)  {};
  \draw [dataset line] (TrainSetArrowOut) -- (Block3ArrowIn)  {};
  \draw [dataset line] (TrainSetArrowOut) -- (Block2ArrowIn)  {};
  \draw [dataset line] (TrainSetArrowOut) -- (Block1ArrowIn)  {};

  \node (model01) [submodel, right of=Block1Mid, node distance=\modelSpacer] {\modelIdStr{1}};
  \coordinate[left of=model01, node distance=\modelArrowInSpacer] (model01ArrowIn);
  \coordinate[right of=model01, node distance=\modelArrowOutSpacer] (model01ArrowOut);

  \node (model02) [submodel, right of=Block2Mid, node distance=\modelSpacer] {\modelIdStr{2}};
  \coordinate[left of=model02, node distance=\modelArrowInSpacer] (model02ArrowIn);
  \coordinate[right of=model02, node distance=\modelArrowOutSpacer] (model02ArrowOut);

  \node (model03) [submodel, right of=Block3Mid, node distance=\modelSpacer] {\modelIdStr{3}};
  \coordinate[left of=model03, node distance=\modelArrowInSpacer] (model03ArrowIn);
  \coordinate[right of=model03, node distance=\modelArrowOutSpacer] (model03ArrowOut);

  \node (model04) [submodel, right of=Block4Mid, node distance=\modelSpacer] {\modelIdStr{4}};
  \coordinate[left of=model04, node distance=\modelArrowInSpacer] (model04ArrowIn);
  \coordinate[right of=model04, node distance=\modelArrowOutSpacer] (model04ArrowOut);

  \node (model05) [submodel, right of=Block5Mid, node distance=\modelSpacer] {\modelIdStr{5}};
  \coordinate[left of=model05, node distance=\modelArrowInSpacer] (model05ArrowIn);
  \coordinate[right of=model05, node distance=\modelArrowOutSpacer] (model05ArrowOut);

  \node [%
    description box,
    minimum width=\modelBoxWidth,
    minimum height=\descBoxHeight,
    label={[align=center, label distance=\labelDistanceSpacer, font=\mainFontSize] Ensemble of \\ ${\nModel = 5}$ Submodels\vphantom{Training Sets}}
  ] at (model03) {};

  \draw [dataset line] (Block1ArrowOut) -- (model01ArrowIn)  {};
  \draw [dataset line] (Block2ArrowOut) -- (model02ArrowIn)  {};
  \draw [dataset line] (Block3ArrowOut) -- (model03ArrowIn)  {};
  \draw [dataset line] (Block4ArrowOut) -- (model04ArrowIn)  {};
  \draw [dataset line] (Block5ArrowOut) -- (model05ArrowIn)  {};

  \node (pred01) [prediction below, right of=model01, node distance=\predictionSpacer] {\predictionStr{\setValOne}};
  \coordinate[left of=pred01, node distance=\predArrowInSpacer] (pred01ArrowIn);

  \node (pred02) [prediction below, right of=model02, node distance=\predictionSpacer] {\predictionStr{\setValTwo}};
  \coordinate[left of=pred02, node distance=\predArrowInSpacer] (pred02ArrowIn);

  \node (pred03) [prediction below, right of=model03, node distance=\predictionSpacer] {\predictionStr{\setValThree}};
  \coordinate[left of=pred03, node distance=\predArrowInSpacer] (pred03ArrowIn);

  \node (pred04) [prediction below, right of=model04, node distance=\predictionSpacer] {\predictionStr{\setValFour}};
  \coordinate[left of=pred04, node distance=\predArrowInSpacer] (pred04ArrowIn);

  \node (pred05) [prediction above, right of=model05, node distance=\predictionSpacer] {\predictionStr{\setValFive}};
  \coordinate[left of=pred05, node distance=\predArrowInSpacer] (pred05ArrowIn);

  \node [%
    description box,
    minimum width=\predBoxWidth,
    minimum height=\descBoxHeight,
    label={[align=center, label distance=\labelDistanceSpacer, font=\mainFontSize] Submodel \\ Predictions\vphantom{Training Sets}}
  ] at (pred03) {};

  \draw [prediction line] (model01ArrowOut) -- (pred01ArrowIn)  {};
  \draw [prediction line] (model02ArrowOut) -- (pred02ArrowIn)  {};
  \draw [prediction line] (model03ArrowOut) -- (pred03ArrowIn)  {};
  \draw [prediction line] (model04ArrowOut) -- (pred04ArrowIn)  {};
  \draw [prediction line] (model05ArrowOut) -- (pred05ArrowIn)  {};

  \coordinate (ThresholdMid) at ($(pred04)!0.5!(pred05)$);
  \node [execute at begin node=\setlength{\baselineskip}{1.8ex}] (ThresholdLabel) [
    right of=ThresholdMid,
    node distance=\predLabelDescDistance,
    text width=30pt,
    align=center,
    font=\footnotesize,
  ] {Threshold ~~$\threshold$~~};

  \draw [->, label line] (ThresholdLabel) -- (ThresholdMid)  {};

  \coordinate (Median) at ($(pred03)$);
  \node (MedianLabel) [
    median label,
    right of=Median,
    node distance=\predLabelDescDistance,
    font=\footnotesize,
  ] {\begin{tabular}{c} Initial \\ Median \end{tabular}};
  \draw [->, label line] (MedianLabel) -- (pred03)  {};

\end{tikzpicture}
    \caption{%
    \textbf{Overlapping Certified Ensemble}: Simple visualization of the ensemble architecture for (weighted) overlapping certified regression.
    Function~\eqsmall{$\hashTrain$} partitions training set~\eqsmall{$\trainSet$} into (\eqsmall{${\nBlocks = 7}$}) blocks.
    Function~\eqsmall{$\hashModel$} defines each of the \eqsmall{${\nModel = 5}$} submodel training sets, \eqsmall{${\modTrainSetOne, \ldots, \modTrainSetI[5]}$}.
    The ensemble prediction is the median submodel prediction, i.e., \eqsmall{${\decFunc{\xTe} \defeq \medFunc{\set{\decFuncOne{\xTe}, \ldots, \decFuncFin{\xTe}}}}$}.%
  }
  \label{fig:CertifiedOverlap:Ensemble}
\end{figure*}

\section{Certified Regression Using Overlapping Training Data}%
\label{sec:CertifiedOverlap}

This section reduces certified regression to a third certified classifier, specifically \citepos{Wang:2022:DeterministicAggregation} reformulation of DPA where the submodels are trained on overlapping data.
This makes the submodels interdependent, meaning one training set modification may alter multiple submodel predictions.
Fig.~\ref{fig:CertifiedOverlap:Ensemble} visualizes an ensemble trained on overlapping training sets.
Again, $\nModel$ is the number of submodels.%
\footnote{
  In practice, for overlapping certified regression to guarantee better robustness than \bothDisjoint{} the number of submodels generally must increase by several folds over partitioned regression.%
}
Function \eqsmall{$\func{\hashTrain}{\domainZ}{\setint{\nBlocks}}$} still partitions the instance space into $\nBlocks$~disjoint blocks, where ${\nBlocks \geq \nModel}$.
Following \citeauthor{Wang:2022:DeterministicAggregation}, a second deterministic function \eqsmall{$\func{\hashModel}{\setint{\nBlocks}}{\powerSetInt{\nModel}}$} maps each training set block to one or more submodel training sets.  Formally, submodel \eqsmall{$\decI$}'s training set is
\eqsmall{%
  ${%
    \modTrainSetI%
    \defeq%
    \bigsqcup_{%
            \modIdx \in \hashModelFunc{\blockIdx}
      }
      \blockI%
  }$%
}. %
\noindent%
Let \eqsmall{${\spreadDegI \defeq \abs{\hashModelFunc{\blockIdx}}}$} denote \eqsmall{$\blockI$}'s \keyword{spread degree}, i.e.,~the number of models that use~\eqsmall{$\blockI$} during training.
Denote the \keyword{maximum spread degree} as \eqsmall{${\spreadDegMax \defeq \max\set{\spreadDegOne, \ldots, \spreadDegFin}}$}.
The ensemble's decision function is still the median submodel prediction. %

Below, we first consider certified regression on overlapping data under the unit-cost assumption.
We then improve overlapping regression by leveraging our weighted reformulation.

\subsection{Overlapping Certified Regression}%
\label{sec:CertifiedOverlap:Uniform}

Irrespective of whether the submodels are trained on disjoint or overlapping data, under the unit-cost assumption, at least \eqsmall{${\nLower - \halfModels}$} submodel predictions must exceed~\eqsmall{$\threshold$} to perturb the ensemble's median.
Observe that each submodel training set~\eqsmall{${\modTrainSetI \subset \trainSet}$} is composed of one or more dataset blocks.
Perturbing \textit{any} block in~\eqsmall{$\modTrainSetI$} is sufficient to perturb the submodel's prediction, with an optimal attacker \textit{minimizing the number of training set (block) modifications}.

If the goal were to perturb all \eqsmall{$\nModel$}~submodels, then for arbitrary block mapping function~\eqsmall{$\hashModel$}, determining the minimum number of blocks that need to be modified reduces to \keyword{minimum set cover}, which is NP\=/hard~\citep{Slavik:1997:GreedySetCover}.
Specifically, the set to cover is \eqsmall{${\lowerSubMods \defeq \setbuild{\modIdx}{\decFuncI{\xTe} \leq \threshold}}$}, i.e.,~the submodels predicting at most~\eqsmall{$\threshold$}, and
the collection of subsets is
   \eqsmall{%
     ${%
       \mathbf{S}%
       \defeq%
         \big\{%
             \setbuildDynamic{\blockIdx}%
                      {\modIdx \in \hashModelFunc{\blockIdx}}
             \bm:  %
             \decFuncI{\xTe} \leq \threshold%
         \big\}%
     }$%
   },
which contains the dataset blocks each relevant submodel is trained on.

However, recall that for median perturbation under unweighted swaps, we only need to perturb (i.e.,~cover) \eqsmall{${\nLower - \halfModels}$} submodels -- not all of them.
Therefore, rather than mapping to set cover, our problem reduces to the related problem of \keyword{partial set cover}, where only a constant fraction of the instances (i.e.,~submodels) need to be covered.
For arbitrary block mapping function~\eqsmall{$\hashModel$}, Lemma~\ref{lem:CertifiedOverlap:Uniform:Hardness} below establishes that finding the optimal~\eqsmall{$\certBound$} here is NP\=/hard~\citep{Slavik:1997:PartialDCover,Elomaa:2010:GreedyPartialCover}.

\begin{lemma}%
  \label{lem:CertifiedOverlap:Uniform:Hardness}
  Finding optimal certified robustness~\eqsmall{$\certBound$} for overlapping certified regression is NP\=/hard.
\end{lemma}

  Although our problem is NP\=/hard, it is polynomial-time approximable.
Specifically, the approximation uses the famous greedy set-cover algorithm where in each iteration, the subset (training block~\eqsmall{$\blockI$}) covering the most remaining elements (submodels) is selected~\citep{Chvatal:1979:GreedySetCover,Slavik:1997:PartialDCover}.
Let \eqsmall{$\greedyBound$} denote the bound found by this greedy method, and define
\eqsmall{${\swapBound \defeq \nLower - \halfModels}$}.
Then for the non-naive case where \eqsmall{${\swapBound \geq 2}$}, %

{%
  \equationSize%
  \begin{equation}\label{eq:CertifiedOverlap:Greedy:ApproximationRatio}
      \certBound
        \geq
        \ceil{%
            \frac{\greedyBound}%
                 {\min\set{%
                    \harmonic{\spreadDegMax},
                    \ln \swapBound - \ln\ln \swapBound + 3 + \ln\ln 32 - \ln 32
                   }%
                 }
        }%
      \text{,}%
  \end{equation}%
}%
\noindent%
where \eqsmall{$\harmonic{\spreadDegMax}$} is the \eqsmall{$\spreadDegMax$}\=/th harmonic number.
This bound
follows directly from partial set cover \keyword{approximation factor} analysis \citetext{\citealp[Thm.~4]{Slavik:1997:GreedySetCover}; \citealp[Thm.~3]{Slavik:1997:PartialDCover}}.%
\footnote{%
  Eq.~\eqref{eq:CertifiedOverlap:Greedy:ApproximationRatio}'s bound is tighter (often significantly so) than the much more famous approximation factor, ${\harmonic{\swapBound}}$, of \citet{Johnson:1974:GreedySetCover} and \citet{Lovasz:1975:RatioOptimalCovers}.%
}
\citet{Slavik:1997:GreedySetCover} shows that the difference between this approximation factor's overall lower and upper bound is only roughly~1.1, meaning this general approximation is quite good overall.

However, in most cases, the performance advantage of overlapping versus disjoint unit-cost regressors is small enough that the greedy optimality gap wipes out all gains.
Instead, we rely on Fig.~\ref{fig:CertifiedGeneral:ILP}'s integer linear program~(ILP) to bound \eqsmall{$\certBound$} in the overlapping case.\footnote{Fig.~\ref{fig:CertifiedGeneral:ILP} jointly formulates calculating $\certBound$ under unit and weighted costs.}
This ILP is directly adapted from standard partial set-cover, where for unit costs \eqsmall{${\forall_{\modIdx} \, \covI = 1}$}.

While the ILP is still NP\=/hard in the worst case, modern LP solvers often find a (near) optimal solution in reasonable time (e.g.,~a few seconds)~\citep{Gurobi}.
In cases where finding true robustness~\eqsmall{$\certBound$} is computationally expensive,
these solvers generally return guaranteed bounds on~\eqsmall{$\certBound$} that are (much) better than the greedy approximation~\citep{Vanderbei:2014:LinearProgramming}.%
\footnote{Sec.~\ref{sec:ExpRes}'s experiments use a fixed time limit to ensure tractability.}
We refer to this unit-cost, ILP-based approach as \keyword{overlapping certified regression}~(\overlapSingleMethod).

\subsection{Weighted Overlapping Certified Regression}%
\label{sec:CertifiedOverlap:NonUniform}

Recall that Sec.~\ref{sec:CertifiedGeneral:NuPCR} improves certified regressor \disjointSingleMethod{} by reformulating DPA so as not to be restricted by the unit-cost assumption.
Here, we follow the same approach of improving certified regressor \overlapSingleMethod{} by generalizing \citepos{Wang:2022:DeterministicAggregation} certified classifier to non-unit costs.

As with \disjointMultiMethod{} earlier, \eqsmall{${\covI > 1}$} entails that submodel \eqsmall{$\decI$}'s training set must be modified at least \eqsmall{$\covI$}~times for \eqsmall{${\decFuncI{\xTe} > \threshold}$}.
This prevents weighted overlapping regression from applying partial set cover since each submodel~\eqsmall{$\decI$} now has a \textit{coverage} requirement.
Instead, \keyword{partial set multicover}~(PSMC) generalizes partial set cover to support coverage requirements \eqsmall{${\covI \geq 0}$}~\citep{Shi:2019:ApproximatingPartialSetMulticover,Ran:2021:PartialSetCoverRmax}, %
and we adapt PSMC to weighted, overlapping regression.
PSMC, and by extension our task, is provably hard. %

\begin{corollary}%
  \label{col:CertifiedOverlap:NonUniform:Hardness}
  Finding the optimal certified robustness~\eqsmall{$\certBound$} for weighted overlapping certified regression is NP\=/hard.
\end{corollary}

PSMC is far less studied than (partial) set cover.
PSMC is polynomial-time approximable -- albeit with worse known bounds than partial set cover.
\citet{Ran:2021:PartialSetCoverRmax} provide the best-known PSMC bounds;
their method is much more complicated than greedy partial set cover and relies on a reduction to another NP\=/hard problem, minimum densest subcollection.
Let $\greedyBound$ be the solution generated by \citeauthor{Ran:2021:PartialSetCoverRmax}'s algorithm, then

{%
  \equationSize%
  \begin{equation}%
    \label{eq:CertifiedOverlap:NonUniform:ApproximationRatio}
    \certBound
      \geq
      \ceil{%
        \frac{\greedyBound}%
             {4\lg \nModel \harmonic{\spreadDegMax}\ln\swapBound + 2\lg \nModel \sqrt{\nModel}}%
      }%
    \text{.}%
  \end{equation}%
}%
Like with unweighted overlapping regression, Eq.~\eqref{eq:CertifiedOverlap:NonUniform:ApproximationRatio}'s approximation factor is large enough that it usually wipes out the performance gains derived from weighted costs.
Instead, we use Fig.~\ref{fig:CertifiedGeneral:ILP}'s ILP to bound~\eqsmall{$\certBound$} %
in accordance with Lem.~\ref{lem:Warmup:NonUniformSwap}.
In the ILP,
${\isMultiVar = 1}$ in the weighted case and 0~otherwise.
Hence, at least \eqsmall{${(\nLower - \halfModels + 1)}$} submodels must be covered (i.e.,~perturbed) in the weighted case.
Following Eq.~\eqref{eq:WarmUp:NonUniformSwap:Bound}, sum \eqsmall{${\sum_{\blockIdx = 1}^{\nBlocks} \dsBlockVarI}$} is decremented by one in the ILP.

We refer to this overlapping ILP-based approach as \keyword{weighted overlapping certified regression}~(\overlapMultiMethod).

\begin{figure}[t]
  \centering
  {%
    \equationSize%
\newcommand{\ilpLineTerm}{,\quad}

\begin{mini!}
  {}{\certBound=\sum_{\blockIdx = 1}^{\nBlocks} \dsBlockVarI - \isMultiVar}{}{}
  \addConstraint{\lowerSubMods}%
                {= \setbuild{\modIdx}{\decFuncI{\xTe} \leq \threshold} \ilpLineTerm}%
                {}
  \addConstraint{\covMax}%
                {= \max\setbuild{\covI}{\modIdx \in \setint{\nModel}}}%
                {}
  \addConstraint{\isMultiVar}%
                {= \ind{\covMax > 1}}
                {}
  \addConstraint{\sum_{\modIdx \in \lowerSubMods} \modelVarI}%
                {\geq \nLower - \halfModels + \isMultiVar \ilpLineTerm}%
                {\text{(Median perturb.)}}
  \addConstraint{\covI \modelVarI}%
                {\leq \sum_{\blockI \subseteq \modTrainSetI} \dsBlockVarI \ilpLineTerm}%
                {\modIdx \in \lowerSubMods}
  \addConstraint{\modelVarI \in \set{0,1}}%
                {\ilpLineTerm}%
                {\modIdx \in \lowerSubMods}
  \addConstraint{\dsBlockVarI \in \set{0, \ldots, \covMax}}%
                {\ilpLineTerm}%
                {\blockIdx \in \setint{\nBlocks}}
\end{mini!}
  }%
  \caption{%
      \textbf{Overlapping Certified Regression Integer Linear Program}:
      Adapted from the partial set (multi)cover integer linear program.
      Calculates certified robustness~$\certBound$ for both \overlapSingleMethod{} and \overlapMultiMethod{}
      with indicator variable~$\isMultiVar$ adjusting the program to account for weighted costs.
      For arbitrary feature vector~$\xTe$,
      \eqsmall{$\lowerSubMods$} is the set of submodels that predict \eqsmall{${\decFuncI{\xTe} \leq \threshold}$}.
      Variable \eqsmall{$\dsBlockVarI$} contains the number of modifications made to training set block~\eqsmall{$\blockI$}.
      Binary variable \eqsmall{${\modelVarI = 1}$} if submodel \eqsmall{$\decI$} has been sufficiently modified for \eqsmall{${\decFuncI{\xTe} > \threshold}$} and 0~otherwise.%
    }
  \label{fig:CertifiedGeneral:ILP}
\end{figure}

\subsection{Computational Cost}

See suppl.\ Sec.~\ref{sec:App:MoreExps:GurobiTime} for an empirical evaluation and extended discussion of the \overlapSingleMethod{} \& \overlapMultiMethod{} ILP execution time.

\section{Certifying Any Model Beyond Unit Cost}%
\label{sec:BeyondUnitCost}

The preceding sections describe the benefits of having more robust ensemble components (i.e., \eqsmall{${\covI > 1}$}) but do not address how to find \eqsmall{$\covI$}.
Apart from IBLs and ensembles, the two methods we focus on in this work, we know of no general method for computing insertion/deletion robustness efficiently.
We attribute this scarcity to the task's difficulty.
Nonetheless, we believe this work shows that certification beyond unit cost merits future study.  %
This section explores certifying beyond unit cost from two perspectives.
First, we consider the obvious idea of combining IBLs with ensembles and explain why that performs poorly.
Next, we propose a simple, general approach to certify any (sub)model beyond unit cost, albeit with a (slightly) more restricted threat model.

\subsection{Combining Instance-Based Learners \& Ensembles}%
\label{sec:BeyondUnitCost:CombiningIBLsAndEnsembles}

The points raised below apply to both fixed-population and region-based IBLs.
We exclusively discuss \knnMethod{} here with the extension to other certified IBLs straightforward.

In practice, %
function~\eqsmall{$\hashTrain$} partitions instance space~\eqsmall{$\domainZ$} uniformly at random (u.a.r.) into $\nBlocks$~approximately equal-sized regions.
For simplicity and w.l.o.g., consider an ensemble of \knnMethod{} submodels trained on disjoint subsets where ${\nModel = \nBlocks}$.

\newcommand{\kNeighMain}{\kNeigh'}

Let \eqsmall{$\kNeighMain$} and \eqsmall{$\certBound$} denote the neighborhood size and certified robustness (resp.)\ of a \knnMethod{} model trained on i.i.d.\ training set~\eqsmall{$\trainSet$}.
If \eqsmall{$\trainSet$} is partitioned u.a.r.\ to train $\nModel$ \knnMethod{} submodels each with \eqsmall{${\kNeigh \approx \frac{\kNeighMain}{\nModel}}$}, then each submodel's expected robustness is roughly~\eqsmall{$\frac{\certBound}{\nModel}$}.
In the best case for the defender (\eqsmall{${\forall_{\modIdx}\, \decFuncI{\xTe} \leq \threshold}$}), an adversary only needs to perturb at most \eqsmall{$\halfModels$}~submodels.
Combining the above with
Theorem~\ref{thm:CertifiedGeneral:NonUniformPCR} for \disjointMultiMethod{},
this \knnMethod{} ensemble's expected certified robustness is approximately %

{%
  \equationSize%
  \begin{equation}%
    \label{eq:BeyondUnitCost:CombiningIBLsAndEnsembles}%
    \halfModels
    \left(
      \frac{\certBound}%
           {\nModel}
    \right)
    -%
    1%
    =
    \frac{\certBound}%
         {2}
    +
    \frac{\certBound}%
         {2\nModel}
    -
    1
    <%
    \certBound%
    \text{.}%
  \end{equation}%
}%
\noindent%
As ${\nTr,\,\nModel \rightarrow \infty}$, then by Eq.~\eqref{eq:BeyondUnitCost:CombiningIBLsAndEnsembles}, a \knnMethod{} ensemble's expected robustness \textit{decreases by 50\% versus the single \knnMethod{} model baseline}.
Intuitively, for ensembles, an adversary only needs to directly attack about half of the submodels and by extension half of the training data.
In contrast, when there is only a single \knnMethod{} model trained on all of~$\trainSet$, the adversary must attack the whole training set.

\subsection{Certifying Non-Unit Costs by Construction}
\label{sec:BeyondUnitCost:CertifyingByConstruction}

Since IBLs are a poor candidate to marry with ensembles, we need an alternative approach to certify a model's robustness beyond \eqsmall{${\covSym = 1}$}.
Given the dearth of existing methods (known to us), we fill in the gap and propose a simple, general-purpose method to \textit{certify robustness against arbitrary deletions}.

To be clear, this is a (slightly) restricted version of the full threat model considered so far, which allows arbitrary insertions and deletions.
Nonetheless, this restricted threat model still has broad applicability.
For example, an adversary may only be able to insert poisoned instances into a training set but not delete clean ones~\citep{Chen:2017:Targeted,Liu:2017:RobustLinearRegression,Shafahi:2018:PoisonFrogs,Wallace:2021,Hammoudeh:2022:GAS}.
Motivated by \citepos{Cook:1982:InfluenceRegression} classic \keyword{case deletion diagnostics},
we use a constructive proof to certify a (sub)model's pointwise robustness under instance deletions.
Consider training ${(\nTr + 1)}$~deterministic models -- one model using full set~\eqsmall{${\trainSet = \set{(\xI, \yI)}_{\trIdx = 1}^{\nTr}}$} and another $\nTr$~models on each of the leave-one-out subsets~\eqsmall{${\trainSet \setminus (\xI, \yI)}$} for all~${\trIdx \in \setint{\nTr}}$.
If all ${(\nTr + 1)}$~trained models make the same prediction (e.g.,~a value not exceeding~$\threshold$ for some {$\xTe$}), then by construction, the model trained on all of~\eqsmall{$\trainSet$} has, at minimum, ${\covSym = 2}$ for arbitrary deletions.
Lemma~\ref{lem:BeyondUnitCost:CertifyingDeletions} generalizes the above for an arbitrary number of deletions ${\covSym < \nTr}$.

\begin{lemma}
  \label{lem:BeyondUnitCost:CertifyingDeletions}
  \newcommand{\trSubLem}{\trSub'}
  For
  \eqsmall{${\xTe \in \domainX}$}, %
  training set \eqsmall{$\trainSet$} where \eqsmall{$\powerSet{\trainSet}$} is its power set,
  \eqsmall{${\covSym \in \setint{\abs{\trainSet} - 1}}$},
  and \eqsmall{${\threshold \in \real}$},
  denote a deterministic model trained on subset \eqsmall{${\trSub \subseteq \trainSet}$} as \eqsmall{$\decSubset{\trSub}$}.
  Given
  \,\eqsmall{${\forall_{\trSubLem \in \powerSet{\trainSet}}\, \abs{\trSubLem} < \covSym \implies \decSubsetFunc{\trainSet \setminus \trSubLem}{\xTe} \leq \threshold}$},
  then for any \eqsmall{${\trainSetPerturb \subset \trainSet}$},
  if \eqsmall{${\decSubsetFunc{\trainSetPerturb}{\xTe} > \threshold}$} then at least \eqsmall{$\covSym$} instances from~\eqsmall{$\trainSet$} were deleted in \eqsmall{$\trainSetPerturb$}.

\end{lemma}

A strength of Lemma~\ref{lem:BeyondUnitCost:CertifyingDeletions} is its flexibility; it can be adapted to any model class, including both classifiers and regressors.
Its clear limitation is its computational complexity.

\vspace{5pt}%
\noindent%
\textit{Computational Complexity}:
Certifying ${\covSym > 1}$ requires training \eqsmall{$\bigO{\nTr^{(\covSym - 1)}}$}~models; this is a one-time, amortized cost.%
\footnote{%
  It may be possible to train just one model using full training set~$\trainSet$ and then apply certified machine unlearning~\citep{Guo:2020:CertifiedRemoval,Brophy:2021:MachineUnlearningForests} to get the remaining models.%
}

Consider separately the cost to certify each prediction.
During inference, the \eqsmall{$\bigO{\nTr^{(\covSym - 1)}}$}~models are checked.
While this may be problematic in some cases, it should be contextualized.
Recall that Sec.~\ref{sec:CertifiedKNN} explores IBLs like \knn{}, which have inference complexity~\eqsmall{\bigOn}.
Therefore, our method to certify ${\covSym = 2}$ has the same time complexity \textit{during inference} as \knn{}.

\subsection{More Submodels vs.\ Weighted Costs}
\label{sec:BeyondUnitCost:MoreModelsVsWeighted}

Increasing submodel count~\eqsmall{$\nModel$} and using weighted costs are partially conflicting approaches to increase~\eqsmall{$\certBound$}.
A natural question is which of the two approaches yields better certified robustness.
Above, we explain why increasing~\eqsmall{$\nModel$} is a poor strategy for IBLs.
For ensembles, increasing \eqsmall{$\nModel$} generally means that each submodel is trained on fewer data.

As an intuition, consider when \eqsmall{${\forall_{\modIdx}\, \covI = 2}$}.
For a unit-cost ensemble to certify equivalent~\eqsmall{$\certBound$}, submodel count~\eqsmall{$\nModel$} must about double, and
each submodel is trained on 50\% fewer data, which can significantly degrade submodel performance.
In contrast, weighted costs with ${\covSym = 2}$ reduces submodel training set sizes by~1 (Lem.~\ref{lem:BeyondUnitCost:CertifyingDeletions}).
By training weighted submodels on much more data, weighted submodels can outperform submodels from ensembles with larger~\eqsmall{$\nModel$}.
This improved submodel performance can in turn improve certified robustness.%

\newcommand{\expResParagraph}[1]{%
  \vspace{5pt}%
  \noindent%
  \textbf{#1}
}
\newcounter{TakeawayCounter}
\setcounter{TakeawayCounter}{0}
\newcommand{\ResTakeaway}[1]{%
  \vspace{5pt}%
  \stepcounter{TakeawayCounter}%
  \noindent%
  \textbf{Takeaway \#\theTakeawayCounter}: \textit{#1}
}

\section{Evaluation}%
\label{sec:ExpRes}

This section evaluates our five primary certified regressors: \knnM{} certified regression~(\knnMethod), %
partitioned certified regressors \disjointSingleMethod{} \& \disjointMultiMethod{}
as well as overlapping certified regressors \overlapSingleMethod{} \& \overlapMultiMethod{}.
Additional experimental results are in the supplement, including full \knnMethod{} certification plots~(\ref{sec:App:MoreExps:FullKNN}), alternate $\threshold$ value analysis~(\ref{sec:App:MoreExps:AltThresholds}), model training times~(\ref{sec:App:MoreExps:TrainingTime}), ILP execution time (\ref{sec:App:MoreExps:TrainingTime}), etc.

\revised{%
  To the extent of our knowledge, we propose the first pointwise certified regression methods that make no assumptions about the test distribution or model architecture.
  Without a clear baseline, we compare our five methods against each other. %
  As a reference on the clean-data performance, we report each dataset's ``\textit{uncertified}'' (non-robust) accuracy.
}

\subsection{Experimental Setup}%
\label{sec:ExpRes:Setup}

Due to space, most evaluation setup details (e.g.,~hyperparameters) are deferred to suppl.\ Sec.~\ref{sec:App:ExpSetup} with a brief summary below.
For each experiment in this section, at least ten trials were performed.
To improve readability, we only report the mean values below with variances in suppl.\ Sec.~\ref{sec:App:MoreExps:DetailedResults}.

\expResParagraph{Dataset Configuration}
Each (sub)model is trained on \eqsmall{$\frac{1}{\datasetDiv}$}\=/th of the training data, where ${\datasetDiv \in \intsPos}$.
For \knnMethod{}, always ${\datasetDiv = 1}$.
For our four ensemble-based methods \bothDisjoint{} and \bothOverlap{}, $\datasetDiv$~can significantly affect the ensemble's accuracy and best-case certified robustness~(\eqsmall{$\certBound$}).
As such, for each dataset, we report results with three different $\datasetDiv$ values.
For all ensembles, function~\eqsmall{$\hashTrain$} partitions training set~$\trainSet$ u.a.r.%
\footnote{%
  Each dataset's largest $\datasetDiv$~value maximized the ensembles' certified robustness~($\certBound$).
  For each dataset, we also report small and medium $\datasetDiv$~values.
  In practice, $\datasetDiv$ should be as small as possible while guaranteeing sufficient robustness given each application's maximum anticipated poisoning rate.
}

For our partitioned regressors \bothDisjoint{}, $\trainSet$ is split into $\datasetDiv$~blocks, with ${\nModel = \datasetDiv}$.
For our overlapping regressors \bothOverlap{}, we followed \citepos{Wang:2022:DeterministicAggregation} overlapping certified classifier evaluation.
Specifically, $\trainSet$ is partitioned into ${\datasetDiv \spreadDegreeSym}$~blocks u.a.r.
All blocks have fixed spread degree~${\spreadDegreeSym > 1}$ (see Tab.~\ref{tab:ExpRes:DatasetInfo}), and \eqsmall{$\hashModel$} assigns blocks to submodels at random.
Hence,
each overlapping ensemble necessarily has ${\nModel = \datasetDiv \spreadDegreeSym}$ submodels.

\expResParagraph{Submodel Architectures}
To demonstrate their generality, our ensemble methods use two different submodel architectures, namely ridge regression and \xgb{}~\citep{Chen:2016:XGBoost} gradient-boosted forests.
Model determinism is enforced via a fixed random seed.
Below, we report whichever submodel architecture performed the best on a held-out validation set.

\expResParagraph{Evaluation Metric}
For each test instance~\eqsmall{${(\xTe, \yTe)}$}, our goal is to determine the largest pointwise certified robustness~\eqsmall{$\certBound$} that guarantees \eqsmall{${\lThreshold \leq \decFunc{\xTe} \leq \uThreshold}$}.
Throughout this evaluation, \eqsmall{${\lThreshold \defeq \yTe - \threshold}$} and \eqsmall{${\uThreshold \defeq \yTe + \threshold}$}.  %
These bounds are w.r.t.\ each test example's \textit{true target value}~$\yTe$, \underline{not} predicted value~\eqsmall{$\decFunc{\xTe}$}.
Therefore, a large certified robustness~\eqsmall{$\certBound$} means that the \textit{prediction is both accurate and stable}.
Here, \keyword{error threshold}~$\threshold$ may be a specific fraction (e.g.,~15\%) of each instance's target value~$\yTe$ or a fixed value for the entire dataset (see Table~\ref{tab:ExpRes:DatasetInfo}). %
In practice, the appropriate $\threshold$ value is application specific.

\newcommand{\certAccVal}{\psi}
Our evaluation metric is \keyword{certified accuracy}, which is the fraction of instances with robustness \eqsmall{${\certBound \geq \certAccVal}$} for \eqsmall{${\certAccVal \in \nats}$}.
In each trial, we calculated the certified robustness~(\eqsmall{$\certBound$}) for at least 100~random test instances and report the mean certified accuracy across all trials.
See suppl.\ Sec.~\ref{sec:App:MoreExps:DetailedResults} for the certified accuracy variance.
Note that existing certified classifiers were previously evaluated using certified accuracy~\citep{Jia:2022:CertifiedKNN,Levine:2021:DPA,Wang:2022:DeterministicAggregation} with ${\threshold = 0}$, i.e.,~the predicted label must match true label~\eqsmall{$\yTe$}.

\expResParagraph{Datasets}
Our certified regressors are evaluated on six datasets: five regression and one binary classification.
Like previous work~\citep{Brophy:2022:TreeInfluence}, the datasets are preprocessed where all categorical features are transformed into one-hot encodings.
Table~\ref{tab:ExpRes:DatasetInfo} summarizes each dataset's key attributes, including its size, error threshold~($\threshold$), ensemble submodel architecture, etc. A brief description of each dataset is below.

Ames~\citep{DeCock:2011:AmesHousing} and Austin~\citep{AustinHousingDataset} estimate home prices in two American cities.
Diamonds~\citep{DiamondsDataset} predicts a diamond's price based on attributes such as cut, color, carat, etc.
Weather~\citep{Malinin:2021:Shifts} estimates ground temperature (in degrees Celsius) using date, time, and longitude/latitude information.
For computational efficiency, Weather's size was downsampled by $10\times$ u.a.r.
Life~\citep{LifeExpectancyDataset} estimates life expectancy (in years) using epidemiological and other national statistics.
Spambase~\citep{SpambaseDataset} is a binary classification dataset where emails are labeled as either spam or ham.
Spambase's positive training prior is 39\%.

\expResParagraph{Certifying ${\covSym > 1}$}
For our two weighted methods, \disjointMultiMethod{} and \overlapMultiMethod{}, our evaluation attempts to certify each submodel's robustness against deletions up to ${\covSym = 2}$.

\begin{table}[t]
  \centering
  \caption{%
      \textbf{Evaluation Dataset Summary}:
      Training set size~($\nTr$),
      data dimension,
      overlapping spread degree~($\spreadDegreeSym$),
      error threshold~($\threshold$),
      and submodel architecture for the six datasets.
      Error thresholds that are a percentage of each instance's true target value are denoted~${X\% \cdot \y}$.
      Alternate $\threshold$ values are evaluated in suppl.\ Sec.~\ref{sec:App:MoreExps:AltThresholds}.
  }\label{tab:ExpRes:DatasetInfo}
  {%
    \tableSize%
\newcommand{\fifteenY}{$15\% \cdot \y$}
\newcommand{\rDeg}{XX}

\newcommand{\dsName}[2]{#1}

\newcommand{\tabXGB}{\xgb{}}
\newcommand{\clsThreshold}{0}

\begin{tabular}{@{}lrrrrr@{}}
  \toprule
  Dataset   & Size~($\nTr$) & Dim. & Deg.~($\spreadDegreeSym$) & Error~($\threshold$)  & Submodel \\
  \midrule
  \dsName{Ames}{DeCock:2011:AmesHousing}
            & 2.6k    & 253   & 17     & \fifteenY{}      & \tabXGB{}  \\
  \dsName{Austin}{AustinHousingDataset}
            & 12k     & 315   & 13     & \fifteenY{}      & \tabXGB{}  \\
  \dsName{Diamonds}{DiamondsDataset}
            & 48k     & 26    & 9      & \fifteenY{}      & Ridge      \\
  \dsName{Weather}{Malinin:2021:Shifts}
            & 308k    & 140   & 5      & $3^{\circ}$\,C   & Ridge      \\
  \dsName{Life}{LifeExpectancyDataset}
            & 2.6k    & 204   & 13     & 3~years          & \tabXGB{}  \\
  \dsName{Spambase}{SpambaseDataset}
            & 4.1k    & 57    & 17     & \clsThreshold{}  & Ridge      \\
  \bottomrule
\end{tabular}
  }%
\end{table}

\subsection{Analyzing the Certified Accuracy}%
\label{sec:ExpRes:Results}

Figure~\ref{fig:ExpRes:CertificationBoundTrend} visualizes our methods' mean certified accuracy for the six datasets.
Due to space, the corresponding numerical values, including variance, appear in Sec.~\ref{sec:App:MoreExps:DetailedResults}.
Below, we briefly summarize the experiments' primary takeaways.

\ResTakeaway{Both our ensemble and IBL regressors certify non-trivial fractions of the training set.}
For the Ames and Life datasets, \overlapMultiMethod{} certifies 50\% of test predictions up to 1\% training set corruption.
Similarly, \knnMethod{} certifies 30\% of predictions on Ames up to 4\%~corruption.
These certified guarantees are without explicit assumptions about the data distribution or, in the case of ensembles, the submodel's architecture.
For other datasets, we certify predictions up to hundreds or thousands of training set modifications.

\ResTakeaway{Ensemble regressors have better peak performance.}
Across all six datasets, the ensemble-based methods all had better peak certified accuracy than \knnMethod{}.
The performance gap was as large as 3.5$\times$ and is not primarily due to feature dimension as \knnMethod{} performed worst on Diamonds, which has the smallest dimension by far.

\ResTakeaway{\overlapMultiMethod{} achieves the largest certified robustness~(\eqsmall{$\certBound$}) amongst the ensemble methods.}
This is observed using each dataset's largest $\datasetDiv$ value.
For all six datasets, there is a (significant) gap between \overlapMultiMethod{} (\protect\pgfPlotsLegendRef{\ref{leg:ExpRes:Bound:WOCR}}) and our second-best ensemble method, \disjointMultiMethod{} (\protect\pgfPlotsLegendRef{\ref{leg:ExpRes:Bound:WPCR}}).

\ResTakeaway{\knnMethod{} achieves the largest certified robustness.}
Although \knnMethod{} certifies (far) fewer instances than the ensembles, for instances that it can certify, its maximum robustness~\eqsmall{$\certBound$} is generally far larger than that of \overlapMultiMethod{}.
For example with Weather, \knnMethod{}'s maximum~\eqsmall{$\certBound$} is 5$\times$ larger than \overlapMultiMethod{}'s.
Suppl.\ Sec.~\ref{sec:App:MoreExps:FullKNN} best visualizes this trend in its plots of \knnMethod{}'s full certified accuracy.

\ResTakeaway{\overlapMultiMethod{} achieves state-of-the-art certified accuracy for binary classification.}
While regression is this work's primary focus, recall
that binary classification can be solved by a regressor.
For binary classification, \knnMethod{}'s \eqsmall{$\certBound$} is identical to \citepos{Jia:2022:CertifiedKNN} \knn{} classifier; \disjointSingleMethod{} certifies equivalent robustness as DPA, and \overlapSingleMethod{} very closely approximates \citepos{Wang:2022:DeterministicAggregation} overlapping method.
Observe that \overlapMultiMethod{} outperforms the unweighted ensembles and \knnMethod{} on Spambase's~\citep{SpambaseDataset} two largest $\datasetDiv$ values.
Note that Spambase's maximum $\datasetDiv$ value cannot be increased further without severely degrading submodel performance.
This provides empirical evidence for Sec.~\ref{sec:BeyondUnitCost:MoreModelsVsWeighted}'s claim that a weighted strategy can outperform increasing submodel count~$\nModel$.

\ResTakeaway{$\datasetDiv$ can significantly affect certified accuracy.}
Previous certified classifier evaluations~\citep{Jia:2022:CertifiedKNN,Wang:2022:DeterministicAggregation,Levine:2021:DPA} under-explore $\datasetDiv$'s role.
Those works primarily evaluate vision datasets where the training data is supplemented by
(1)~using a pre-trained model to extract much better features~\citep{Jia:2021:CertifiedBaggingRobustness,Jia:2022:CertifiedKNN}
or (2)~using vision data augmentation~\citep{Levine:2021:DPA,Wang:2022:DeterministicAggregation}.
For the tabular datasets evaluated here, such options are not as available.

Without such augmentation, increasing~$\datasetDiv$ can significantly degrade an ensemble's peak certified accuracy.
As an example, the ensembles' peak certified accuracy can decline by up to 28\% between training a model on all of~$\trainSet$ versus a dataset's maximum $\datasetDiv$~value (compare to uncertified accuracy \pgfPlotsLegendRef{\ref{leg:ExpRes:Bound:Uncertified}} in Fig.~\ref{fig:ExpRes:CertificationBoundTrend}).
Therefore, when thinking about certified classifiers and regressors, always consider the potential benefits of external (clean) data augmentation.
For instance, in our experiments, XGBoost certified ensembles' accuracy improved by multiple percent when using \textit{mixup} data augmentation~\citep{Zhang:2018:Mixup}.%

\clearpage
\newpage
\newgeometry{margin=0.75in, top=0.65in}

\newcommand{\boundTrend}[9]{%
  \centering
  \pgfplotstableread[col sep=comma] {plots/data/#1}\thedata%
  \begin{tikzpicture}
    \begin{axis}[
      smooth,
      width={#2},%
      height={\BoundTrendHeight},%
      xmin={0},%
      xmax={#3},%
      xtick distance={#4},
      minor x tick num={3},
      x tick label style={font=\boundFontSize,align=center},%
      xlabel={#5},
      xmajorgrids,
      axis x line*=bottom,  %
      ymin=0,
      ymax={#6},
      ytick distance={#7},
      minor y tick num={3},
      y tick label style={font=\boundYTickFontSize,align=center},%
      ylabel style={font=\boundYAxisFontSize,align=center},%
      ylabel={#8},
      ymajorgrids,
      axis y line*=left,  %
      mark size=0pt,
      #9,
      ]
      \shadePlot{Baseline}{baseline trend line}{overlap line fill}

      \shadePlot{KNN}{knn trend line}{overlap line fill}

      \shadePlot{Disjoint}{single disjoint line}{disjoint line fill}
      \shadePlot{Overlap}{single overlap line}{overlap line fill}

      \shadePlot{Multi-Disjoint}{multi disjoint line}{disjoint line fill}
      \shadePlot{Multi-Overlap}{multi overlap line}{overlap line fill}

    \end{axis}
  \end{tikzpicture}
}

\newcommand{\shadePlot}[3]{%
  \addplot[#2] table [x=X, y=#1-Mean] \thedata;
}

\newcommand{\boundDatasetDivCaption}[1]{${\datasetDiv = #1}$}
\newcommand{\boundLegendSpacer}{\vspace{6pt}}

\newcommand{\baseXlabelStr}{Certified Robustness~($\certBound$)}
\newcommand{\certifiedAccStr}{Certified Acc.\ (\%)}

\newcommand{\boundXLabel}[1]{\footnotesize\baseXlabelStr{}, ${\datasetDiv = #1}$}
\newcommand{\boundYLabel}[1]{\textbf{#1} \\[1.03ex] \certifiedAccStr}
\newcommand{\boundSubcaption}[2]{}
\newcommand{\BoundTrendHeight}{1.50in}
\newcommand{\boundYAxisFontSize}{\scriptsize}
\newcommand{\FirstMiniWidth}{0.362\textwidth}
\newcommand{\FirstPlotWidth}{0.98\textwidth}
\newcommand{\RestMiniWidth}{0.308\textwidth}
\newcommand{\RestPlotWidth}{\textwidth+30pt}
\newcommand{\boundDatasetSpacer}{\vspace{04pt}}

\newcommand{\plotHorizontalSpacer}{\hfill}

\newcommand{\boundYTickDist}{20}

\newcommand{\SkipZeroTick}{yticklabels={,,20,40,60,80},}
\newcommand{\SkipZeroTickFifteen}{yticklabels={,,15,30,45,60},}
\newcommand{\NoYTicks}{yticklabels={,,,},}

\begin{figure*}[t]
  \centering
\newcommand{\oursText}{~\revised{\scriptsize(ours)}}
\newcommand{\boundLegendFontSize}{\footnotesize}
\begin{tikzpicture}
  \begin{axis}[%
      hide axis,  %
      no marks,
      xmin=0,  %
      xmax=1,
      ymin=0,
      ymax=1,
      scale only axis,
      width=1mm, %
      legend cell align={left},              %
      legend style={font=\boundLegendFontSize},
      legend columns=6,
      legend style={/tikz/every even column/.append style={column sep=0.35cm}},
      legend image post style={xscale=0.6},  %
    ]
    \addplot [
      baseline trend line,
      line width=0.9pt,
    ] coordinates {(0,0)};
    \addlegendentry{Uncertified (${\datasetDiv = 1}$)}%
    \label{leg:ExpRes:Bound:Uncertified}

    \addplot [
      single disjoint line,
      line width=0.9pt,
    ] coordinates {(0,0)};
  \addlegendentry{\disjointSingleMethod{}\oursText{}}%
    \label{leg:ExpRes:Bound:PCR}

    \addplot [
      single overlap line,
      line width=0.9pt,
    ] coordinates {(0,0)};
    \addlegendentry{\overlapSingleMethod{}\oursText{}}%
    \label{leg:ExpRes:Bound:OCR}

    \addplot [
      multi disjoint line,
      line width=0.9pt,
    ] coordinates {(0,0)};
    \addlegendentry{\disjointMultiMethod{}\oursText{}}%
    \label{leg:ExpRes:Bound:WPCR}

    \addplot [
      multi overlap line,
      line width=0.9pt,
    ] coordinates {(0,0)};
    \addlegendentry{\overlapMultiMethod{}\oursText{}}%
    \label{leg:ExpRes:Bound:WOCR}

    \addplot [
      knn trend line,
      line width=0.9pt,
    ] coordinates {(0,0)};
    \addlegendentry{\knnMethod{}\oursText{}}%
    \label{leg:ExpRes:Bound:KNN}
  \end{axis}
\end{tikzpicture}
 
  \newcommand{\AmesYMax}{93.5}
  \boundLegendSpacer{}
  \begin{subfigure}[b]{\textwidth}
    \begin{subfigure}[b]{\FirstMiniWidth}
      \boundTrend{ames/ames_00025.csv}  %
                 {\FirstPlotWidth}    %
                 { 23}                %
                 {  5}                %
                 {\boundXLabel{25}}   %
                 {\AmesYMax}          %
                 {\boundYTickDist}    %
                 {\boundYLabel{Ames Housing}}  %
                 {\SkipZeroTick}  %
      \boundSubcaption{Ames Housing}{25}
    \end{subfigure}
    \plotHorizontalSpacer
    \begin{subfigure}[b]{\RestMiniWidth}
      \boundTrend{ames/ames_00125.csv}  %
                 {\RestPlotWidth}       %
                 { 47}                  %
                 { 10}                  %
                 {\boundXLabel{125}}    %
                 {\AmesYMax}            %
                 {\boundYTickDist}      %
                 {}                     %
                 {\NoYTicks{}}  %
      \boundSubcaption{Ames Housing}{125}
    \end{subfigure}
    \plotHorizontalSpacer
    \begin{subfigure}[b]{\RestMiniWidth}
      \boundTrend{ames/ames_00251.csv}  %
                 {\RestPlotWidth}       %
                 { 97}                  %
                 { 20}                  %
                 {\boundXLabel{251}}    %
                 {\AmesYMax}            %
                 {\boundYTickDist}      %
                 {}                     %
                 {\NoYTicks{}}  %
      \boundSubcaption{Ames Housing}{251}
    \end{subfigure}
  \end{subfigure}

  \newcommand{\AustinYMax}{73.0}
  \boundDatasetSpacer{}
  \begin{subfigure}[b]{\FirstMiniWidth}
    \boundTrend{austin/austin_00051.csv}  %
               {\FirstPlotWidth}    %
               { 21}                %
               {  5}                %
               {\boundXLabel{51}}   %
               {\AustinYMax}        %
               { 15}                %
               {\boundYLabel{Austin Housing}}  %
               {\SkipZeroTickFifteen}  %
    \boundSubcaption{Austin Housing}{51}
  \end{subfigure}
  \plotHorizontalSpacer
  \begin{subfigure}[b]{\RestMiniWidth}
    \boundTrend{austin/austin_00301.csv}  %
               {\RestPlotWidth}     %
               {115}                %
               { 25}                %
               {\boundXLabel{301}}  %
               {\AustinYMax}        %
               { 15}                %
               {}                   %
               {\NoYTicks{}}  %
    \boundSubcaption{Austin Housing}{301}
  \end{subfigure}
  \plotHorizontalSpacer
  \begin{subfigure}[b]{\RestMiniWidth}
    \boundTrend{austin/austin_00701.csv}  %
               {\RestPlotWidth}     %
               {215}                %
               { 50}                %
               {\boundXLabel{701}}  %
               {\AustinYMax}        %
               { 15}                %
               {}                   %
               {\NoYTicks{}}  %
    \boundSubcaption{Austin Housing}{701}
  \end{subfigure}

  \newcommand{\DiamondsYMax}{78.0}
  \boundDatasetSpacer{}
  \begin{subfigure}[b]{\FirstMiniWidth}
    \boundTrend{diamonds/diamonds_00151.csv}  %
               {\FirstPlotWidth}              %
               {160}                          %
               { 30}                          %
               {\boundXLabel{151}}            %
               {\DiamondsYMax}                %
               {\boundYTickDist}              %
               {\boundYLabel{Diamonds}}       %
               {\SkipZeroTick}                %
    \boundSubcaption{Diamonds}{151}
  \end{subfigure}
  \plotHorizontalSpacer
  \begin{subfigure}[b]{\RestMiniWidth}
    \boundTrend{diamonds/diamonds_00501.csv}  %
               {\RestPlotWidth}               %
               {475}                          %
               {100}                          %
               {\boundXLabel{501}}            %
               {\DiamondsYMax}                %
               {\boundYTickDist}              %
               {}                             %
               {\NoYTicks{}}                  %
    \boundSubcaption{Diamonds}{501}
  \end{subfigure}
  \plotHorizontalSpacer
  \begin{subfigure}[b]{\RestMiniWidth}
    \boundTrend{diamonds/diamonds_01001.csv}  %
               {\RestPlotWidth}               %
               {675}                          %
               {150}                          %
               {\boundXLabel{1001}}           %
               {\DiamondsYMax}                %
               {\boundYTickDist}              %
               {}                             %
               {\NoYTicks{}}  %
    \boundSubcaption{Diamonds}{1001}
  \end{subfigure}

  \newcommand{\WeatherYMax}{88.0}
  \boundDatasetSpacer{}
  \begin{subfigure}[b]{\FirstMiniWidth}
    \boundTrend{shifts-weather-medium/shifts-weather-medium_00051_filt.csv}  %
               {\FirstPlotWidth}              %
               { 49}                          %
               { 10}                          %
               {\boundXLabel{51}}             %
               {\WeatherYMax}                 %
               {\boundYTickDist}              %
               {\boundYLabel{Weather}} %
               {\SkipZeroTick}  %
    \boundSubcaption{Weather}{51}
  \end{subfigure}
  \plotHorizontalSpacer
  \begin{subfigure}[b]{\RestMiniWidth}
   \boundTrend{shifts-weather-medium/shifts-weather-medium_01501_filt.csv}  %
               {\RestPlotWidth}               %
               {1650}                         %
               { 300}                         %
               {\boundXLabel{1501}}           %
               {\WeatherYMax}                 %
               {\boundYTickDist}              %
               {}                             %
               {\NoYTicks{}}  %
    \boundSubcaption{Weather}{1501}
  \end{subfigure}
  \plotHorizontalSpacer
  \begin{subfigure}[b]{\RestMiniWidth}
    \boundTrend{shifts-weather-medium/shifts-weather-medium_03001_filt.csv}  %
               {\RestPlotWidth}               %
               {2700}                         %
               { 600}                         %
               {\boundXLabel{3001}}           %
               {\WeatherYMax}                 %
               {\boundYTickDist}              %
               {}                             %
               {\NoYTicks{}}  %
    \boundSubcaption{Weather}{3001}
  \end{subfigure}

  \newcommand{\LifeYMax}{95.0}
  \boundDatasetSpacer{}
  \begin{subfigure}[b]{\FirstMiniWidth}
    \boundTrend{life/life_00025.csv}   %
               {\FirstPlotWidth}       %
               { 26}                   %
               {  5}                   %
               {\boundXLabel{25}}      %
               {\LifeYMax}             %
               {\boundYTickDist}       %
               {\boundYLabel{Life}}    %
               {\SkipZeroTick}         %
    \boundSubcaption{Life}{25}
  \end{subfigure}
  \plotHorizontalSpacer
  \begin{subfigure}[b]{\RestMiniWidth}
    \boundTrend{life/life_00101.csv}   %
               {\RestPlotWidth}        %
               { 52}                   %
               { 10}                   %
               {\boundXLabel{101}}     %
               {\LifeYMax}             %
               {\boundYTickDist}       %
               {}                      %
               {\NoYTicks{}}           %
    \boundSubcaption{Life}{101}
  \end{subfigure}
  \plotHorizontalSpacer
  \begin{subfigure}[b]{\RestMiniWidth}
    \boundTrend{life/life_00201.csv}   %
               {\RestPlotWidth}        %
               {130}                   %
               { 30}                   %
               {\boundXLabel{201}}     %
               {\LifeYMax}             %
               {\boundYTickDist}       %
               {}                      %
               {\NoYTicks{}}  %
    \boundSubcaption{Life}{201}
  \end{subfigure}

  \newcommand{\SpambaseYMax}{90.0}
  \boundDatasetSpacer{}
  \begin{subfigure}[b]{\FirstMiniWidth}
    \boundTrend{spambase/spambase_00025.csv}  %
               {\FirstPlotWidth}              %
               { 26}                          %
               {  5}                          %
               {\boundXLabel{25}}             %
               {\SpambaseYMax}                %
               {\boundYTickDist}              %
               {\boundYLabel{Spambase}}  %
               {\SkipZeroTick}  %
    \boundSubcaption{Spambase}{25}
  \end{subfigure}
  \plotHorizontalSpacer
  \begin{subfigure}[b]{\RestMiniWidth}
    \boundTrend{spambase/spambase_00151.csv}  %
               {\RestPlotWidth}               %
               {115}                          %
               { 25}                          %
               {\boundXLabel{151}}            %
               {\SpambaseYMax}                %
               {\boundYTickDist}              %
               {}     %
               {\NoYTicks{}}  %
    \boundSubcaption{Spambase}{151}
  \end{subfigure}
  \plotHorizontalSpacer
  \begin{subfigure}[b]{\RestMiniWidth}
    \boundTrend{spambase/spambase_00301.csv}  %
               {\RestPlotWidth}               %
               {215}                          %
               { 50}                          %
               {\boundXLabel{301}}            %
               {\SpambaseYMax}                %
               {\boundYTickDist}              %
               {}     %
               {\NoYTicks{}}  %
    \boundSubcaption{Spambase}{301}
  \end{subfigure}
  \null{}\vfill
  \caption{\textbf{Certified Accuracy}:
    Mean certified accuracy (larger is better) for our five primary certified regressors.
    \knnMethod{} is always trained on all of training set~$\trainSet$ (i.e., ${\datasetDiv = 1}$).
    Ensemble submodels are trained on $\frac{1}{\datasetDiv}$\=/th of~$\trainSet$, with three $\datasetDiv$ values tested per dataset.
    The x\=/axis is clipped to enhance readability; see suppl.\ Sec.~\ref{sec:App:MoreExps:FullKNN} for \knnMethod{}'s full results.
    \revised{
      The best performing method depends on the target certified robustness~$\certBound$.
      For smaller $\certBound$~values, \overlapMultiMethod{} achieves the best certified accuracy.
      For larger {$\certBound$}~values, \knnMethod{} outperforms the ensemble methods.
      This result aligns with previous findings on certified classification~\citep{Jia:2022:CertifiedKNN}.
      Sec.~\ref{sec:ExpRes:Results} summarizes these experiments' primary takeaways.
    }
    See Sec.~\ref{sec:App:MoreExps:DetailedResults} for the numerical results, including variance.
  }
  \label{fig:ExpRes:CertificationBoundTrend}
\end{figure*}
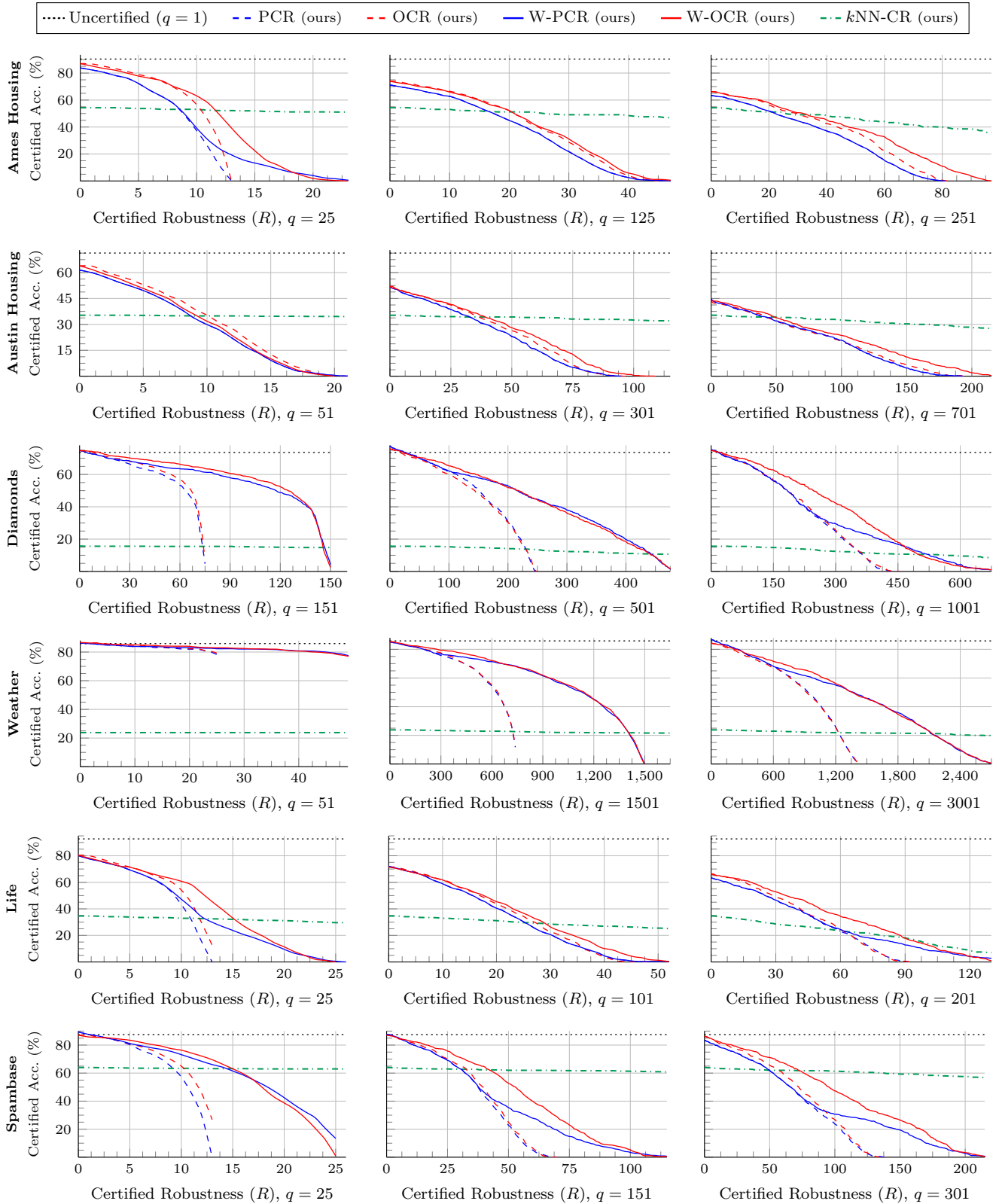
 \restoregeometry
\clearpage
\newpage

\section{Conclusions}%
\label{sec:Conclusions}

This paper describes a novel reduction from certified regression to certified classification based on median perturbation.
Applying this reduction, we propose six new certified regressors that require no assumptions about the data distribution or model architecture.
As improved voting-based, certified classifiers are proposed in the future, our reduction can be applied to those methods too.
This enables certified regression to keep pace with the rapid advances in certified classification.%

While this work focuses on certified defenses against poisoning attacks, some certified \textit{evasion} defenses also rely on voting-based guarantees~\citep{Levine:2020:RandomizedAblation,Jia:2022:CertifiedKNN}.
Our reduction from certified regression to certified classification applies to those certified evasion defenses as well.
Lastly, %
our empirical results show that improved certified guarantees are possible when the unit-cost assumption is replaced by our tighter weighted analysis.
These certification gains apply to both classification and regression, but
Sec.~\ref{sec:BeyondUnitCost:CertifyingByConstruction}'s approach is computationally expensive.
We advocate for better methods that efficiently certify beyond ${\covSym = 1}$.
\section*{Acknowledgments}
  The authors thank Jonathan Brophy and Will Bolden for helpful discussions and feedback on earlier drafts of this manuscript.

  This work was supported by a grant from the Air Force Research Laboratory and the Defense Advanced Research Projects Agency (DARPA) — agreement number FA8750\=/16\=/C\=/0166, subcontract K001892\=/00\=/S05, as well as a second grant from DARPA, agreement number HR00112090135.
  This work benefited from access to the University of Oregon high performance computer, Talapas.
\newrefcontext[sorting=nyt]
\renewcommand*{\bibfont}{\footnotesize}
\printbibliography%
\newrefcontext[sorting=ynt]

\FloatBarrier
\newpage
\clearpage
\startcontents  %
\newpage
\onecolumn
\thispagestyle{empty}
\pagenumbering{arabic}%
\renewcommand*{\thepage}{A\arabic{page}}

\appendix

\SupplementaryMaterialsTitle{}

\setlength{\parskip}{8pt}

\begin{center}
  \textbf{\large Organization of the Appendix}
\end{center}
\printcontents{Appendix}{1}[2]{}

\clearpage
\newpage
\section{\revised{Nomenclature Reference}}%
\label{sec:App:Nomenclature}

\revised{%
\renewcommand*{\arraystretch}{1.2}
\begin{longtable}{lp{5.7in}}
  \caption{%
    \revised{%
      \textbf{Nomenclature Reference}:
      Related symbols are grouped together.
      For example, the first group lists the acronyms of our primary certified regressors.
      Note that this table continues on the next page.%
    }%
  }\label{tab:App:Nomenclature}
  \\\toprule
  \endfirsthead
  \caption{%
    \revised{%
      \textbf{Nomenclature Reference (Continued)}:
      Related symbols are grouped together.%
    }%
  }%
  \\\toprule
  \endhead
  \knnMethod{}      & Our \knn{}-based certified regressor (Sec.~\ref{sec:CertifiedKNN:FixedPopulation}) \\
  \disjointSingleMethod{}
                    & Our \underline{p}artitioned \underline{c}ertified \underline{r}egressor (Sec.~\ref{sec:CertifiedGeneral:PCR}) \\
  \disjointMultiMethod{}
                    & Our \underline{w}eighted-cost \underline{p}artitioned \underline{c}ertified \underline{r}egressor (Sec.~\ref{sec:CertifiedGeneral:NuPCR}) \\
  \disjointSingleMethod{}
                    & Our \underline{o}verlapping \underline{c}ertified \underline{r}egressor (Sec.~\ref{sec:CertifiedOverlap:Uniform}) \\
  \disjointMultiMethod{}
                    & Our \underline{w}eighted-cost \underline{o}verlapping \underline{c}ertified \underline{r}egressor (Sec.~\ref{sec:CertifiedOverlap:NonUniform}) \\
  \midrule
  $\certBound$      & Pointwise certified robustness -- number of possible training set insertions or deletions without violating the prediction bounds -- our primary goal \\
  \midrule
  $\setint{k}$      & Integer set $\set{1, \ldots, k}$       \\
  $\powerSetInt{k}$ & Power set of integer set~$\setint{k}$  \\
  $\ind{q}$         & Indicator function where ${\ind{q} = 1}$ if $q$ is true and 0 otherwise \\
  $\medFunc{A}$     & Median of (multi)set $A$  \\
  $\harmonic{k}$    & $k$\=/th harmonic number where \eqsmall{${\harmonic{k} = \sum_{i = 1}^{k} \frac{1}{i}}$} \\
  \midrule
  $\X$              & Feature vector \\
  $\domainX$        & Feature domain where ${\forall_{\X}\, \X \in \domainX}$ and ${\domainX \subseteq \real^{\dimX}}$ \\
  $\dimX$           & Feature dimension \\
  $\y$              & Target value \\
  $\domainY$        & Target range where ${\forall_{\y} \, \y \in \domainY}$ and ${\domainY \subseteq \real}$ \\
  $\domainZ$        & Instance space where ${\domainZ \defeq \domainX \times \domainY}$  \\
  ${(\xTe, \yTe)}$  & Arbitrary test instance \\
  \midrule
  $\trainSet$       & Training set where ${\trainSet \subseteq \domainZ}$ \\
  $\nTr$            & Training set size where ${\nTr \defeq \abs{\trainSet}}$ \\
  $\nBlocks$        & Number of training set blocks \\
  $\hashTrain$      & Training set partitioning function where $\func{\hashTrain}{\domainZ}{\setint{\nBlocks}}$ \\
  $\blockIdx$       & Training set block index where ${\blockIdx \in \setint{\nBlocks}}$ \\
  \eqsmall{$\blockI$}
                    & $\blockIdx$\=/th training set block where
                      \eqsmall{${\blockI \defeq \setbuild{\z \in \trainSet}{\hashTrainFunc{\z} = \blockIdx}}$}
                      and
                      \eqsmall{${\forall_{\blockIdx' \ne \blockIdx} \, \blockI \cap \blockI[\blockIdx'] = \emptyset}$}
                      \\
  \midrule
  $\dec$            & Robust regressor -- either an ensemble or instance-based learner -- where $\func{\dec}{\domainX}{\domainY}$ \\
  $\decFunc{\xTe}$  & Regressor $\dec$'s prediction for test feature vector~$\xTe$ \\
  $\threshold$      & One-sided upper bound for robustness certification where ${\decFunc{\xTe} \leq \threshold}$ \\
  $\lThreshold$     & Two-sided lower bound for robustness certification where ${\lThreshold \leq \decFunc{\xTe} \leq \uThreshold}$ \\
  $\uThreshold$     & Two-sided upper bound for robustness certification where ${\lThreshold \leq \decFunc{\xTe} \leq \uThreshold}$ \\
  \midrule
  \knn{}            & Vanilla $\kNeigh$\=/nearest neighbors \\
  \knnM{}           & $\kNeigh$\=/nearest neighbors with median as the decision function \\
  \rnn{}            & Radius nearest neighbors \\
  $\neighTe$        & Nearest-neighbors neighborhood (multi)set for test feature vector~$\xTe$ \\
  \midrule
  $\nModel$         & Ensemble submodel count where by definition $\nModel$ is selected to be odd\=/valued \\
  $\modIdx$         & Submodel index where ${\modIdx \in \setint{\nModel}}$ \\
  $\decI$           & $\modIdx$\=/th ensemble submodel \\
  $\covI$           & Number of training set modifications required to violate invariant ${\lThreshold \leq \decFuncI{\xTe} \leq \uThreshold}$. Note that $\covI$ is one larger than $\decI$'s certified robustness \\
  $\covMax$         & Maximum submodel modification requirement where ${\covMax \defeq \max_{\modIdx} \covI}$ \\
  $\hashModel$      & Overlapping training set block assignment function where $\func{\hashModel}{\setint{\nBlocks}}{\powerSetInt{\nModel}}$ \\
  $\modTrainSetI$   & Submodel $\decI$'s training set where for overlapping regression \eqsmall{${\modTrainSetI = \bigcup_{\substack{\blockIdx \in \setint{\nBlocks} \\ \modIdx \in \hashModelFunc{\blockIdx}}} \blockI}$} \\
  $\datasetDiv$     & Inverse of the fraction of the training set used to train each submodel, where \eqsmall{${\forall_{\modIdx} \, q = \frac{\nTr}{\lvert \modTrainSetI \rvert}}$} \\
  $\spreadDegI$     & Spread degree of training set block \eqsmall{$\blockI$} where \eqsmall{${\spreadDegI \defeq \lvert{\hashModelFunc{\blockI}}\rvert}$} \\
  $\spreadDegMax$   & Maximum spread degree where ${\spreadDegMax \defeq \max_{\blockIdx} \spreadDegI}$ \\
  $\lowerSubMods$   & Set of ensemble submodels predicting ${\decFuncI{\xTe} \leq \threshold}$ where ${\lowerSubMods \subseteq \setint{\nModel}}$ \\
  \midrule
  $\setVals$        & Real-valued (multi)set, e.g., \knn{} neighborhood or ensemble submodel predictions, where ${\nModel = \abs{\setVals}}$ \\
  $\lowerSetVals$   & Lower thresholded real-valued multiset where ${\lowerSetVals \defeq \setbuild{\setScalarI \in \setVals}{\setScalarI \leq \threshold}}$ \\
  $\upperSetVals$   & Upper thresholded real-valued multiset where ${\upperSetVals \defeq \setbuild{\setScalarI \in \setVals}{\setScalarI > \threshold}}$ \\
  $\setValsZO$      & Binarized multiset where ${\setValsZO \defeq \setbuild{\sgnp{\setScalarI}}{\setScalarI \in \setVals}}$ \\
  $\setValsMod$     & Adversarially perturbed real-valued (multi)set formed from (multi)set $\setVals$ \\
  \midrule
  $\covSet$         & Weight set where ${\covSet \defeq \setbuild{\covI}{\modIdx \in \setint{\nModel}}}$ \\
  $\lowerCovVals$   & Weight set corresponding to values set $\lowerSetVals$ where ${\lowerCovVals \defeq \setbuild{\covI \in \covSet}{\setScalarI \in \setVals}}$ \\
  $\swapBound$      & Midpoint distance where ${\swapBound \defeq \nLower - \halfModels}$ \\
  $\smallestCovVals$
                    & ${\swapBound}$ smallest values in $\lowerCovVals$ \\
  \midrule
  ILP               & Integer linear program \\
  $\dsBlockVarI$    & ILP integral variable representing the number of instance modifications made to training set block~$\blockI$ \\
  $\modelVarI$      & ILP binary variable which equals~1 if submodel $\decI$ has been perturbed such that ${\decFuncI{\xTe} > \threshold}$ and 0~otherwise \\
  $\isMultiVar$     & ILP binary variable which equals~1 if in the case of weighted analysis and 0~otherwise \\
  PSMC              & Partial set multicover \\
  \midrule
  $\greedyBound$    & Upper-bound on certified robustness~$\certBound$ returned by a greedy algorithm \\
  \bottomrule
\end{longtable}
}%
 
\clearpage
\newpage
\newcommand{\proofSectionHeader}[1]{%
  \vspace{10pt}%
  \noindent%
  \textbf{\large #1}

}

\section{Worst-Case Insertions and Deletions are Interchangeable: Proof}%
\label{sec:App:AdditionalLemma}
\begin{lemma}%
  \label{lem:App:Proofs:AddDeleteSame}%
  \newcommand{\setScalarIM}{\setScalarI_{\nModUpper}}
  For real multiset~\eqsmall{$\setVals$} of cardinality \eqsmall{${\nModel > 1}$},
  if an arbitrarily-large value is inserted into~\eqsmall{$\setVals$} or the smallest value in \eqsmall{$\setVals$} is deleted, the resulting sets' medians are equivalent.
\end{lemma}

\begin{proof}
  For simplicity and w.l.o.g., assume \eqsmall{$\setVals$} is ordered where \eqsmall{${\setScalarOne \leq \cdots \leq \setScalarFin}$}.
  Let \eqsmall{${\medIdx \in \sbrack{1, \nModel}}$} denote the median's index.
  If \eqsmall{$\nModel$}~is odd, \eqsmall{${\medIdx = \halfModels}$}; otherwise, when \eqsmall{$\nModel$} is even, \eqsmall{${\medIdx = \frac{\nModel}{2} + \frac{1}{2}}$}, i.e.,~the midpoint between the \eqsmall{$\frac{\nModel}{2}$}\=/th and \eqsmall{${(\frac{\nModel}{2} + 1)}$}\=/th largest values in~\eqsmall{$\setVals$}.
  Consider first an arbitrarily-large insertion when \eqsmall{$\nModel$}~\ul{odd}.
  Each insertion increases the set's cardinality by~1.
  When \eqsmall{$\setVals$}'s cardinality is odd, then the cardinality of this new set after the first insertion is even.
  Therefore, this new set's median has index%
  {%
    \equationSize%
    \begin{align}%
      \medIdx' &\defeq \frac{\nModel + 1}{2} + \frac{1}{2} & \eqcomment{\text{Median's index for new set of even size ${T + 1}$}} \\
               &= \halfModels + \frac{1}{2} \\
               &= \medIdx + \frac{1}{2}
      \text{.}
    \end{align}%
  }%
  \noindent%
  Since \eqsmall{${\nModel \geq \medIdx'}$} and the inserted element is larger than all values in~\eqsmall{$\setVals$}, the value corresponding to index \eqsmall{${(\halfModels - \frac{1}{2})}$} is equivalent for both original set~\eqsmall{$\setVals$} and the new set after the insertion.

  Next, consider the deletion case for odd~\eqsmall{$\nModel$}.
  Similar to above, the cardinality of the modified set after one deletion is even; therefore, this modified set's median has index%
  {%
    \equationSize%
    \begin{align}%
    \medIdx'' &\defeq \frac{\nModel - 1}{2} + \frac{1}{2} & \eqcomment{\text{Median's index for new set of even size ${T - 1}$}} \\
                &= \halfModels - \frac{1}{2} \\
                &= \medIdx - \frac{1}{2} \text{.}
    \end{align}%
  }%
  \noindent%
  This new set's cardinality is one smaller than the original set with the smallest element removed.
  Hence, the value corresponding to index \eqsmall{${(\halfModels - \frac{1}{2})}$} in this shrunken set equals the value at index \eqsmall{${(\halfModels + \frac{1}{2})}$} in original set~\eqsmall{$\setVals$}.

  Since indices \eqsmall{${\medIdx'}$} and \eqsmall{${\medIdx''}$} correspond to the same value in~\eqsmall{$\setVals$}, the resulting sets' medians are equivalent.
\end{proof}

The primary takeaway from Lemma~\ref{lem:App:Proofs:AddDeleteSame} is that under the insertion/deletion paradigm, worst-case insertions and deletions are interchangeable.
Note that for our purposes, there is an edge case where worst-case insertions and deletions exhibit divergent behavior.
Specifically, after \eqsmall{$\nModel$}~deletions (i.e.,~all elements in~\eqsmall{$\setVals$} are removed), the median of an empty set is not generally defined.
In contrast, the median after \eqsmall{$\nModel$} arbitrarily-large insertions is itself arbitrarily large.
For consistency, we define the empty set's median as~$\infty$ to match the insertion case.
 
\clearpage
\newpage
\section{Proofs for the Main Paper}%
\label{sec:App:Proofs}

\proofSectionHeader{Proof of Lemma~\ref{lem:Warmup:Swap}}

\begin{proof}
  Let \eqsmall{${\lowerSetVals}$} be all elements in \eqsmall{$\setVals$} that do not exceed~\eqsmall{$\threshold$}.

  Under the swap paradigm, the optimal strategy to maximally increase a set's median is to iteratively replace the set's smallest value with~$\infty$.
  Apply this optimal strategy to~\eqsmall{$\setVals$}.
  After one swap, the resulting set contains \eqsmall{${\nLower - 1}$} elements that are less than or equal to~\eqsmall{$\threshold$}.
  After two swaps, there are \eqsmall{${\nLower - 2}$}~such elements with each subsequent swap's effects proceeding inductively.
  Once the modified set contains exactly \eqsmall{${\halfModels}$}~elements less than or equal to~\eqsmall{$\threshold$}, no additional swaps are possible without causing the resulting set's median to exceed~\eqsmall{$\threshold$}.

  Therefore, by induction, the maximum number of swaps that can be performed on \eqsmall{$\setVals$} and it remains guaranteed that the resulting set's median does not exceed~\eqsmall{$\threshold$} is

  {%
    \equationSize%
    \begin{equation}%
      \certBound%
        =
        \nLower
        -
        \halfModels%
      \text{.}%
    \end{equation}%
  }%
\end{proof}
 
\proofSectionHeader{Proof of Lemma~\ref{lem:Warmup:AddDelete}}%
\label{sec:App:Proof:Lemma:Warmup:AddDelete}

\begin{proof}
  For simplicity and w.l.o.g., assume \eqsmall{$\setVals$} is ordered where \eqsmall{${\setScalarOne \leq \cdots \leq \setScalarFin}$}.
  An attacker's optimal insertion strategy is to insert arbitrarily-large values into~\eqsmall{$\setVals$} while the optimal deletion strategy is to always delete \eqsmall{$\setVals$}'s smallest value.
  Lemma~\ref{lem:App:Proofs:AddDeleteSame} proves that these worst-case operations perturb the median identically so we only consider insertions below.

  Let \eqsmall{$\medIdx$} denote the median's index.
  If \eqsmall{$\nModel$}~is odd, \eqsmall{${\medIdx = \halfModels}$}; otherwise, when \eqsmall{$\nModel$} is even, \eqsmall{${\medIdx = \frac{\nModel}{2} + \frac{1}{2}}$}, i.e.,~the midpoint between the \eqsmall{$\frac{\nModel}{2}$}\=/th and \eqsmall{${(\frac{\nModel}{2} + 1)}$}\=/th largest values in~\eqsmall{$\setVals$}.

  Each insertion increases the set's size by~1.
  When \eqsmall{${\setVals}$}'s size is odd, then the size of this new set after the first insertion is even.
  Therefore, this new set's median has index%
  {%
    \equationSize%
    \begin{align}%
      \medIdx' &\defeq \frac{\nModel + 1}{2} + \frac{1}{2} & \text{Median's index for new even size ${T + 1}$}\\
               &= \halfModels + \frac{1}{2} \\
               &= \medIdx + \frac{1}{2}
      \text{.}
    \end{align}%
  }%
  \noindent
  The analysis is essentially identical when \eqsmall{$\nModel$}~is even and is excluded for brevity.
  Note that each insertion always increases the median's index~\eqsmall{$\medIdx$} by~\eqsmall{$\frac{1}{2}$}.

  As long as \eqsmall{${\medIdx \leq \thresholdIdx}$}, it is guaranteed that \eqsmall{${\medFunc{\setVals} \leq \threshold}$}.
  Since each insertion changes \eqsmall{$\medIdx$} by~\eqsmall{$\frac{1}{2}$}, then \eqsmall{${2 (\thresholdIdx - \medIdx)}$}~arbitrary insertions can be made in~\eqsmall{$\setVals$} with it remaining guaranteed that the modified set's~median does not exceed~\eqsmall{$\threshold$}.
  Regardless of whether \eqsmall{$\nModel$} is odd or even,%
  \footnote{%
    When $\nModel$ is odd, ${2\medIdx = 2\halfModels = \nModel + 1}$ while in the even case ${2\medIdx = 2\left(\frac{\nModel}{2} + \frac{1}{2}\right) = \nModel + 1}$.%
  }
  it holds that
  {%
    \equationSize%
    \begin{equation}%
      \label{eq:Lemma:Warmup:AddDelete:Proof:UntightR}%
      \certBound
      \geq
        2(\thresholdIdx - \medIdx)
      =
        2\nLower - 2\medIdx
      =
        2\nLower - \nModel - 1
      \text{.}%
    \end{equation}%
  }%
\end{proof}
 
\newpage%
\clearpage%
\proofSectionHeader{Proof of Lemma~\ref{lem:Warmup:NonUniformSwap}}

\newcommand{\covLowI}[1][\modIdx]{\modSub[#1]{\tilde{\covSym}}}

\begin{proof}
  \newcommand{\covPartial}{\covLowI[\swapBound + 1]}

  The first portion of this proof follows the same argument as the proof of Lemma~\ref{lem:Warmup:Swap}, with one primary difference.
  There, the optimal strategy to perturb \eqsmall{$\setVals$}'s median swapped out the smallest values in~\eqsmall{$\setVals$} first.
  For the weighted version, the optimal (greedy) strategy swaps out whichever value in \eqsmall{${\lowerSetVals \subseteq \setVals}$} has the smallest weight.

  To perturb \eqsmall{$\setVals$}'s median above~\eqsmall{$\threshold$}, it is sufficient to swap any \eqsmall{${\nLower - \halfModels}$}~values in \eqsmall{$\lowerSetVals$} with an arbitrarily large replacement.
  For simplicity and without loss of generality, let \eqsmall{${\covLowI[1], \ldots, \covLowI[\nLower]}$} be the weights of the elements in~\eqsmall{$\lowerSetVals$} arranged in ascending order.
  Define \eqsmall{${\swapBound \defeq \nLower - \halfModels}$}.
  Applying Lemma~\ref{lem:Warmup:Swap}, up to \eqsmall{${\swapBound}$}~values in \eqsmall{$\setVals$} can be replaced without perturbing the median.
  This entails a minimum cost of
  {%
    \equationSize%
    \begin{equation}%
      \label{eq:App:Proof:Lemma:WeightedSwap:LooseBound}%
      \certBound
      \geq
        \sum_{\modIdx = 1}^{\swapBound}
          \covLowI
        \text{.}%
    \end{equation}%
  }%

  Denote the \eqsmall{${(\swapBound + 1)}$}\=/th largest weight in~\eqsmall{${\lowerCovVals}$} as \eqsmall{${\covPartial}$}.
  Observe that adding \eqsmall{${\covPartial - 1}$} to Eq.~\eqref{eq:App:Proof:Lemma:WeightedSwap:LooseBound} is insufficient to swap out any remaining elements in~\eqsmall{$\lowerSetVals$} since all elements with weight less than \eqsmall{${\covPartial}$} are already replaced and all remaining elements have weight at least \eqsmall{$\covPartial$}.
  Therefore, the certified robustness is
  {%
    \equationSize%
    \begin{align}%
      \certBound
      &=
        (\covPartial - 1)
        +
        \sum_{\modIdx = 1}^{\swapBound}
          \covLowI
      \\
      &=
        \sum_{\modIdx = 1}^{\swapBound + 1}
          \covLowI
        -
        1
      \\
      &=
        \sum_{\modIdx = 1}^{\nLower - \halfModels + 1}
          \covLowI
        -
        1
      \\
      &=
        \sum_{\covSym \in \smallestCovVals}
          \covSym
        -
        1
        \label{eq:WarmUp:NonUniformSwap:Proof:NonZeroBound}
      \text{.}
    \end{align}%
  }%

  Moreover, increasing Eq.~\eqref{eq:WarmUp:NonUniformSwap:Proof:NonZeroBound} by one would allow for the \eqsmall{${(\swapBound + 1)}$}\=/th largest value in \eqsmall{$\setVals$} to be swapped, which would in turn perturb the set's median above~\eqsmall{$\threshold$}.
  Therefore, Eq.~\eqref{eq:WarmUp:NonUniformSwap:Proof:NonZeroBound}'s bound is tight.
\end{proof}
 
\proofSectionHeader{Proof of Lemma~\ref{lem:RegressionToClassification}}

\begin{proof}
  The median perturbation paradigms formalized in Lemmas~\ref{lem:Warmup:Swap}, \ref{lem:Warmup:AddDelete}, and~\ref{lem:Warmup:NonUniformSwap} calculate their certified robustness using three values, namely: \eqsmall{$\nModel$}, \eqsmall{$\nLower$}, and \eqsmall{$\halfModels$}.
  If these three values are equivalent for \eqsmall{$\setVals$} and \eqsmall{$\setValsZO$}, then their associated certified robustness~(\eqsmall{$\certBound$}) must also be equal.

  Since \eqsmall{${\abs{\setVals} = \abs{\setValsZO}}$}, they have equivalent \eqsmall{$\nModel$} and \eqsmall{$\halfModels$}.
  By definition, the binarization of \eqsmall{$\setVals$} to \eqsmall{$\setValsZO$} does not change the value of~\eqsmall{$\nLower$} either.
  Therefore, for all three median perturbation paradigms, binary multiset~\eqsmall{$\setValsZO$} and real multiset~\eqsmall{$\setVals$} have equivalent certified robustness~\eqsmall{$\certBound$}.
\end{proof}

\proofSectionHeader{Proof of Theorem~\ref{thm:CertifiedKNN:FixedSize}}

\begin{proof}
  For fixed-population IBLs, certifying that \eqsmall{${\decFunc{\xTe} \leq \threshold}$} simplifies to median perturbation under Sec.~\ref{sec:WarmUp:Swap}'s unweighted swap paradigm since all necessary criteria are met, namely that
  \begin{enumerate}
    \item $\dec$'s~decision function is a median operation over a set of values,
      i.e.,~\eqsmall{${\decFunc{\xTe} \defeq \medFunc{\neighTe}}$}.

    \item Neighborhood~\eqsmall{$\neighTe$} has fixed cardinality~\eqsmall{$\nModel$}, and \eqsmall{$\nModel$}~is odd.

    \item A worst-case modification to training set~\eqsmall{$\trainSet$} causes an element in~\eqsmall{$\neighTe$} to be replaced with a different one.
  \end{enumerate}

  Lemma~\ref{lem:Warmup:Swap}, therefore, provides a (lower) bound on the number of training set modifications that can be made without the resulting model violating the requirement that~\eqsmall{${\decFunc{\xTe} \leq \threshold}$}.
  That is why certified robustness~$\certBound$ in Eqs.~\eqref{eq:CertifiedKNN:FixedSize} and~\eqref{eq:Warmup:Swap:Lemma:Bound} (Thm.~\ref{thm:CertifiedKNN:FixedSize} \& Lem.~\ref{lem:Warmup:Swap}, resp.) are equivalent.
\end{proof}
 
\proofSectionHeader{Proof of Theorem~\ref{thm:CertifiedKNN:DynamicSize}}

\begin{proof}
  This proof follows a very similar structure as Theorem~\ref{thm:CertifiedKNN:FixedSize}'s proof above.
  The primary distinction is that a different median perturbation paradigm from Sec.~\ref{sec:WarmUp} is needed here.

  For region-based IBLs, certifying that \eqsmall{${\decFunc{\xTe} \leq \threshold}$} simplifies to median perturbation under Sec.~\ref{sec:WarmUp:InsertOrDelete}'s insertion/deletion paradigm since the three necessary criteria are met:
  \begin{enumerate}
    \item $\dec$'s~decision function is a median operation over a set of values, i.e.,~\eqsmall{${\decFunc{\xTe} \defeq \medFunc{\neighborhood{\xTe}}}$}.
    \item Neighborhood cardinality~\eqsmall{$\nModel$} is not fixed but can increase and/or decrease.
    \item Each modification of~\eqsmall{$\neighborhood{\xTe}$} takes the form of either an insertion or deletion, i.e.,~not swaps.
  \end{enumerate}

  Therefore, Lemma~\ref{lem:Warmup:AddDelete} bounds the total number of training set insertions/deletions that can be performed without violating the requirement that~\eqsmall{${\decFunc{\xTe} \leq \threshold}$}.
  That is why \eqsmall{$\certBound$}'s definition in Eq.~\eqref{eq:CertifiedKNN:DynamicSize} is identical to Eq.~\eqref{eq:Warmup:AddDelete:Lemma:Bound}.
\end{proof}

\proofSectionHeader{Proof of Theorem~\ref{thm:CertifiedGeneral:PCR}}

\begin{proof}
  Certifying here that \eqsmall{${\decFunc{\xTe} \leq \threshold}$} simplifies to median perturbation under the unweighted swap paradigm since all necessary criteria are satisfied, specifically that
  \begin{enumerate}
    \item $\dec$'s~decision function is a median operation over a set of fixed, deterministic values, i.e.,~\eqsmall{${\decFunc{\xTe} \defeq \medFunc{\set{\decFuncOne{\xTe}, \ldots, \decFuncFin{\xTe}}}}$}.
    \item Since the submodels are trained on disjoint data/feature regions, a change to one submodel (i.e.,~value \eqsmall{$\decFuncI{\xTe}$}) has no effect on any other submodel (value).
    \item Each submodel perturbation causes an existing value in the set to be replaced by a new value.
    \item \eqsmall{$\nModel$} is fixed and odd-valued.
    \item The cost to change any submodel (i.e.,~value) is one, i.e., \eqsmall{${\forall_{\modIdx}\, \covI = 1}$}.
  \end{enumerate}

  Lemma~\ref{lem:Warmup:Swap} provides a (lower) bound on the number of training set modifications that can be performed without violating the requirement that \eqsmall{${\decFunc{\xTe} \leq \threshold}$}.
  Certified robustness~\eqsmall{$\certBound$} in Eq.~\eqref{eq:CertifiedGeneral:PCRBound} is then identical to Lemma~\ref{lem:Warmup:Swap}'s Eq.~\eqref{eq:Warmup:Swap:Lemma:Bound}.
\end{proof}
 
\proofSectionHeader{Proof of Theorem~\ref{thm:CertifiedGeneral:NonUniformPCR}}

\begin{proof}
  Here, we extend the argument in Theorem~\ref{thm:CertifiedGeneral:PCR}'s proof to the weighted case.
  Four of the five criteria in Thm.~\ref{thm:CertifiedGeneral:PCR}'s proof still hold, specifically that
  \begin{enumerate}
    \item $\dec$'s~decision function is a median over a set of values.
    \item Each submodel is independent and deterministic.
    \item Modifications to the set of values take the form of swaps.
    \item \eqsmall{$\nModel$}~is fixed and odd-valued.
  \end{enumerate}
  \noindent%
  The only difference is that the perturbations are weighted where each value \eqsmall{$\decFuncI{\xTe}$} now has an associated cost \eqsmall{${\covI \geq 0}$}.
  Therefore, Sec.~\ref{sec:Warmup:NonUniformSwap}'s weighted swap paradigm applies.
  Certified robustness~\eqsmall{$\certBound$} in Eq.~\eqref{eq:CertifiedGeneral:Thm:NonUniformPCRBound} follows directly from and is identical to \eqsmall{$\certBound$} in Lemma~\ref{lem:Warmup:NonUniformSwap}'s Eq.~\eqref{eq:WarmUp:NonUniformSwap:Bound}.
\end{proof}\
 
\proofSectionHeader{Proof of Lemma~\ref{lem:CertifiedOverlap:Uniform:Hardness}}

\newcommand{\nSetCover}{\nModel'}%
\begin{proof}
  To prove a problem is NP\=/hard, it suffices to show that there exists a polynomial time reduction from a known NP\=/hard problem to it.
  As explained in Sec.~\ref{sec:CertifiedOverlap:Uniform}, partial set cover is NP\=/hard~\citep{Slavik:1997:PartialDCover,Slavik:1997:GreedySetCover}.
  Below we map partial set cover to overlapping certified regression.
  Let \eqsmall{${\groundSet \defeq \setint{\nSetCover}}$} be a ground set of \eqsmall{$\nSetCover$}~elements,
  and let \eqsmall{${\setCollection \defeq \set{\subsetOne, \ldots, \subsetFin}}$} be a collection of sets where each \eqsmall{${\subsetI \subseteq \groundSet}$} and
  {%
    \equationSize%
    \begin{equation*}
      \bigcup_{\subsetSym \in \setCollection}
        \subsetSym
        =
        \groundSet
        \text{.}
    \end{equation*}%
  }%
  The goal is to find the subcover \eqsmall{${\subcover \subseteq \setCollection}$} of minimum cardinality s.t.
  {%
    \equationSize%
    \begin{equation*}%
      \swapBound
        \leq
        \bigg\vert
          \bigcup_{\subsetSym \in \subcover}
            \subsetSym%
        \bigg\rvert%
      \text{,}
    \end{equation*}
  }%
  \noindent%
  where \eqsmall{${\swapBound \in \setint{\nSetCover}}$}.

  It is straightforward to map the above to overlapping certified regression.
  Let the ensemble have \eqsmall{${(2 \nSetCover + 1)}$}~submodels.
  Function~\eqsmall{$\hashTrain$} partitions the training set into \eqsmall{$\nBlocks$}~blocks with the blocks denoted \eqsmall{${\blockOne, \ldots, \blockFin}$}.
  Define the block mapping function as:
  {%
    \equationSize%
    \begin{equation}%
      \label{eq:App:Proof:Lemma:OCR:NP:Hard:HashTrain}%
      \hashModelFunc{\blockIdx}
        \defeq
        \begin{cases}
          \subsetI,  & \blockIdx \leq \nSetCover \\
          \emptyset, & \text{Otherwise}
        \end{cases}
      \text{.}
    \end{equation}%
  }%
  \noindent%
  Intuitively, each of the first \eqsmall{$\nSetCover$}~submodels is trained on one of the subsets in~\eqsmall{$\setCollection$}, while the remaining models are not trained on any data.

  Let all submodels be constant a function.
  Define the submodel function as
  {%
    \equationSize%
    \begin{equation}%
      \decFuncI{\xTe}
        \defeq
        \begin{cases}
          -\infty, & \modIdx \leq \swapBound + \halfModels \\
          \infty,  & \text{Otherwise}
        \end{cases}
      \text{.}
    \end{equation}%
  }%
  For any finite \eqsmall{$\threshold$}, \eqsmall{${\nLower = \swapBound + \halfModels}$}.
  Applying Theorem~\ref{thm:CertifiedGeneral:PCR}, the number of submodels overlapping certified regression perturbs is \eqsmall{${\nLower - \halfModels = \swapBound}$}.

  Overlapping certified regression's robustness~\eqsmall{$\certBound$} is the solution to the original partial set cover problem because
  \begin{enumerate}
    \item Only models with index \eqsmall{${\modIdx \leq \nSetCover}$} will be perturbed since all other submodels have no training data.
    \item The training set of each of these \eqsmall{$\nSetCover$}~submodels maps directly to a subset in~\eqsmall{$\setCollection$}.
    \item Overlapping certified regression seeks to find the minimum number of dataset blocks that must be modified to perturb the median prediction.  In this formulation, the number of blocks to be modified is~\eqsmall{$\swapBound$} -- same as in the original partial-set cover problem.
  \end{enumerate}

  If overlapping certified regression were solvable in polynomial-time, then partial set cover would also be solvable in polynomial time.
  However, partial set cover is NP\=/hard, meaning overlapping certified regression must also be NP\=/hard.
\end{proof}
 
\proofSectionHeader{Proof of Corollary~\ref{col:CertifiedOverlap:NonUniform:Hardness}}

\begin{proof}
  From Lemma~\ref{lem:CertifiedOverlap:Uniform:Hardness} above, (unweighted) overlapping certified regression is NP\=/hard.
  The unweighted case trivially maps to the weighted one where \eqsmall{${\forall_{\modIdx}\, \covI = 1}$}.
  Therefore, weighted \overlapSingleMethod{} must be at least as hard as the unweighted case meaning \overlapMultiMethod{} is also NP\=/hard.
\end{proof}
 
\proofSectionHeader{Proof of Lemma~\ref{lem:BeyondUnitCost:CertifyingDeletions}}

\begin{proof}
  By construction

  Given a deterministic training algorithm, a model's prediction can be certified against the deletion of any subset \eqsmall{${\trSub \subset \trainSet}$} by training a model on just dataset \eqsmall{${\trSub \setminus \trSub}$} and verifying the prediction does not violate the associated \keyword{invariant},
  i.e.,~\eqsmall{${\decSubsetFunc{\trainSet \setminus \trSub}{\xTe} \leq \threshold}$}.

  Consider training a separate model on each subset of~\eqsmall{$\trainSet$} of size at least ${\nTr - \covSym + 1}$.
  If \textit{all} of those models also satisfy the invariant,
  then by construction, ${\covSym - 1}$~deletions or fewer are insufficient to violate the invariant.
  If ${\covSym - 1}$~deletions are not enough, then at least ${\covSym}$~deletions are required.
\end{proof}

  Lemma~\ref{lem:BeyondUnitCost:CertifyingDeletions}'s proof above only applies if model prediction and training is deterministic, i.e.,~repeating training and then the prediction always yields the same predicted value.
Otherwise, proof by construction would require verifying all random seeds for each subset of~\eqsmall{$\trainSet$}.
  
\clearpage
\newpage
\newcommand{\jiaPluralClass}{c}
\newcommand{\jiaOtherClass}{\jiaPluralClass'}

\newcommand{\jiaNPlural}{\nModel_{\jiaPluralClass}}
\newcommand{\jiaNOther}{\nModel_{\jiaOtherClass}}

\newcommand{\jiaInd}{\ind{\jiaPluralClass > \jiaOtherClass}}

\section{Reducing \knnMethod{} to \citeauthor{Jia:2022:CertifiedKNN}'s Certified \knn{} Classifier}%
\label{sec:App:JiaVsKnnCR}

Sec.~\ref{sec:RelatedWork} mentions that \citet{Jia:2022:CertifiedKNN} propose a certified classifier by leveraging the implicit robustness of nearest neighbor methods.
In their paper, \citeauthor{Jia:2022:CertifiedKNN} examine both \keyword{pointwise certification}\footnote{\citet{Jia:2022:CertifiedKNN} term ``pointwise certification'' as ``\keyword{individual certification}.''} of test instances individually as well as \keyword{joint certification} of multiple test instances collectively.
Here, we exclusively consider \citeauthor{Jia:2022:CertifiedKNN}'s pointwise contributions.

This section explores in detail how our certified robust regressor, \knnMethod{}, reduces to \citeauthor{Jia:2022:CertifiedKNN}'s certified \knn{} classifier.
For classification, label space~\eqsmall{$\domainY$} is nominal and consists of \eqsmall{$\abs{\domainY}$}~classes.
Given test instance \eqsmall{${\xTe \in \domainX}$}, a \knn{} classifier returns neighborhood \eqsmall{$\neighTe$}, which is a multiset whose \keyword{underlying set} is~\eqsmall{$\domainY$}.
Let \eqsmall{$\jiaPluralClass$} and \eqsmall{$\jiaOtherClass$} be the labels with the largest and second-largest \keyword{multiplicity} in~\eqsmall{$\neighTe$}.
In other words, \eqsmall{$\jiaPluralClass$} and \eqsmall{$\jiaOtherClass$} are the first and second most popular labels in \eqsmall{$\xTe$}'s neighborhood.
Denote the multiplicity of labels \eqsmall{$\jiaPluralClass$} and \eqsmall{$\jiaOtherClass$} as \eqsmall{$\jiaNPlural$} and \eqsmall{$\jiaNOther$}, respectively.

\citet[Thm.~1]{Jia:2022:CertifiedKNN} specify their certified \knn{} classifier's robustness bound as
{%
  \equationSize%
  \begin{equation}%
    \label{eq:App:JiaVsKnnCR:BaseBound}
    \certBound
    =
      \ceil{%
        \frac{%
               \jiaNPlural
               -
               \jiaNOther
               +
               \ind{\jiaPluralClass > \jiaOtherClass}
             }%
             {2}
      }
      -
      1
      \text{,}
  \end{equation}%
}%
\noindent%
where indicator function \eqsmall{$\jiaInd$} breaks ties by choosing whichever label is assigned the larger number.

Lemma~\ref{lem:App:JiaVsKnnCR:Same} below establishes that for binary classification, \citeauthor{Jia:2022:CertifiedKNN}'s method and \knnMethod{} certify equivalent robustness guarantees~($\certBound$).
This symmetry between \citeauthor{Jia:2022:CertifiedKNN}'s method and \knnMethod{} extends beyond classification to regression as well.

Recall that Lemma~\ref{lem:RegressionToClassification} establishes symmetry between certified robustness in a real-valued domain and robustness in a binarized domain.
Applying that insight here, \citeauthor{Jia:2022:CertifiedKNN}'s certified \knn{} (binary) classifier is extended from the binary domain to certify the robustness of \knnMethod{}'s real-valued regression predictions.
Observe that this interplay is the fundamental concept underpinning any reduction.
Put simply, \citeauthor{Jia:2022:CertifiedKNN}'s certification algorithm serves as, in essence, a ``subroutine'' within Fig.~\ref{fig:RegressionToClassification}'s certified regressor framework.

\vfill
\begin{lemma}%
  \label{lem:App:JiaVsKnnCR:Same}
  For \eqsmall{${\xTe \in \domainX}$} and \eqsmall{$\dec$} a $\kNeigh$\=/nearest neighbor classifier,
  let \eqsmall{$\neighTe$} be \eqsmall{$\xTe$}'s neighborhood under \eqsmall{$\dec$} with \eqsmall{${\nModel \defeq \abs{\neighTe}}$} odd.
  For binary classification, \citet{Jia:2022:CertifiedKNN} and \knnMethod{} certify equivalent robustness.
\end{lemma}

\begin{proof}
  For binary classification with odd neighborhood cardinality~\eqsmall{$\nModel$}, there cannot be ties between labels.
  Moreover, \eqsmall{${\jiaNPlural - \jiaNOther}$} is always odd by the pigeonhole principle.
  Applying these two observations, Eq.~\eqref{eq:App:JiaVsKnnCR:BaseBound} simplifies to
  {%
    \equationSize%
    \begin{align}%
      \certBound
      &=
        \ceil{%
          \frac{%
                 \jiaNPlural
                 -
                 \jiaNOther
               }%
               {2}
        }
        -
        1
      &\eqcomment{\text{No ties}}
      \label{eq:App:PcrVsDpa:NoInd}%
      \\
      &=
        \frac{%
               \jiaNPlural
               -
               \jiaNOther
               +
               1
             }%
             {2}
        -
        1
      \text{.}
        &\eqcomment{{\jiaNPlural - \jiaNOther} \text{ is odd}}
      \label{eq:App:PcrVsDpa:OddSize}%
    \end{align}%
  }%

  Under binary classification, \eqsmall{${\jiaNPlural + \jiaNOther = \nModel}$} meaning
  {%
    \equationSize%
    \begin{align}
      \certBound
      &=
        \frac{%
              \jiaNPlural
              -
              (\nModel - \jiaNPlural)
              +
              1
             }%
             {2}
             -
             1
      \\
      &=
        \frac{%
              2\jiaNPlural
              -
              \nModel
              -
              1
             }%
             {2}
      \\
      &=
        \jiaNPlural
        -
        \frac{%
              \nModel
              +
              1
             }%
             {2}
      \\
      &=
        \jiaNPlural
        +
        \halfModels
      \text{.}
      \label{eq:App:JiaVsKnnCR:FinalBound}
    \end{align}%
  }%

  In the case of binary classification, \eqsmall{$\lowerSetVals$} represents all models that predict majority label~\eqsmall{$\jiaPluralClass$}.
  Therefore, \eqsmall{${\nLower = \jiaNPlural}$} meaning Eq.~\eqref{eq:App:JiaVsKnnCR:FinalBound} is equivalent to Eq.~\eqref{eq:CertifiedKNN:FixedSize}.
\end{proof}

\vfill
 
\clearpage
\newpage
\newcommand{\dpaPluralClass}{c}
\newcommand{\dpaOtherClass}{\dpaPluralClass'}

\newcommand{\dpaNPlural}{\nModel_{\dpaPluralClass}}
\newcommand{\dpaNOther}{\nModel_{\dpaOtherClass}}

\newcommand{\dpaInd}{\ind{\dpaOtherClass < \dpaPluralClass}}

\section{Reducing \disjointSingleMethod{} to DPA}
\label{sec:App:PcrVsDpa}

As detailed in Sec.~\ref{sec:RelatedWork}, \citet{Levine:2021:DPA}'s deep partition aggregation~(DPA) defense relies on an ensemble of \eqsmall{$\nModel$}~fully-independent, deterministic submodels,%
\footnote{%
  \citeauthor{Levine:2021:DPA} do not specify that $\nModel$ is odd, but such a choice is standard in binary classification settings.
}
each of which is trained on disjoint data.
Like our certified regressors, DPA certifies the \textit{pointwise} robustness of an individual model prediction,~\eqsmall{$\decFunc{\xTe}$}.

In the classification case, label space~\eqsmall{$\domainY$} is nominal and consists of \eqsmall{$\abs{\domainY}$}~classes.
Given \eqsmall{${\xTe \in \domainX}$}, each submodel assigns $\xTe$~some label, i.e., \eqsmall{${\decFuncI{\xTe} \in \domainY}$}.
Denote the ensemble's plurality label for \eqsmall{$\xTe$} as ${\dpaPluralClass \in \domainY}$, and denote the number of submodels that predict $\dpaPluralClass$ for $\xTe$ as~\eqsmall{${\dpaNPlural \in \setint{\nModel}}$}.
Let \eqsmall{${\dpaOtherClass \in \domainY \setminus \set{\dpaPluralClass}}$} denote any label other than the ensemble's plurality prediction, and let \eqsmall{${\dpaNOther \in \set{0, \ldots, \halfModels}}$} denote the number of submodels that predict $\dpaOtherClass$ for~\eqsmall{$\xTe$}.

\citet[Thm.~1]{Levine:2021:DPA} specify DPA's certified robustness bound as
{%
  \equationSize%
  \begin{equation}%
    \label{eq:App:PcrVsDpa:BaseBound}
    \certBound
      =
      \floor{%
        \frac{\dpaNPlural - \max_{\dpaOtherClass \ne \dpaPluralClass}\setDynamic{\dpaNOther + \dpaInd}}%
             {2}%
      }
    \text{,}
  \end{equation}%
}%
\noindent%
where $\dpaInd$ breaks ties by predicting whichever label has the lower assigned number.

Lemma~\ref{lem:App:PcrVsDpa:Same} below establishes that DPA and \disjointSingleMethod{} certify equivalent robustness guarantees for binary classification with an odd number of submodels.
This symmetry between DPA and \disjointSingleMethod{} extends beyond classification to regression as well.

Recall that Lemma~\ref{lem:RegressionToClassification} establishes symmetry between certified robustness in a real-valued domain and robustness in a binarized domain.
Applying that insight, it is clear then that DPA is certifying the robustness of our certified regressor \disjointSingleMethod{}, albeit using \eqsmall{$\setVals$}'s surrogate \eqsmall{${\setValsZO \defeq \setbuild{\sgnp{\setScalarI - \threshold}}{\setScalarI \in \setVals}}$}.
This is the fundamental idea behind reductions.
Here, DPA's certification algorithm serves as a type of ``subroutine'' within the certified regressor framework (Fig.~\ref{fig:RegressionToClassification}).

\vfill
\begin{lemma}
  \label{lem:App:PcrVsDpa:Same}
  Let \eqsmall{$\dec$} be an ensemble of \eqsmall{$\nModel$}~fully-independent, deterministic submodels where each submodel is trained on disjoint data with \eqsmall{$\nModel$}~odd.
  For binary classification, DPA and \disjointSingleMethod{} certify equivalent robustness.
\end{lemma}

\begin{proof}
  This proof is very similar to that of Lem.~\ref{lem:App:JiaVsKnnCR:Same}.
  Nonetheless, we repeat the full details to make each proof standalone.

  Let \eqsmall{${\xTe \in \domainX}$} be any feature vector.
  Given ensemble~\eqsmall{$\dec$}, \eqsmall{$\xTe$}'s non-majority label~$\dpaOtherClass$ is unique so Eq.~\eqref{eq:App:PcrVsDpa:BaseBound} simplifies to
  {%
    \equationSize%
    \begin{equation}%
      \label{eq:App:PcrVsDpa:NoMax}%
      \certBound
        =
        \floor{%
          \frac{\dpaNPlural - (\dpaNOther + \dpaInd)}%
               {2}%
        }
      \text{.}
    \end{equation}%
  }%
  \noindent%
  Similarly for odd~\eqsmall{$\nModel$}, ties are not possible making the indicator function unnecessary and able to be ignored depending how $\dpaPluralClass$ and $\dpaOtherClass$ are chosen to be defined.
  This simplifies Eq.~\eqref{eq:App:PcrVsDpa:NoMax} to
  {%
    \equationSize%
    \begin{equation}%
      \label{eq:App:PcrVsDpa:NoIndicator}%
      \certBound
        =
        \floor{%
          \frac{\dpaNPlural - \dpaNOther}%
               {2}%
        }
        \text{.}
    \end{equation}%
  }%

  Eq.~\eqref{eq:App:PcrVsDpa:NoIndicator} can be rewritten as
  {%
    \equationSize%
    \begin{align}
      \certBound
        &=
          \frac{\dpaNPlural - \dpaNOther - 1}%
               {2}%
        & \eqcomment{\dpaNPlural - \dpaNOther \text{ always odd when } \nModel \text{ is odd}}%
      \\
        &=
          \frac{\dpaNPlural - (\nModel - \dpaNPlural) - 1}%
               {2}%
        & \eqcomment{\dpaNPlural + \dpaNOther = \nModel \text{ for binary classification}}%
      \\
        &=
          \dpaNPlural
          -
          \frac{\nModel + 1}%
               {2}
      \\
      &=
        \dpaNPlural
        -
        \halfModels
      \text{.}
      \label{eq:App:PcrVsDpa:FinalBound}
    \end{align}
  }%

  In the case of binary classification, \eqsmall{$\lowerSetVals$} represents all models that predict majority label~$\dpaPluralClass$.
  Therefore, \eqsmall{${\nLower = \dpaNPlural}$} meaning Eq.~\eqref{eq:App:PcrVsDpa:FinalBound} is equivalent to Eq.~\eqref{eq:CertifiedGeneral:PCRBound}.
\end{proof}

\vfill
 
\clearpage
\newpage
\section{Relating \overlapSingleMethod{} to \citeauthor{Wang:2022:DeterministicAggregation}'s Overlapping Certified Classifier}%
\label{sec:App:WangDFA}

As discussed in Sec.~\ref{sec:RelatedWork}, \citet{Wang:2022:DeterministicAggregation} propose an overlapping certified classifier they name \textit{deterministic finite aggregation}~(DFA).
Note the acronym similarity between DFA and \citepos{Levine:2021:DPA} deep partition aggregation~(DPA).

As a basic intuition on DFA, recall that \eqsmall{$\spreadDegMax$} is the maximum spread degree of any dataset block over the \eqsmall{${\nModel}$}~overlapping submodels.
As detailed in Sec.~\ref{sec:CertifiedOverlap}, when certifying robustness for both classification and regression, (partially) covering set~\eqsmall{$\setint{\nModel}$} is \textit{not} the goal.
Rather, the goal is to partially cover set
{%
  \equationSize%
  \begin{equation}%
    \lowerSubMods%
    \defeq%
      \setbuild{\modIdx}%
               {\decFuncI{\xTe} \leq \threshold}%
    \subseteq%
      \setint{\nModel}%
    \text{,}%
  \end{equation}%
}%
\noindent%
i.e.,~those subset of models that predict at most~\eqsmall{$\threshold$}.
Over this restricted subset, the maximum block spread degree may be less than \eqsmall{$\spreadDegMax$}.
As a very simplified explanation, \citeauthor{Wang:2022:DeterministicAggregation} use this insight to provide worst-case deterministic bounds on the robustness of an overlapping classification prediction.

Our discussion of overlapping regression in Sec.~\ref{sec:CertifiedOverlap} does not directly apply \citeauthor{Wang:2022:DeterministicAggregation}'s robustness certifier for multiple reasons, including
\begin{enumerate}
  \item \citeauthor{Wang:2022:DeterministicAggregation} do not explain that determining optimal \eqsmall{$\certBound$} under their formulation is NP\=/hard. We believe this is an important insight.

  \item \citeauthor{Wang:2022:DeterministicAggregation} do not analyze the optimality gap of their approach.  Both our ILP approach and partial set cover approximation give optimality bounds.
    What is more, our ILP-based approach can actually prove optimality gives \eqsmall{${\covOne, \ldots, \covFin}$}.
    In contrast, when \citeauthor{Wang:2022:DeterministicAggregation}'s solution is optimal, no indication of that fact is readily given.

  \item \citeauthor{Wang:2022:DeterministicAggregation} only consider the unit-cost assumption, meaning their formulation does not natively support weighted robustness analysis.  To ensure a fair and direct comparison between \overlapSingleMethod{} and \overlapMultiMethod{}, we use an ILP for both methods.
\end{enumerate}

One clear limitation of our choice to use an ILP for overlapping robustness certification is the added computational cost.
However, in our implementation, solving a single ILP uses very few resources, e.g.,~just a single core in a multicore CPU.
Moreover, we enforce tractable computation times by specifying an ILP time limit.
While these two factors combined are not a panacea, they do assuage (some) computational concerns.
 
\clearpage
\newpage
\section{On the Tightness of Certified Regression}%
\label{sec:App:Tightness}

This section explores two cases where our certified robustness bounds may not be tight.

\subsection{Region-Based Neighborhood IBL}%
\label{sec:App:Tightness:RegionBasedIBL}

Recall that our insertion/deletion paradigm specifies the certified robustness as
{%
  \equationSize%
  \begin{equation*}%
    \certBound
    =
      2\nLower
      -
      \nModel
      -
      1
    \text{.}
  \end{equation*}%
}%
\noindent%
Consider the case where \eqsmall{${\nLower < \nModel}$}.
After \eqsmall{$\certBound$} worst-case insertions into set~\eqsmall{$\setVals$}, the median is \eqsmall{$\frac{\max\lowerSetVals + \min\upperSetVals}{2}$}.
It is possible that \eqsmall{$\threshold$} can still exceed this value, meaning one more worst-case insertion/deletion is possible.
In other words, Eq.~\eqref{eq:Warmup:AddDelete:Lemma:Bound}'s bound would be non-tight by at most one.

To make the insertion/deletion paradigm's certified robustness bound tight, redefine Lemma~\ref{lem:Warmup:AddDelete}'s certified robustness as,
{%
  \equationSize%
  \begin{equation}%
    \label{eq:Warmup:AddDelete:Lemma:Bound:Tight}
    \certBound
    =
      2\nLower
      -
      \nModel
      -
      1
      +
      \indSize{%
        \frac{\max{\lowerSetVals} + \min{\upperSetValsPerturb}}%
             {2}
        \leq
        \threshold%
      }%
      {\bigg}
    \text{,}
  \end{equation}%
}%
\noindent%
where
\eqsmall{${\upperSetValsPerturb \defeq \upperSetVals \cup \set{\infty}}$}.%
\footnote{%
  Including~$\infty$
  addresses an edge case in Eq.~\eqref{eq:Warmup:AddDelete:Lemma:Bound:Tight} when ${\setVals = \lowerSetVals}$, and ${\min \emptyset}$ is undefined. When ${\upperSetVals \ne \emptyset}$, the $\infty$ has no effect.%
}

Since Theorem~\ref{thm:CertifiedKNN:DynamicSize} derives its bounds directly from Lemma~\ref{lem:Warmup:AddDelete}, our certified robustness bound for region-based neighborhood IBLs can also be non-tight by one.
We deliberately simplify our formulation to exclude this case to ensure consistency with the other presented ideas.

\subsection{Weighted Overlapping Certified Regression}

A weighted ensemble never has a worse \textit{true} certified robustness than its unit-cost equivalent.
However, there are cases where \overlapMultiMethod{}'s bound is one lower than \overlapSingleMethod{}'s bound, i.e.,~Fig.~\ref{fig:CertifiedGeneral:ILP}'s ILP bound is not tight.
This occurs, when for optimal~\eqsmall{$\certBound$}, it is possible to perturb \eqsmall{${\nLower - \halfModels + 1}$}~submodels (i.e.,~one more than the required \eqsmall{${\nLower - \halfModels}$} submodels).
In such cases, Eq.~\eqref{eq:WarmUp:NonUniformSwap:Bound}'s decrementing by one causes the bound to be one less than the ideal value.
It is possible to formulate a more complicated ILP to account for this corner case.
However, we deliberately keep the formulation simple, knowing \eqsmall{$\certBound$} may be marginally non-tight.
Sec.~\ref{sec:ExpRes}'s evaluation shows this niche case occurs only rarely in practice -- almost exclusively when $\nModel$ is small.
 
\clearpage
\newpage
\newcommand{\expSetupParagraph}[1]{%
  \vspace{4pt}%
  \noindent%
  \textbf{#1}~%
}%

\newcommand{\ExpSetupTabCaption}[2]{%
  \textbf{#1 Hyperparameters}:
  Hyperparameter settings for the three datasets that used #2 as the ensemble submodel architecture.
  Hyperparameters are reported for the three $\datasetDiv$~values used in Fig.~\ref{fig:ExpRes:CertificationBoundTrend} and Sec.~\ref{sec:App:MoreExps:DetailedResults}.
  We also report the hyperparameters for uncertified accuracy when ${\datasetDiv = 1}$.%
}

\section{Evaluation Setup}%
\label{sec:App:ExpSetup}

This section details the evaluation setup used in Section~\ref{sec:ExpRes}'s experiments, including implementation details, dataset configuration, and hyperparameter settings.

Our source code can be downloaded from \url{\sourceCodeUrl}.
All experiments were implemented and tested in Python~3.7.1.
Experiments were performed using one core of a fourteen-core Intel E5\=/2690v4 CPU and 12GB of RAM.
Ridge regression models were trained using Scikit-Learn~\citep{ScikitLearn}, while the decision forests used the XGBoost library~\citep{Chen:2016:XGBoost}.

The overlapping regressor ILPs (Fig.~\ref{fig:CertifiedGeneral:ILP}) were optimized using Gurobi~\citep{Gurobi} with a time limit of 1200s.
Our implementation loads the \texttt{gurobipy} \href{https://pypi.org/project/gurobipy/}{Gurobi python package} by default.
Gurobi is a commercial product and requires a license to solve most non-trivial linear programs.
Gurobi offers free, unlimited \href{https://www.gurobi.com/academia/academic-program-and-licenses/}{licenses for academic use}, including both for an individual system and for a cloud/HPC environment.

\subsection{Dataset Configuration}

Our source code automatically downloads all necessary datasets.

Regarding dataset preprocessing, categorical features were transformed into one-hot-encoded features in line with previous work~\citep{Brophy:2022:TreeInfluence,Hammoudeh:2022:GAS:Arxiv}.
Standardizing features by dataset mean/variance breaks submodel independence and so was not performed.
Minimal manual feature engineering was performed to improve the housing datasets' results, e.g.,~adding a home's age, total square feet, total number of bathrooms, etc.; this feature engineering was done based on existing features in the dataset (e.g.,~total square feet equals the sum of the first and second-floor square footage).
None of the engineered features affect submodel independence.

Most of the six datasets in Sec.~\ref{sec:ExpRes:Setup} do not have a dedicated test set.
In such cases, the data was split 90\%/10\% at random between training and test.

When training \knnMethod{} models, each feature dimension was normalized to the range~$\sbrack{0, 1}$.
Without feature normalization, \knnMethod{} generally prioritizes whichever feature has the largest magnitude.
This transformation implicitly restricts arbitrary insertions to the feature range in the original dataset.
Such normalization is implicitly done in certified classifier evaluation on image datasets where each pixel has a consistent, fixed range.

\subsection{\revised{Dataset Target Value Statistics}}

\revised{%
Table~\ref{tab:App:ExpSetup:DatasetYStats} summarizes the test set's target~($\y$) value distribution statistics for Sec.~\ref{sec:ExpRes}'s five regression datasets.

Recall from Table~\ref{tab:ExpRes:DatasetInfo} that the Ames, Austin, and Diamonds datasets set error threshold~$\threshold$ as a fixed percentage of~$\yTe$.
This choice was made because these three datasets exhibit significant $\y$~variance.
For example, for Diamonds, the largest $\y$~value (\$18.8k) is about two orders of magnitude larger than the smallest $\y$~value (\$339).
Using a fixed $\threshold$~value on these three datasets would have made certifying instances with small~$\y$ unrealistically easy while making certification of instances with large~$\y$ unreasonably difficult.
Making the error threshold a fraction of $\yTe$ allows the certification difficulty to be more consistent across the range of $\y$~values.

Datasets Weather and Life used fixed $\threshold$ values of 3~degrees (Celsius) and 3~years respectively.
Both of these threshold values are less than one-third of each dataset's $\y$ standard deviation.

Supplemental Section~\ref{sec:App:MoreExps:AltThresholds} evaluates the performance of our certified regressors on additional $\threshold$ values -- both larger and smaller than the $\threshold$ values used in Sec.~\ref{sec:ExpRes}.
}

\begin{table}[h]
  \centering
  \caption{
    \revised{%
      \textbf{Target Value Test Distribution Statistics}:
      Mean ($\bar{\y}$), standard deviation ($\sigma_{\y}$), minimum value ($\y_{\min}$) and maximum value ($\y_{\max}$) for the test instances' target $\y$~value for Sec.~\ref{sec:ExpRes}'s five regression datasets.%
    }%
  }\label{tab:App:ExpSetup:DatasetYStats}

\newcommand{\DatasetName}[1]{#1}

\newcommand{\WeatherDeg}[1]{${#1^{\circ} \text{C}}$}

\revised{%
\begin{tabular}{lllll}
    \toprule
    \textbf{Dataset}         & $\bar{\y}$         & $\sigma_{\y}$       & $\y_{\min}$         & $\y_{\max}$           \\
    \midrule
    \DatasetName{Ames}       & \$184k             & \$83.4k             & \$12.8k             & \$585k                \\
    \DatasetName{Austin}     & \$466k             & \$266k              & \$81.0k             & \$2.6M                \\
    \DatasetName{Diamonds}   & \$3.8k             & \$3.9k              & \$0.3k              & \$18.8k               \\
    \DatasetName{Weather}    & \WeatherDeg{14.9}  & \WeatherDeg{10.3}   & \WeatherDeg{-44.0}  & \WeatherDeg{54.0}     \\
    \DatasetName{Life}       & 69.3~years         & 9.6~years           & 36.3~years          & 89.0~years            \\
    \bottomrule%
\end{tabular}
}%
 \end{table}

\subsection{Hyperparameters}

Following \citepos{Jia:2022:CertifiedKNN} certified \knn{} classifier evaluation, \knnMethod{}'s neighborhood size,~$\kNeigh$, was set to the (larger) odd integer nearest to~$\frac{\nTr}{2}$.
We use the Minkowski distance as the neighborhood's distance metric.

For our ensemble regressors, hyperparameters were tuned using Bayesian optimization as implemented in the \texttt{scikit-optimize} library~\citep{ScikitOptimize}.
The partitioned and overlapping certified regressors (unweighted and weighted) used the same hyperparameter settings.

\expSetupParagraph{Ridge Regression Hyperparameters}
For three datasets -- Diamonds~\citep{DiamondsDataset}, Weather~\citep{Malinin:2021:Shifts}, and Spambase~\citep{SpambaseDataset} -- our four ensemble regressors used ridge regression as the submodel architecture.
For each dataset and $\datasetDiv$ value, we tuned three ridge regression hyperparameters.
Below, we list those hyperparameters along with the set of values considered.

\begin{itemize}
  \setlength{\itemsep}{6pt}
  \item \textit{Weight Decay}~($\hyperRidgeWD$): $L_2$~regularization strength.  We considered values between~$10^{-8}$ and~$10^{4}$.
  \item \textit{Error Tolerance}~($\hyperRidgeTol$): Minimum validation error that defines when a model is considered converged.  The tested values were \eqsmall{$\set{10^{-8}, 10^{-7}, \ldots, 10^{-3}}$}.
  \item \textit{Maximum Number of Iterations}~($\hyperRidgeItr$): Defines the maximum number of optimizer iterations.  If the error tolerance is achieved before the iteration count is met, the model is treated as converged, and optimization stops. The tested values were \eqsmall{$\set{10^{2}, 10^{3}, \ldots, 10^{8}}$}.
\end{itemize}

\noindent%
Table~\ref{tab:App:ExpSetup:Hyperparams:Ridge} lists the final hyperparameters for each experimental setup that used ridge regression as the submodel architecture.

\expSetupParagraph{XGBoost Hyperparameters}
For three datasets -- Ames Housing~\citep{DeCock:2011:AmesHousing}, Austin Housing~\citep{AustinHousingDataset}, and Life~\citep{LifeExpectancyDataset} -- our four ensemble regressors used XGBoost~\citep{Chen:2016:XGBoost} as the submodel architecture.
For each dataset and $\datasetDiv$ value, we tuned seven XGBoost hyperparameters.
Below, we list those hyperparameters along with the set of values considered.

\begin{itemize}
  \setlength{\itemsep}{6pt}
  \item \textit{Number of Trees}~($\hyperXgbNumEst$): Number of trees in the ensemble. The tested values were $\set{50, 100, 250, 500, 1000}$.
  \item \textit{Maximum Tree Depth}~($\hyperXgbMaxDepth$): Maximum depth of each tree in the ensemble. The tested values were $\set{1, \ldots, 4}$.
  \item \textit{Evaluation Metric}~($\hyperXgbMetric$): Applied to the validation set and is the metric being minimized. The tested values were root mean squared error (RMSE) and mean absolute error (MAE).
  \item \textit{Weight Decay}~($\hyperXgbWD$): $L_2$~regularization strength.  We considered values between~$10^{-3}$ and~$10^{5}$.
  \item \textit{Minimum Split Loss}~($\hyperXgbMinLoss$): Minimum reduction in loss required to split a node instead of making it a leaf. The values considered were $\set{0.005, 0.01, 0.05, 0.1, 0.3, 0.5, 1}$.
  \item \textit{Learning Rate}~($\hyperXgbLR$): Larger value makes the boosting more conservative. The tested values were $\set{0.01, 0.1, 0.3, 1}$.
\end{itemize}

\noindent%
Table~\ref{tab:App:ExpSetup:Hyperparams:XGBoost} lists the final hyperparameters for each experimental setup that used XGBoost as the submodel architecture.
Mixup~\citep{Zhang:2018:Mixup} data augmentation was used to improve XGBoost's performance.\footnote{Mixup does not apply to convex models like ridge regression.}

\vfill

\begin{table}[h!]
  \centering
  \caption{%
    \ExpSetupTabCaption{Ridge Regression}{ridge regression}
  }\label{tab:App:ExpSetup:Hyperparams:Ridge}
  {%
    \tableSize%
\newcommand{\HyperDsName}[1]{\multirow{4}{*}{#1}}

\newcommand{\IterSN}[2][1]{$#1\text{E}{#2}$}
\newcommand{\WdSN}[2]{$#1\text{E}{#2}$}
\newcommand{\TolSN}[2][1]{$#1\text{E}{-#2}$}

\newcommand{\HyperDivSep}{\cdashline{2-5}}
\newcommand{\HyperDsSep}{\midrule}

\renewcommand{\arraystretch}{1.2}
\setlength{\dashlinedash}{0.4pt}
\setlength{\dashlinegap}{1.5pt}
\setlength{\arrayrulewidth}{0.3pt}
\setlength{\tabcolsep}{8.4pt}

\begin{tabular}{crrrr}
  \toprule
  Dataset & $\datasetDiv$  & \multicolumn{1}{c}{$\hyperRidgeWD$} & \multicolumn{1}{c}{$\hyperRidgeTol$} & \multicolumn{1}{c}{$\hyperRidgeItr$}  \\
  \midrule
  \HyperDsName{Diamonds}
          & 1      & \WdSN{3.16}{-3}  & \TolSN{6}   &  \IterSN{6}     \\\HyperDivSep
          & 151    & \WdSN{6.01}{-2}  & \TolSN{7}   &  \IterSN{8}     \\\HyperDivSep
          & 501    & \WdSN{1.00}{-8}  & \TolSN{6}   &  \IterSN{8}     \\\HyperDivSep
          & 1001   & \WdSN{1.38}{-8}  & \TolSN{6}   &  \IterSN{2}     \\\HyperDsSep
  \HyperDsName{Weather}
          & 1      & \WdSN{3.16}{-3}  & \TolSN{8}   &  \IterSN{7}     \\\HyperDivSep
          & 51     & \WdSN{1.00}{+3}  & \TolSN{5}   &  \IterSN{6}     \\\HyperDivSep
          & 1501   & \WdSN{3.16}{+2}  & \TolSN{6}   &  \IterSN{2}     \\\HyperDivSep
          & 3001   & \WdSN{3.16}{+2}  & \TolSN{6}   &  \IterSN{3}     \\\HyperDsSep
  \HyperDsName{Spambase}
          & 1      & \WdSN{3.16}{+2}  & \TolSN{6}   &  \IterSN{5}     \\\HyperDivSep
          & 25     & \WdSN{3.16}{-3}  & \TolSN{6}   &  \IterSN{6}     \\\HyperDivSep
          & 151    & \WdSN{3.16}{-6}  & \TolSN{7}   &  \IterSN{6}     \\\HyperDivSep
          & 301    & \WdSN{3.16}{-3}  & \TolSN{6}   &  \IterSN{6}     \\
  \bottomrule
\end{tabular}
   }
\end{table}

\vfill

\begin{table}[h!]
  \centering
  \caption{%
    \ExpSetupTabCaption{XGBoost}{XGBoost}
  }\label{tab:App:ExpSetup:Hyperparams:XGBoost}
  {%
    \tableSize%
\newcommand{\HyperDsName}[1]{\multirow{4}{*}{#1}}

\newcommand{\HyperParamHead}[1]{\multicolumn{1}{c}{#1}}

\newcommand{\HyperDivSep}{\cdashline{2-8}}
\newcommand{\HyperDsSep}{\midrule}

\renewcommand{\arraystretch}{1.2}
\setlength{\dashlinedash}{0.4pt}
\setlength{\dashlinegap}{1.5pt}
\setlength{\arrayrulewidth}{0.3pt}
\setlength{\tabcolsep}{8.4pt}

\newcommand{\NEst}[1]{#1}
\newcommand{\MaxDepth}[1]{#1}
\newcommand{\XgbMetric}[1]{#1}
\newcommand{\WdSN}[2][1]{$#1\text{E}{#2}$}
\newcommand{\MinLoss}[2][1]{$#1\text{E}{#2}$}
\newcommand{\LrSN}[1]{#1}

\newcommand{\Nz}{}

\begin{tabular}{crrrrrrr}
  \toprule
  Dataset & $\datasetDiv$
          & \HyperParamHead{$\hyperXgbNumEst$}
          & \HyperParamHead{$\hyperXgbMaxDepth$}
          & \HyperParamHead{$\hyperXgbMetric$}
          & \HyperParamHead{$\hyperXgbWD$}
          & \HyperParamHead{$\hyperXgbMinLoss$}
          & \HyperParamHead{$\hyperXgbLR$} \\
  \midrule
  \HyperDsName{Ames Housing}
          & 1      & \NEst{250}  & \MaxDepth{2}  & \XgbMetric{RMSE}  & \WdSN{-1}    & \MinLoss[5]{-3}   & \LrSN{0.3\Nz}  \\\HyperDivSep
          & 25     & \NEst{500}  & \MaxDepth{2}  & \XgbMetric{MAE}   & \WdSN{-3}    & \MinLoss[5]{-3}   & \LrSN{0.3\Nz}  \\\HyperDivSep
          & 125    & \NEst{500}  & \MaxDepth{3}  & \XgbMetric{RMSE}  & \WdSN{-2}    & \MinLoss[5]{-3}   & \LrSN{1.0\Nz}  \\\HyperDivSep
          & 251    & \NEst{250}  & \MaxDepth{1}  & \XgbMetric{RMSE}  & \WdSN{-1}    & \MinLoss[5]{-3}   & \LrSN{1.0\Nz}  \\\HyperDsSep
  \HyperDsName{Austin Housing}
          & 1      & \NEst{500}  & \MaxDepth{4}  & \XgbMetric{MAE}   & \WdSN{+2}    & \MinLoss{-2}      & \LrSN{0.3\Nz}   \\\HyperDivSep
          & 151    & \NEst{1000} & \MaxDepth{1}  & \XgbMetric{RMSE}  & \WdSN{-2}    & \MinLoss[5]{-3}   & \LrSN{1.0\Nz}   \\\HyperDivSep
          & 301    & \NEst{250}  & \MaxDepth{1}  & \XgbMetric{MAE}   & \WdSN{+0}    & \MinLoss{-2}      & \LrSN{1.0\Nz}   \\\HyperDivSep
          & 701    & \NEst{250}  & \MaxDepth{1}  & \XgbMetric{MAE}   & \WdSN{-2}    & \MinLoss{-2}      & \LrSN{1.0\Nz}   \\\HyperDsSep
  \HyperDsName{Life}
          & 1      & \NEst{500}  & \MaxDepth{5}  & \XgbMetric{RMSE}  & \WdSN{+1}    & \MinLoss{-2}      & \LrSN{0.1\Nz}   \\\HyperDivSep
          & 25     & \NEst{250}  & \MaxDepth{4}  & \XgbMetric{RMSE}  & \WdSN[0]{+0} & \MinLoss[5]{-2}   & \LrSN{0.3\Nz}   \\\HyperDivSep
          & 101    & \NEst{250}  & \MaxDepth{3}  & \XgbMetric{MAE}   & \WdSN{+0}    & \MinLoss{-2}      & \LrSN{1.0\Nz}   \\\HyperDivSep
          & 201    & \NEst{250}  & \MaxDepth{4}  & \XgbMetric{RMSE}  & \WdSN[0]{+0} & \MinLoss[5]{-3}   & \LrSN{0.3\Nz}   \\
  \bottomrule
\end{tabular}
   }
\end{table}

\vfill
 
\clearpage
\newpage
\section{Additional Experiments}%
\label{sec:App:MoreExps}

Limited space prevents us from including all experimental results in the main paper.
We provide additional results below.

\newcommand\primitiveinput[1]
{\input{#1} }

\newcommand{\DivHeader}[2][5]{\multirow{#1}{*}{$#2$}}
\newcommand{\DivOne}{\DivHeader[1]{1}}

\newcommand{\BoundDist}[1]{$#1$}
\newcommand{\BoundVal}[2]{#1 $\pm$ #2}%
\newcommand{\BoundValB}[2]{\textBF{#1} $\pm$ \textBF{#2}}

\newcommand{\DatasetDivSep}{\midrule}
\newcommand{\BoundSep}{\cdashline{2-7}}

\newcommand{\DetailedMethod}[1]{\multicolumn{1}{c}{#1}}

\newcommand{\FullResultsTable}[1]{%
  \renewcommand{\arraystretch}{1.2}
  \setlength{\dashlinedash}{0.4pt}
  \setlength{\dashlinegap}{1.5pt}
  \setlength{\arrayrulewidth}{0.3pt}
  \setlength{\tabcolsep}{8.4pt}

  \begin{tabular}{rrrrrrr}
    \toprule
    $\datasetDiv$
      & $\certBound$
      & \DetailedMethod{\disjointSingleMethod{}}
      & \DetailedMethod{\overlapSingleMethod{}}
      & \DetailedMethod{\disjointMultiMethod{}}
      & \DetailedMethod{\overlapMultiMethod{}}
      & \DetailedMethod{\knnMethod{}}  \\\midrule
    \input{tables/data/#1_full_results.tex}
    \bottomrule
  \end{tabular}
}%
 
\newcommand{\FullResultsCaption}[2]{%
    \textbf{#1 Full Results}:
    Certified accuracy mean and standard deviation for the #1~\citep{#2} dataset.
    Each ensemble submodel was trained on \eqsmall{$\frac{1}{\datasetDiv}$}\=/th of the training set with three $\datasetDiv$~values tested per dataset, while
    \knnMethod{} was always trained on the whole training set (i.e.,~${\datasetDiv = 1}$).
    The certified accuracy results of five robustness values~(\eqsmall{$\certBound$}) are reported per $\datasetDiv$~value.
    Also reported as a baseline is the uncertified accuracy (\eqsmall{${\certBound = 0}$}) when training a single model on all of training set~\eqsmall{$\trainSet$} (\eqsmall{${\datasetDiv = 1}$}).
    Results are averaged across 10~trials per method, with each \eqsmall{$\certBound$}'s best mean certified accuracy in bold.%
}%

\subsection{Detailed Experimental Results}%
\label{sec:App:MoreExps:DetailedResults}

\noindent%
\subsubsection{Baseline Accuracy}
\label{sec:App:MoreExps:DetailedResults:Baseline}

Table~\ref{tab:App:MoreExps:DetailedResults:Baseline} shows the baseline accuracy when a model is trained on all of training set~$\trainSet$ (i.e.,~${\datasetDiv = 1}$).
For each dataset, the model architecture (either ridge regression or XGBoost) aligns with those used for Sec.~\ref{sec:ExpRes}'s ensembles.
See Table~\ref{tab:ExpRes:DatasetInfo}.

\null{}
\begin{table}[h!]
  \centering
  \caption{%
    \textbf{Baseline Accuracy}:
    Summary of the baseline (i.e.,~uncertified) accuracy mean and standard deviation for Sec.~\ref{sec:ExpRes}'s six datasets.
    Submodels were trained on all of training set~$\trainSet$ (i.e.,~${\datasetDiv = 1}$).
    Beside each dataset's name is the submodel architecture used by the ensemble.
    Threshold $\threshold$ matches values in Table~\ref{tab:ExpRes:DatasetInfo}.%
  }\label{tab:App:MoreExps:DetailedResults:Baseline}
  {%
    \tableSize%
\newcommand{\baseXGB}{XGBoost}
\newcommand{\baseRidge}{Ridge}

\newcommand{\baseVal}[2]{${#1 \pm #2}$}

\begin{tabular}{llr}
  \toprule
  Dataset     & Submodel      & Base Acc.~(\%) \\
  \midrule
  Ames        & \baseXGB{}    & \baseVal{90.4}{2.4}  \\%
  Austin      & \baseXGB{}    & \baseVal{71.3}{4.1}  \\%
  Diamonds    & \baseRidge{}  & \baseVal{73.6}{4.0}  \\%
  Weather     & \baseRidge{}  & \baseVal{85.9}{3.4}  \\%
  Life        & \baseXGB{}    & \baseVal{92.7}{3.1}  \\%
  Spambase    & \baseRidge{}  & \baseVal{87.5}{2.9}  \\
  \bottomrule
\end{tabular}
   }
\end{table}

\noindent%
\subsubsection{Numerical Results}%
\label{sec:App:MoreExps:DetailedResults:Numerical}

Fig.~\ref{fig:ExpRes:CertificationBoundTrend} visualizes our certified regressors' certified robustness on six datasets -- five regression and one binary classification.
This section provides the certified accuracy in numerical form, including the associated variance.

\null\vfill{}
\begin{table}[h!]
  \centering
  \caption{%
    \FullResultsCaption{Ames Housing}{DeCock:2011:AmesHousing}%
  }%
  \label{tab:App:MoreExps:DetailedResults:Ames}%
  {%
    \tableSize%
  \renewcommand{\arraystretch}{1.2}
  \setlength{\dashlinedash}{0.4pt}
  \setlength{\dashlinegap}{1.5pt}
  \setlength{\arrayrulewidth}{0.3pt}
  \setlength{\tabcolsep}{8.4pt}

  \begin{tabular}{rrrrrrr}
    \toprule
    $\datasetDiv$
      & $\certBound$
      & \DetailedMethod{\disjointSingleMethod{}}
      & \DetailedMethod{\overlapSingleMethod{}}
      & \DetailedMethod{\disjointMultiMethod{}}
      & \DetailedMethod{\overlapMultiMethod{}}
      & \DetailedMethod{\knnMethod{}}  \\\midrule
    \input{tables/data/ames-housing_full_results.tex}
    \bottomrule
  \end{tabular}
  }
\end{table}

\null\vfill{}
\begin{table}[h!]
  \centering
  \caption{%
    \FullResultsCaption{Austin Housing}{AustinHousingDataset}%
  }%
  \label{tab:App:MoreExps:DetailedResults:Austin}%
  {%
    \tableSize%
  \renewcommand{\arraystretch}{1.2}
  \setlength{\dashlinedash}{0.4pt}
  \setlength{\dashlinegap}{1.5pt}
  \setlength{\arrayrulewidth}{0.3pt}
  \setlength{\tabcolsep}{8.4pt}

  \begin{tabular}{rrrrrrr}
    \toprule
    $\datasetDiv$
      & $\certBound$
      & \DetailedMethod{\disjointSingleMethod{}}
      & \DetailedMethod{\overlapSingleMethod{}}
      & \DetailedMethod{\disjointMultiMethod{}}
      & \DetailedMethod{\overlapMultiMethod{}}
      & \DetailedMethod{\knnMethod{}}  \\\midrule
    \input{tables/data/austin-housing_full_results.tex}
    \bottomrule
  \end{tabular}
  }
\end{table}
\null\vfill{}

\begin{table}[t]
  \centering
  \caption{%
    \FullResultsCaption{Diamonds}{DiamondsDataset}%
  }%
  \label{tab:App:MoreExps:DetailedResults:Diamonds}
  {%
    \tableSize%
  \renewcommand{\arraystretch}{1.2}
  \setlength{\dashlinedash}{0.4pt}
  \setlength{\dashlinegap}{1.5pt}
  \setlength{\arrayrulewidth}{0.3pt}
  \setlength{\tabcolsep}{8.4pt}

  \begin{tabular}{rrrrrrr}
    \toprule
    $\datasetDiv$
      & $\certBound$
      & \DetailedMethod{\disjointSingleMethod{}}
      & \DetailedMethod{\overlapSingleMethod{}}
      & \DetailedMethod{\disjointMultiMethod{}}
      & \DetailedMethod{\overlapMultiMethod{}}
      & \DetailedMethod{\knnMethod{}}  \\\midrule
    \input{tables/data/diamonds_full_results.tex}
    \bottomrule
  \end{tabular}
  }
\end{table}

\begin{table}[t]
  \centering
  \caption{%
    \FullResultsCaption{Weather}{Malinin:2021:Shifts}%
  }%
  \label{tab:App:MoreExps:DetailedResults:ShiftsWeatherMedium}%
  {%
    \tableSize%
  \renewcommand{\arraystretch}{1.2}
  \setlength{\dashlinedash}{0.4pt}
  \setlength{\dashlinegap}{1.5pt}
  \setlength{\arrayrulewidth}{0.3pt}
  \setlength{\tabcolsep}{8.4pt}

  \begin{tabular}{rrrrrrr}
    \toprule
    $\datasetDiv$
      & $\certBound$
      & \DetailedMethod{\disjointSingleMethod{}}
      & \DetailedMethod{\overlapSingleMethod{}}
      & \DetailedMethod{\disjointMultiMethod{}}
      & \DetailedMethod{\overlapMultiMethod{}}
      & \DetailedMethod{\knnMethod{}}  \\\midrule
    \input{tables/data/shifts-weather-medium_full_results.tex}
    \bottomrule
  \end{tabular}
  }
\end{table}

\begin{table}[t]
  \centering
  \caption{%
    \FullResultsCaption{Life}{LifeExpectancyDataset}%
  }%
  \label{tab:App:MoreExps:DetailedResults:Life}%
  {%
    \tableSize%
  \renewcommand{\arraystretch}{1.2}
  \setlength{\dashlinedash}{0.4pt}
  \setlength{\dashlinegap}{1.5pt}
  \setlength{\arrayrulewidth}{0.3pt}
  \setlength{\tabcolsep}{8.4pt}

  \begin{tabular}{rrrrrrr}
    \toprule
    $\datasetDiv$
      & $\certBound$
      & \DetailedMethod{\disjointSingleMethod{}}
      & \DetailedMethod{\overlapSingleMethod{}}
      & \DetailedMethod{\disjointMultiMethod{}}
      & \DetailedMethod{\overlapMultiMethod{}}
      & \DetailedMethod{\knnMethod{}}  \\\midrule
    \input{tables/data/life_full_results.tex}
    \bottomrule
  \end{tabular}
  }
\end{table}

\begin{table}[t]
  \centering
  \caption{%
    \FullResultsCaption{Spambase}{SpambaseDataset}%
  }%
  \label{tab:App:MoreExps:DetailedResults:Spambase}%
  {%
    \tableSize%
  \renewcommand{\arraystretch}{1.2}
  \setlength{\dashlinedash}{0.4pt}
  \setlength{\dashlinegap}{1.5pt}
  \setlength{\arrayrulewidth}{0.3pt}
  \setlength{\tabcolsep}{8.4pt}

  \begin{tabular}{rrrrrrr}
    \toprule
    $\datasetDiv$
      & $\certBound$
      & \DetailedMethod{\disjointSingleMethod{}}
      & \DetailedMethod{\overlapSingleMethod{}}
      & \DetailedMethod{\disjointMultiMethod{}}
      & \DetailedMethod{\overlapMultiMethod{}}
      & \DetailedMethod{\knnMethod{}}  \\\midrule
    \input{tables/data/spambase_full_results.tex}
    \bottomrule
  \end{tabular}
  }
\end{table}
 
\newpage
\clearpage
\FloatBarrier
\subsection{\knnMethod{} Full Certified Accuracy Plots}%
\label{sec:App:MoreExps:FullKNN}

To improve readability, Fig.~\ref{fig:ExpRes:CertificationBoundTrend} does not show \knnMethod{}'s full certified accuracy trend.
Instead, Fig.~\ref{fig:App:MoreExps:KNNFullBound} below plots \knnMethod{}'s full mean certified accuracy against that of \overlapMultiMethod{} (using each dataset's maximum $\datasetDiv$~value) for each of Sec.~\ref{sec:ExpRes}'s six datasets.
Fig.~\ref{fig:App:MoreExps:KNNFullBound} also visualizes the variance of each method by showing one standard deviation of the certified accuracy as a shaded region around the mean line.
In summary, while \overlapMultiMethod{} certifies more instances (i.e.,~has larger peak certified accuracy), its maximum certified robustness~$\certBound$ is (significantly) smaller than that of~\knnMethod{}.

\vfill

\begin{figure*}[h!]
  \centering
\newcommand{\boundLegendFontSize}{\small}

\begin{tikzpicture}
  \begin{axis}[%
      hide axis,  %
      xmin=0,  %
      xmax=1,
      ymin=0,
      ymax=1,
      scale only axis,width=1mm, %
      line width=\boundLineWidthLegend,
      legend cell align={left},              %
      legend style={font=\boundLegendFontSize},  %
      legend columns=5,
      legend image post style={xscale=0.6},  %
      legend style={/tikz/every even column/.append style={column sep=0.5cm}}
    ]
    \addplot [
      knn trend line,
      line width=0.9pt,
    ] coordinates {(0,0)};
    \addlegendentry{\knnMethod{} (${\datasetDiv = 1}$)}

    \addplot [
      multi overlap line,
      line width=0.9pt,
    ] coordinates {(0,0)};
    \addlegendentry{\overlapMultiMethod{} ($\datasetDiv$ Varies)}
\end{axis}
\end{tikzpicture}

\newcommand{\knnBoundTrend}[9]{%
  \centering
  \pgfplotstableread[col sep=comma] {plots/data/#1}\thedata%
  \begin{tikzpicture}
    \begin{axis}[
      width={#2},%
      height={#3},%
      xmin={0},%
      xmax={#4},%
      xtick distance={#5},
      minor x tick num={3},
      x tick label style={font=\boundFontSize,align=center},%
      xlabel={\boundFontSize Certified Robustness~($\certBound$)},
      xmajorgrids,
      axis x line*=bottom,  %
      ymin=0,
      ymax={#6},
      ytick distance={#7},
      minor y tick num={3},
      y tick label style={font=\boundYAxisFontSize,align=center},%
      ylabel style={font=\boundYAxisFontSize,align=center},%
      ylabel={#8},
      ymajorgrids,
      axis y line*=left,  %
      mark size=0pt,
      #9,
      ]

      \knnShadePlot{NuPCR}{multi overlap line}{overlap line fill}
      \knnShadePlot{KNN}{knn trend line}{knn line fill}
    \end{axis}
  \end{tikzpicture}
}

\newcommand{\knnShadePlot}[3]{%
  \addplot [name path=#1 lower, line fill base] table [
      x=X, y expr=\thisrow{#1-Mean} - \thisrow{#1-Stdev}]{\thedata};
  \addplot [name path=#1 upper, line fill base] table [
      x=X, y expr=\thisrow{#1-Mean} + \thisrow{#1-Stdev}]{\thedata};
  \addplot[#3] fill between[of=#1 lower and #1 upper];
  \addplot[#2] table [x=X, y=#1-Mean] \thedata;
}
 
\newcommand{\knnBoundDatasetSpacer}{\vspace{15pt}}
\newcommand{\knnBoundLegendSpacer}{\vspace{10pt}}

\newcommand{\knnMiniWidth}{0.460\textwidth}
\newcommand{\knnPlotWidth}{\textwidth}
\newcommand{\knnBoundTrendHeight}{1.65in}
\newcommand{\knnPlotHorizontalSpacer}{\hfill}

\newcommand{\knnMaxYValue}{90}
\newcommand{\knnMaxYTickDistance}{20}
\newcommand{\knnSkipZeroTick}{yticklabels={,,20,40,60,80},}

\newcommand{\knnBoundPlot}[4]{
  \begin{subfigure}[b]{\knnMiniWidth}
    \knnBoundTrend{#1/#1_knn_00001_filt.csv}   %
                  {\knnPlotWidth}         %
                  {\knnBoundTrendHeight}  %
                  {#3}                    %
                  {#4}                    %
                  {\knnMaxYValue}         %
                  {\knnMaxYTickDistance}  %
                  {\certifiedAccStr}      %
                  {\knnSkipZeroTick}      %
    \caption{#2}
  \end{subfigure}
}

  \knnBoundLegendSpacer{}
  \knnBoundPlot{ames}%
               {Ames Housing}%
               { 295}%
               {  50}%
  \knnPlotHorizontalSpacer
  \knnBoundPlot{austin}%
               {Austin Housing}%
               { 950}%
               { 200}%

  \knnBoundDatasetSpacer{}
  \knnBoundPlot{diamonds}%
               {Diamonds}%
               {1300}%
               { 300}%
  \knnPlotHorizontalSpacer
  \knnBoundPlot{shifts-weather-medium}%
               {Weather}%
               {16000}%
               { 3000}%

  \knnBoundDatasetSpacer{}
  \knnBoundPlot{life}%
               {Life}%
               { 190}%
               {  40}%
  \knnPlotHorizontalSpacer
  \knnBoundPlot{spambase}%
               {Spambase}%
               { 525}%
               { 100}%

  \caption{%
    \textbf{\knnMethod{} vs.\ \overlapMultiMethod{} Certified Accuracy}:
    Full plots of the mean certified accuracy for Sec.~\ref{sec:ExpRes}'s six datasets.
    The shaded regions visualize one standard deviation of the certified accuracy for each $\certBound$ value.
    \overlapMultiMethod{}'s $\datasetDiv$~value for each dataset is in Table~\ref{tab:App:MoreExps:KNNFullBound:DivVals}.
  }
  \label{fig:App:MoreExps:KNNFullBound}
\end{figure*}
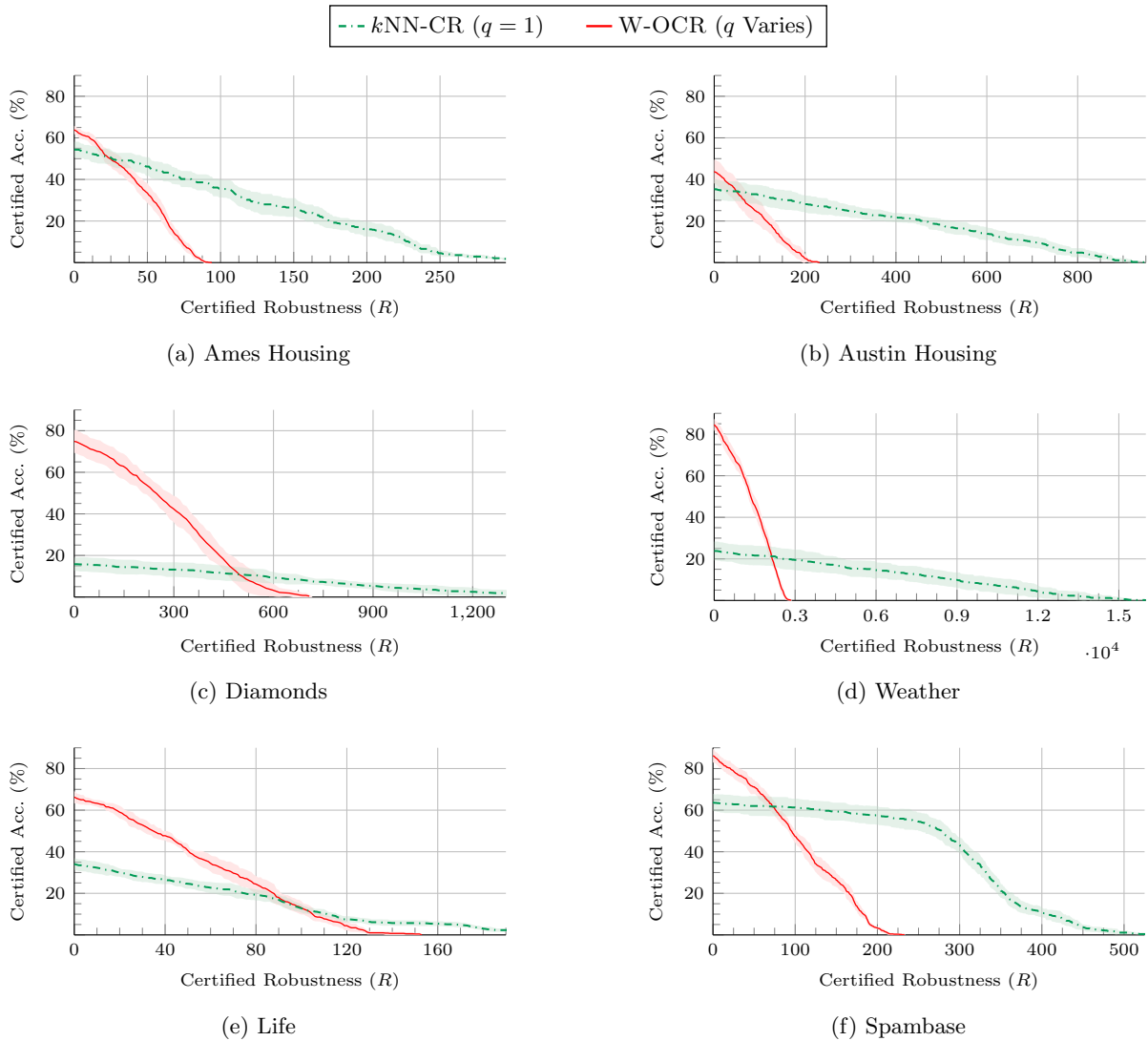

\vfill

\begin{table*}[h!]
  \centering
  \caption{%
    \textbf{\overlapMultiMethod{} $\datasetDiv$ Values}:
    As detailed in Sec.~\ref{sec:ExpRes:Setup}, ensemble submodels were trained on \eqsmall{$\frac{1}{\datasetDiv}$}\=/th of the training data where ${\datasetDiv}$ varies by dataset.
    Below are the \overlapMultiMethod{} $\datasetDiv$~values used in Fig.~\ref{fig:App:MoreExps:KNNFullBound}.
  }\label{tab:App:MoreExps:KNNFullBound:DivVals}

\begin{tabular}{ccccccc}
  \toprule
  Dataset
    & Ames & Austin & Diamonds & Weather & Life & Spambase \\
  \midrule
  $\datasetDiv{}$
    & 251  & 701    & 1,001    & 3,001   & 201  & 301 \\
  \bottomrule
\end{tabular}
 \end{table*}
 
\newpage
\clearpage
\FloatBarrier
\subsection{Alternative \texorpdfstring{$\threshold$}{Threshold} Evaluation}%
\label{sec:App:MoreExps:AltThresholds}

For each dataset in Sec.~\ref{sec:ExpRes}, we evaluate a single threshold value~$\threshold$.
This section explores how $\threshold$ affects our certified regressors' performance on four regression datasets.\footnote{%
For binary classification (e.g.,~Spambase~\citep{SpambaseDataset}), alternate $\threshold$~values do not apply.
}
Specifically, Fig.~\ref{fig:App:MoreExps:AltThreshold:Plots} considers our two best-performing methods, \overlapMultiMethod{} and \knnMethod{}.
Like in Sec~\ref{sec:ExpRes}, \knnMethod{} uses ${\datasetDiv = 1}$.
For \overlapMultiMethod{}, we report results for each dataset's two largest $\datasetDiv$ values.
Note that in Table~\ref{tab:ExpRes:DatasetInfo}, Ames and Diamonds use the same (default) threshold~${\threshold = 15\% \cdot \y}$ while both Life and Weather use (default) value~${\threshold = 3}$.
In Fig.~\ref{fig:App:MoreExps:AltThreshold:Plots} below, each dataset pair considers the same alternate $\threshold$ values.

$\threshold$'s exact effect varies across datasets, but generally, certified accuracy increases roughly linearly with~$\threshold$ until the accuracy saturates.

\vspace{-10pt}
\begin{figure}[h!]
  \centering
\newcommand{\altDistFontSize}{\scriptsize}%
\newcommand{\altTrendFullPlot}[9]{%
  \centering%
  \pgfplotstableread[col sep=comma] {plots/data/#1}\thedata%
  \begin{tikzpicture}%
    \begin{axis}[%
        scale only axis,%
        width={\altDistPlotWidth},%
        height={\altDistPlotHeight},%
        xmin={0},%
        xmax={#2},%
        xtick distance={#3},%
        minor x tick num={3},%
        x tick label style={font=\altDistFontSize,align=center},%
        scaled x ticks=false,
        xlabel={\altDistFontSize #4},%
        xmajorgrids,%
        axis x line*=bottom,  %
        ymin=0,%
        ymax={#5},%
        ytick distance={#6},%
        minor y tick num={3},%
        yticklabels={,,,,,},%
        y tick label style={font=\altDistFontSize,align=center},%
        ylabel style={font=\altDistFontSize,align=center},%
        ylabel={#7},%
        ymajorgrids,%
        axis y line*=left,  %
        mark size=0pt,%
        #8%
      ]%
      #9%
    \end{axis}%
  \end{tikzpicture}%
}%
\newcommand{\singleAltPlot}[3]{%
  \addplot [name path=#1 lower, line fill base] table [%
      x=X, y expr=\thisrow{#1-Mean} - \thisrow{#1-Stdev}]{\thedata};%
  \addplot [name path=#1 upper, line fill base] table [%
      x=X, y expr=\thisrow{#1-Mean} + \thisrow{#1-Stdev}]{\thedata};%
  \addplot[#3] fill between[of=#1 lower and #1 upper];%
  \addplot[#2] table [x=X, y=#1-Mean] \thedata;%
}
\newcommand{\altDistLegFontSize}{\small}

\newcommand{\altDistLegend}[2]{
  \begin{tikzpicture}
    \begin{axis}[%
        hide axis,  %
        no marks,
        xmin=0,  %
        xmax=1,
        ymin=0,
        ymax=1,
        scale only axis,width=1mm, %
        legend cell align={left},              %
        legend style={font=\altDistLegFontSize},
        legend columns=#1,
        legend style={/tikz/every even column/.append style={column sep=0.5cm}}
      ]
      #2
    \end{axis}
  \end{tikzpicture}
}

\newcommand{\altDistLegEntry}[2]{%
    \addplot [%
      #1,
      line width=0.9pt,
    ] coordinates {(0,0)};
    \addlegendentry{#2}%
}
\newcommand{\altDistFirstMiniWidth}{0.355\textwidth}%
\newcommand{\altDistOtherMiniWidth}{0.305\textwidth}%
\newcommand{\altDistPlotHeight}{1.25in}%
\newcommand{\altDistPlotWidth}{2.10in}%
\newcommand{\altDistLegendSpacer}{\vspace{6pt}}%
\newcommand{\altDistDatasetSpacer}{%
  \vspace{15pt}%
  \noindent%
}%
\newcommand{\altDistHorizontalSpacer}{}%
\newcommand{\altDistYMax}{100}%
\newcommand{\altDistYTickDist}{20}%
\newcommand{\altDistSkipZeroTick}{yticklabels={,,20,40,60,80,100},}%
\newcommand{\altDistXLabel}[2]{#1 Certified Robustness~$\certBound$ (${\datasetDiv = #2}$)}%
\newcommand{\legVal}[1]{${\threshold = #1}$}%
\newcommand{\legValDefault}[1]{\legVal{#1}~(Default)}%
\newcommand{\legValY}[1]{${\threshold = #1 \cdot \y}$}%
\newcommand{\legValDefaultY}[1]{\legValY{#1}~(Default)}%
\newcommand{\altDistYLabel}[1]{\textbf{#1} \\ Certified Acc.\ (\%)}%
\newcommand{\altDistNoYLabel}{ylabel={}}
\newcommand{\altDistCaption}[3]{%
  \caption{#2 (${\datasetDiv = #3}$)}%
}%
\newcommand{\altLegendCombo}[2]{%
  \begin{center}
    \altDistLegend{#1}{#2}
  \end{center}

  \altDistLegendSpacer{}
}%
\newcommand{\knnAltDistPlot}[5]{%
  \begin{subfigure}[b]{\altDistOtherMiniWidth}%
    \centering%
    \altTrendFullPlot{#1/alt-dist/knn/#1_alt-dist_knn_filt.csv}  %
                     {#3}                    %
                     {#4}                    %
                     {\altDistXLabel{\knnMethod}{1}} %
                     {\altDistYMax}          %
                     {\altDistYTickDist}     %
                     {}                      %
                     {}                      %
                     {#5}                    %
  \end{subfigure}%
}%
\newcommand{\ensembleAltDistPlot}[9]{%
  \begin{subfigure}[b]{#5}
    \centering%
    \altTrendFullPlot{#1/alt-dist/#3/#1_alt-dist_#3_#4_filt.csv}  %
                     {#6}                    %
                     {#7}                    %
                     {\altDistXLabel{\overlapMultiMethod}{\num{#4}}} %
                     {\altDistYMax}          %
                     {\altDistYTickDist}     %
                     {\altDistYLabel{#2}}    %
                     {#8}                    %
                     {#9}                    %
  \end{subfigure}
}%
\newcommand{\altTrend}[1]{knn alt#1 trend line}%
\newcommand{\altFill}[1]{knn alt#1 line fill}%
\newcommand{\generalRelativeAltPlots}[1]{%
  \singleAltPlot{#1-D05}%
                {\altTrend{01}}%
                {\altFill{01}}%
  \singleAltPlot{#1-D10}%
                {\altTrend{02}}%
                {\altFill{02}}%
  \singleAltPlot{#1-D15}%
                {\altTrend{-main}}%
                {\altFill{-main}}%
  \singleAltPlot{#1-D20}%
                {\altTrend{03}}%
                {\altFill{03}}%
  \singleAltPlot{#1-D25}%
                {\altTrend{04}}%
                {\altFill{04}}%
}%
\newcommand{\generalFixedAltPlots}[1]{%
  \singleAltPlot{#1-D01}%
                {\altTrend{01}}%
                {\altFill{01}}%
  \singleAltPlot{#1-D02}%
                {\altTrend{02}}%
                {\altFill{02}}%
  \singleAltPlot{#1-D03}%
                {\altTrend{-main}}%
                {\altFill{-main}}%
  \singleAltPlot{#1-D04}%
                {\altTrend{03}}%
                {\altFill{03}}%
  \singleAltPlot{#1-D05}%
                {\altTrend{04}}%
                {\altFill{04}}%
  \singleAltPlot{#1-D07}%
                {\altTrend{05}}%
                {\altFill{05}}%
  \singleAltPlot{#1-D09}%
                {\altTrend{06}}%
                {\altFill{06}}%
}%

\newcommand{\amesHousingLegendEntries}{%
  \altDistLegEntry{\altTrend{01}}%
                  {\legValY{5\%}}%
  \altDistLegEntry{\altTrend{02}}%
                  {\legValY{10\%}}%
  \altDistLegEntry{\altTrend{-main}}%
                  {\legValDefaultY{15\%}}%
  \altDistLegEntry{\altTrend{03}}%
                  {\legValY{20\%}}%
  \altDistLegEntry{\altTrend{04}}%
                  {\legValY{25\%}}%
}
\newcommand{\amesEnsemblePlot}[5]{%
  \ensembleAltDistPlot{ames}                   %
                      {Ames Housing}           %
                      {wocr}                   %
                      {#1}                     %
                      {#5}                     %
                      {#2}                     %
                      {#3}                     %
                      {#4}                     %
                      {\generalRelativeAltPlots{WOCR}}%
}%

\altDistDatasetSpacer{}%
\altLegendCombo{5}{\amesHousingLegendEntries}
\amesEnsemblePlot{00125}                  %
                 {70}                     %
                 {15}                     %
                 {\altDistSkipZeroTick}   %
                 {\altDistFirstMiniWidth} %
\altDistHorizontalSpacer{}%
\amesEnsemblePlot{00251}                  %
                 {149}                    %
                 {30}                     %
                 {\altDistNoYLabel}       %
                 {\altDistOtherMiniWidth} %
\altDistHorizontalSpacer{}%
\knnAltDistPlot{ames}%
               {Ames Housing}%
               {450}%
               {100}%
               {\generalRelativeAltPlots{Knn}}%

\newcommand{\diamondsEnsemblePlot}[5]{%
  \ensembleAltDistPlot{diamonds}               %
                      {Diamonds}               %
                      {wocr}                   %
                      {#1}                     %
                      {#5}                     %
                      {#2}                     %
                      {#3}                     %
                      {#4}                     %
                      {\generalRelativeAltPlots{WOCR}}%
}%
\altDistDatasetSpacer{}%
\diamondsEnsemblePlot{00501}                  %
                     {520}                    %
                     {100}                    %
                     {\altDistSkipZeroTick}   %
                     {\altDistFirstMiniWidth} %
\altDistHorizontalSpacer{}%
\diamondsEnsemblePlot{01001}                  %
                     {925}                    %
                     {200}                    %
                     {\altDistNoYLabel}       %
                     {\altDistOtherMiniWidth} %
\altDistHorizontalSpacer{}%
\knnAltDistPlot{diamonds}%
               {Diamonds}%
               {1700}%
               {400}%
               {\generalRelativeAltPlots{Knn}}%

\newcommand{\weatherLegendEntries}{%
  \altDistLegEntry{\altTrend{01}}%
                  {\legVal{1}}%
  \altDistLegEntry{\altTrend{02}}%
                  {\legVal{2}}%
  \altDistLegEntry{\altTrend{-main}}%
                  {\legValDefault{3}}%
  \altDistLegEntry{\altTrend{03}}%
                  {\legVal{4}}%
  \altDistLegEntry{\altTrend{04}}%
                  {\legVal{5}}%
  \altDistLegEntry{\altTrend{05}}%
                  {\legVal{7}}%
  \altDistLegEntry{\altTrend{06}}%
                  {\legVal{9}}%
}%
\newcommand{\weatherEnsemblePlot}[5]{%
  \ensembleAltDistPlot{shifts-weather-medium}  %
                      {Weather}                %
                      {wocr}                   %
                      {#1}                     %
                      {#5}                     %
                      {#2}                     %
                      {#3}                     %
                      {#4}                     %
                      {\generalFixedAltPlots{WOCR}}%
}%
\newcommand{\weatherKnnAltPlots}{%
  \generalFixedAltPlots{Knn}%
}%
\altDistDatasetSpacer{}%
\altLegendCombo{7}{\weatherLegendEntries}%
\weatherEnsemblePlot{01501}                  %
                    {1499}                   %
                    {300}                    %
                    {\altDistSkipZeroTick}   %
                    {\altDistFirstMiniWidth} %
\altDistHorizontalSpacer{}%
\weatherEnsemblePlot{03001}                  %
                    {2999}                   %
                    {600}                    %
                    {\altDistNoYLabel}       %
                    {\altDistOtherMiniWidth} %
\altDistHorizontalSpacer{}%
\knnAltDistPlot{shifts-weather-medium}%
               {Weather}%
               {18000}%
               {4000}%
               {\weatherKnnAltPlots}%

\newcommand{\lifeEnsemblePlot}[5]{%
  \ensembleAltDistPlot{life}                   %
                      {Life}                   %
                      {wocr}                   %
                      {#1}                     %
                      {#5}                     %
                      {#2}                     %
                      {#3}                     %
                      {#4}                     %
                      {\generalFixedAltPlots{WOCR}}%
}%
\newcommand{\lifeKnnAltPlots}{%
  \generalFixedAltPlots{Knn}%
}%
\altDistDatasetSpacer{}%
\lifeEnsemblePlot{00101}                  %
                 {103}                     %
                 {20}                     %
                 {\altDistSkipZeroTick}   %
                 {\altDistFirstMiniWidth} %
\altDistHorizontalSpacer{}%
\lifeEnsemblePlot{00201}                  %
                 {206}                    %
                 {50}                     %
                 {\altDistNoYLabel}       %
                 {\altDistOtherMiniWidth} %
\altDistHorizontalSpacer{}%
\knnAltDistPlot{life}%
               {Life}%
               {350}%
               {75}%
               {\lifeKnnAltPlots}%

  \caption{%
    \textbf{Alternative $\threshold$ Values}:
    Mean certified accuracy of our two best performing methods, \overlapMultiMethod{} and \knnMethod{}, across alternative error thresholds~$\threshold$.
    For \overlapMultiMethod{}, we consider each dataset's two largest $\datasetDiv$ values from Fig.~\ref{fig:ExpRes:CertificationBoundTrend}.
    The ``default'' $\threshold$ value in each sublegend denotes the dataset's corresponding $\threshold$ value that is evaluated in Sec.~\ref{sec:ExpRes} (Tab.~\ref{tab:ExpRes:DatasetInfo}).
    Each line's shaded region visualizes one standard deviation around the mean certified accuracy.
  }
  \label{fig:App:MoreExps:AltThreshold:Plots}
\end{figure}
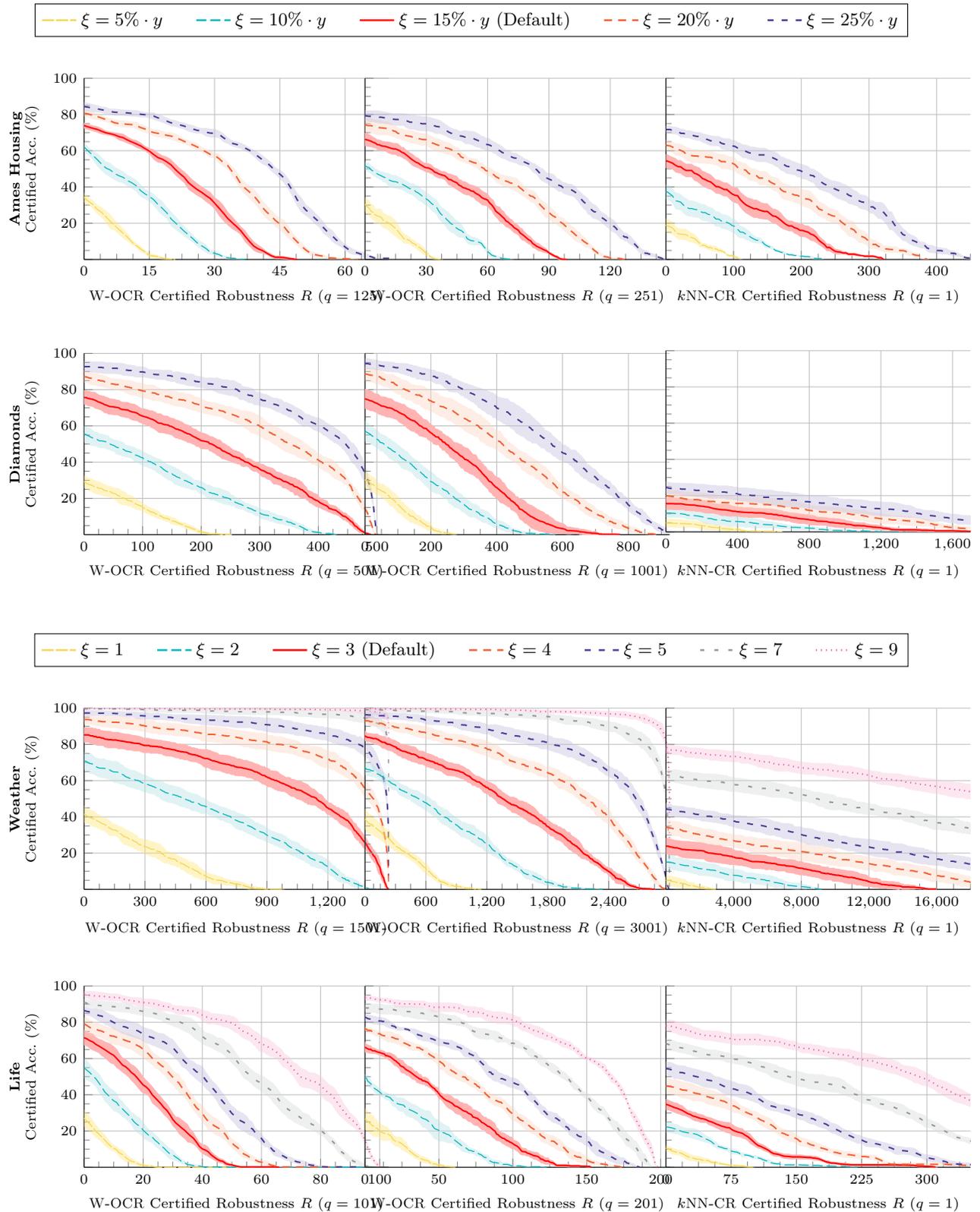

\newpage
\clearpage
\FloatBarrier
\subsection{Model Training Times}%
\label{sec:App:MoreExps:TrainingTime}

This section summarizes the time needed to train our primary certified regressors.
All experiments were performed using a single core of a fourteen-core Intel E5\=/2690v4 CPU with 12GB of 2400MHz DDR4 RAM.
As is standard for \knn{} models, \knnMethod{}'s training time is essentially zero since ``training'' simply entails instance memorization, which is nearly instantaneous.

Table~\ref{tab:App:MoreExps:TrainingTime} details the training times of our two partitioned regressors: \disjointSingleMethod{} which follows the unit-cost assumption (i.e., \eqsmall{${\forall_{\modIdx}\, \covI = 1}$}) and its weighted extension \disjointMultiMethod{}.
The training times are broken down into the time required to train a single submodel as well as the average time to train the whole ensemble.
As expected, \disjointMultiMethod{} takes much longer to train than \disjointSingleMethod{} since the former requires training \bigOnPlT{}~models (for ${\covMax = 2}$) while the latter requires only \bigOT{}~models.

As $\datasetDiv$~increases, each submodel is trained on fewer data.
It might be assumed that submodel training time is inversely related to~$\datasetDiv$.
However, Table~\ref{tab:App:MoreExps:TrainingTime} shows that this is not necessarily the case since different hyperparameter settings can drastically affect how long a model takes to train, even after accounting for~$\datasetDiv$.

Note also that all ensemble submodels can be trained fully in parallel.
Hence, by simply parallelizing submodel training, \disjointSingleMethod{}'s total training time can be sped up by a factor of about~$\nModel$ while \disjointMultiMethod{}'s total training time can be sped up by about a factor of ${(\nTr + \nModel)}$.
Moreover, while training an ensemble (partitioned or overlapping) is often more computationally expensive than training just one model on the whole training set, this additional training time is amortized across all test predictions that need to be certified.

Lastly, recall from Sec.~\ref{sec:ExpRes:Setup}'s evaluation that our overlapping certified ensembles, \overlapSingleMethod{} and \overlapMultiMethod{}, have ${\nModel = \datasetDiv \spreadDegreeSym}$~submodels, where $\spreadDegreeSym$~is the training set block spread degree.
Therefore, the time to train \overlapSingleMethod{} and \overlapMultiMethod{} ensembles is about ${\spreadDegreeSym}$~times larger than that of \disjointSingleMethod{} and \disjointMultiMethod{}, respectively.
See Table~\ref{tab:ExpRes:DatasetInfo} for each dataset's $\spreadDegreeSym$~value.

\vfill

\begin{table}[h!]
  \centering
  \caption{%
    \textbf{(Weighted) Partitioned Certified Regression Model Training Times}:
    Mean and standard deviation \disjointSingleMethod{} and \disjointMultiMethod{} model training times (in seconds) for six datasets in Sec.~\ref{sec:ExpRes}.
    Below each dataset name is the corresponding submodel architecture that was used -- either ridge regression or XGBoost.
    All experiments were performed on a single CPU core.
  }%
  \label{tab:App:MoreExps:TrainingTime}%
  {%
    \tableSize%
\newcommand{\DivRule}{\cdashline{2-6}}
\newcommand{\DsRule}{\midrule}

\renewcommand{\arraystretch}{1.2}
\setlength{\dashlinedash}{0.4pt}
\setlength{\dashlinegap}{1.5pt}
\setlength{\arrayrulewidth}{0.3pt}
\setlength{\tabcolsep}{8.5pt}

\newcommand{\SubName}{\multicolumn{1}{c}{Submodel}}
\newcommand{\TotName}{\multicolumn{1}{c}{Total}}
\newcommand{\xgbTab}{XGBoost}

\newcommand{\DsName}[2]{\multirow{3}{*}{\shortstack{#1 \\ (#2)}}}
\newcommand{\SubTime}[2]{${#1 \pm #2}$}
\newcommand{\TotTime}[2]{${#1}$}

\newcommand{\Pz}{\phantom{.0}}
\newcommand{\Nz}{\phantom{0}}
\newcommand{\Nzz}{\phantom{00}}

{\centering%
  \begin{tabular}{@{}crrrrr@{}}
  \toprule
  \multirow{2}{*}{\shortstack{Dataset \\ (Model Type)}}
      & \multirow{2}{*}{${\datasetDiv}$}
      & \multicolumn{2}{c}{\disjointSingleMethod}  & \multicolumn{2}{c}{\disjointMultiMethod} \\\cmidrule(lr){3-4}\cmidrule(lr){5-6}
      &           & \SubName{}                 & \TotName{}            & \SubName{}                & \TotName{}   \\
  \midrule
  \DsName{Ames}{\xgbTab}
      & 25        & \SubTime{2.5\Nz}{0.50}     & \TotTime{63.1}{}      & \SubTime{213.6}{\Nz28.3}  & \TotTime{5340}{}   \\\DivRule
      & 125       & \SubTime{3.1\Nz}{0.33}     & \TotTime{384\Pz}{}    & \SubTime{27.5}{\Nzz9.8}   & \TotTime{3437}{}   \\\DivRule
      & 251       & \SubTime{1.0\Nz}{0.13}     & \TotTime{261\Pz}{}    & \SubTime{ 3.9}{\Nzz1.7}   & \TotTime{ 984}{}   \\
  \midrule
  \DsName{Austin}{\xgbTab}
      & 51        & \SubTime{2.2\Nz}{0.19}     & \TotTime{111\Pz}{}    & \SubTime{502.3}{109.3}   & \TotTime{25,617}{}   \\\DivRule
      & 301       & \SubTime{0.37}{0.06}       & \TotTime{184\Pz}{}    & \SubTime{25.5}{\Nzz8.9}   & \TotTime{7,688}{}   \\\DivRule
      & 701       & \SubTime{0.89}{0.11}       & \TotTime{626\Pz}{}    & \SubTime{12.5}{\Nzz8.1}   & \TotTime{8,786}{}   \\
  \midrule
  \DsName{Diamonds}{Ridge}
      & 151       & \SubTime{0.01}{0.00}       & \TotTime{0.7}{}       & \SubTime{2.1}{\Nzz0.2}    & \TotTime{322}{}   \\\DivRule
      & 501       & \SubTime{0.01}{0.00}       & \TotTime{1.9}{}       & \SubTime{0.6}{\Nzz0.1}    & \TotTime{293}{}   \\\DivRule
      & 1,001     & \SubTime{0.01}{0.00}       & \TotTime{4.2}{}       & \SubTime{0.3}{\Nzz0.1}    & \TotTime{299}{}   \\
  \midrule
  \DsName{Weather}{Ridge}
      & 51        & \SubTime{0.05}{0.02}       & \TotTime{2.4}{}       & \SubTime{298.7}{109.7}    & \TotTime{15,234}{}   \\\DivRule
      & 1,501     & \SubTime{0.01}{0.00}       & \TotTime{12.6}{}      & \SubTime{2.5}{\Nzz0.2}    & \TotTime{3,776}{}   \\\DivRule
      & 3,001     & \SubTime{0.02}{0.00}       & \TotTime{66.0}{}      & \SubTime{4.7}{\Nz10.6}    & \TotTime{14,051}{}   \\
  \midrule
  \DsName{Life}{\xgbTab}
      & 25        & \SubTime{2.2\Nz}{0.18}     & \TotTime{56.2}{}      & \SubTime{203.8}{\Nz34.9}  & \TotTime{5,095}{}   \\\DivRule
      & 101       & \SubTime{1.8\Nz}{0.31}     & \TotTime{185\Pz}{}    & \SubTime{ 33.5}{\Nz12.6}  & \TotTime{3,388}{}   \\\DivRule
      & 201       & \SubTime{8.1\Nz}{2.6\Nz}   & \TotTime{1,624\Pz}{}  & \SubTime{ 85.5}{\Nz54.6}  & \TotTime{17,184}{}   \\
  \midrule
  \DsName{Spambase}{Ridge}
      & 25        & \SubTime{0.03}{0.01}       & \TotTime{0.8}{}       & \SubTime{4.4}{\Nzz1.2}    & \TotTime{109}{}   \\\DivRule
      & 151       & \SubTime{0.01}{0.00}       & \TotTime{1.7}{}       & \SubTime{0.5}{\Nzz0.3}    & \TotTime{69}{}   \\\DivRule
      & 301       & \SubTime{0.01}{0.00}       & \TotTime{0.9}{}       & \SubTime{0.1}{\Nzz0.0}    & \TotTime{19}{}   \\
  \bottomrule
\end{tabular}
}
  }%
\end{table}

\vfill
 
\newpage
\clearpage
\FloatBarrier
\subsection{Overlapping Regression ILP Execution Time}%
\label{sec:App:MoreExps:GurobiTime}

Recall from Sec.~\ref{sec:CertifiedOverlap} that our two overlapping certified regressors, \overlapSingleMethod{} and \overlapMultiMethod{}, use an integer linear program to bound certified robustness~\eqsmall{$\certBound$}.
Under \overlapSingleMethod{}'s unit-cost structure where \eqsmall{${\forall_{\modIdx}\, \covI \leq 1}$}, Fig.~\ref{fig:CertifiedGeneral:ILP} is actually a \keyword{binary integer program} since \eqsmall{${\forall_{\blockIdx} \, \dsBlockVarI \in \set{0,1}}$}.
In contrast, with \overlapMultiMethod{}'s weighted costs, Fig.~\ref{fig:CertifiedGeneral:ILP} is a true integer linear program.
While both \overlapSingleMethod{} and \overlapMultiMethod{} have the same asymptotic complexity, we observed that LP solvers generally solve programs with only binary variables faster than integral ones.

Sec.~\ref{sec:ExpRes}'s evaluation implemented Fig.~\ref{fig:CertifiedGeneral:ILP}'s LP in Gurobi~\citep{Gurobi}.
All experiments were performed on a single core of a fourteen-core Intel E5\=/2690v4 CPU with 12GB of RAM.
Fig.~\ref{fig:App:MoreExps:GurobiTime} shows the ILP execution time distribution of \overlapSingleMethod{} and \overlapMultiMethod{} for three datasets from Sec.~\ref{sec:ExpRes}.
For each dataset, we report the ILP's execution time for the two largest $\datasetDiv$~values.
Each histogram visualizes at least 1600~trials per method.

In summary, ILP execution time is bimodal with run times clustered around 0~seconds and 1200~seconds (i.e.,~the ILP's time limit).
\overlapSingleMethod{}'s ILP is generally faster on average than \overlapMultiMethod{} (observe that \overlapSingleMethod{} times out less frequently, if at all).

Note also that the ILP's execution time distribution does not change significantly across each dataset's two $\datasetDiv$~values.

\vfill

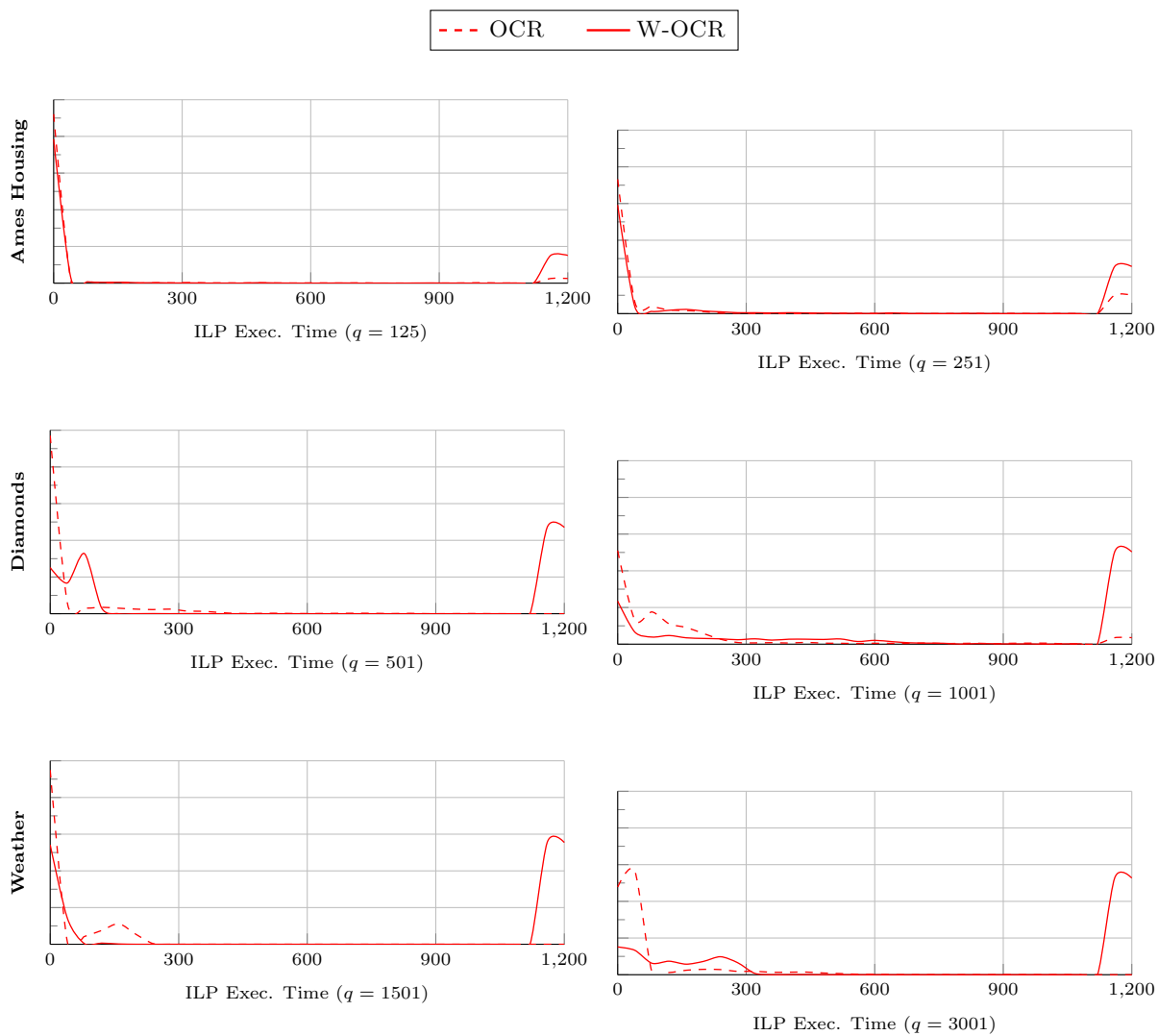
\begin{figure}[h!]
  \centering

\newcommand{\gurobiTimeFontSize}{\scriptsize}%
\newcommand{\gurobiHistPlot}[5]{%
  \pgfplotstableread[col sep=comma] {plots/data/#1}\thedata%
  \begin{tikzpicture}
    \begin{axis}[
        scale only axis,%
        height={\gurobiTimePlotHeight},%
        width={\gurobiTimePlotWidth},%
        xmin=0,%
        xmax={\gurobiTimeXMax},%
        xtick distance={\gurobiTimeXTickDist},
        x tick label style={font=\gurobiTimeFontSize,align=center},%
        xlabel={#2},%
        xlabel style={font=\gurobiTimeFontSize,align=center},%
        xmajorgrids,%
        axis x line*=bottom,  %
        ymin=0,%
        ymax={#3},%
        ytick distance={#4},%
        yticklabels={,,,,,,,},%
        minor y tick num=1,%
        y tick label style={font=\gurobiTimeFontSize,align=center},%
        scaled y ticks=false,
        ylabel={#5},%
        ylabel style={font=\gurobiTimeFontSize,align=center},%
        ymajorgrids,%
        axis y line*=left,  %
        mark size=0pt,%
      ]
      \addplot +[
        line width=\gurobiHistLineWidth,
        smooth,
        single overlap line,
        hist={
          density,
          bins={\gurobiTimeNBins},
          data min={0},
          data max={\gurobiTimeXMax},
          handler/.style={
            sharp plot,
            smooth,
          }
        }
      ] table [y index=0] \thedata;
      \addplot +[
        line width=\gurobiHistLineWidth,
        multi overlap line,
        hist={
          density,
          bins={\gurobiTimeNBins},
          data min={0},
          data max={\gurobiTimeXMax},
          handler/.style={
            smooth,
          }%
        }%
      ] table [y index=1] \thedata;
    \end{axis}
  \end{tikzpicture}
}%
\newcommand{\gurobiSubfigWidth}{0.48\textwidth}%
\newcommand{\gurobiTimePlotHeight}{1.00in}%
\newcommand{\gurobiTimePlotWidth}{2.80in}%
\newcommand{\gurobiTimeHorizontalSpacer}{\hfill}%
\newcommand{\gurobiTimeDatasetSpacer}{\vspace{15pt}}%
\newcommand{\gurobiTimeXMax}{1200}%
\newcommand{\gurobiTimeXTickDist}{300}%
\newcommand{\gurobiTimeNBins}{30}%
\newcommand{\gurobiTimeXLabel}[1]{ILP Exec. Time (${\datasetDiv = #1}$)}%
\newcommand{\gurobiTimeYLabel}[1]{\textbf{#1}}%
\newcommand{\gurobiHistLineWidth}{1.0pt}%
\newcommand{\gurobiTimePlot}[6]{%
  \begin{subfigure}[b]{\gurobiSubfigWidth}%
    \centering%
    \gurobiHistPlot{#1/gurobi/#1_gurobi_#3.csv}  %
                   {\gurobiTimeXLabel{\num{#3}}} %
                   {#4}                          %
                   {#5}                          %
                   {#6}                          %
  \end{subfigure}%
}%

\newcommand{\boundLegendFontSize}{\small}
\begin{tikzpicture}
  \begin{axis}[%
      hide axis,  %
      no marks,
      xmin=0,  %
      xmax=1,
      ymin=0,
      ymax=1,
      scale only axis,width=1mm, %
      legend cell align={left},              %
      legend style={font=\boundLegendFontSize},
      legend columns=2,
      legend style={/tikz/every even column/.append style={column sep=0.5cm}}
    ]

    \addplot [
      single overlap line,
      line width=0.9pt,
    ] coordinates {(0,0)};
    \addlegendentry{\overlapSingleMethod{}}%

    \addplot [
      multi overlap line,
      line width=0.9pt,
    ] coordinates {(0,0)};
    \addlegendentry{\overlapMultiMethod{}}%
  \end{axis}
\end{tikzpicture}

\gurobiTimeDatasetSpacer{}%
\gurobiTimeHorizontalSpacer{}%
\gurobiTimePlot{ames}                   %
               {Ames Housing}           %
               {00125}                  %
               {0.02500}                %
               {0.00500}                %
               {\gurobiTimeYLabel{Ames Housing}}  %
\gurobiTimeHorizontalSpacer{}%
\gurobiTimePlot{ames}                   %
               {Ames Housing}           %
               {00251}                  %
               {0.02500}                %
               {0.00500}                %
               {}                       %
\gurobiTimeHorizontalSpacer{}%

\gurobiTimeDatasetSpacer{}%
\gurobiTimeHorizontalSpacer{}%
\gurobiTimePlot{diamonds}               %
               {Diamonds}               %
               {00501}                  %
               {0.02000}                %
               {0.00400}                %
               {\gurobiTimeYLabel{Diamonds}}  %
\gurobiTimeHorizontalSpacer{}%
\gurobiTimePlot{diamonds}               %
               {Diamonds}               %
               {01001}                  %
               {0.02000}                %
               {0.00400}                %
               {}                       %
\gurobiTimeHorizontalSpacer{}%

\gurobiTimeDatasetSpacer{}%
\gurobiTimeHorizontalSpacer{}%
\gurobiTimePlot{shifts-weather-medium}  %
               {Weather}                %
               {01501}                  %
               {0.02000}                %
               {0.00400}                %
               {\gurobiTimeYLabel{Weather}}  %
\gurobiTimeHorizontalSpacer{}%
\gurobiTimePlot{shifts-weather-medium}  %
               {Weather}                %
               {03001}                  %
               {0.02000}                %
               {0.00400}                %
               {}                       %
\gurobiTimeHorizontalSpacer{}%

  \caption{%
    \textbf{Gurobi ILP Execution Time}:
    Histogram of the Gurobi ILP execution time (in seconds) for \overlapSingleMethod{} and \overlapMultiMethod{} for three datasets from Sec.~\ref{sec:ExpRes}.
    ILP execution times for each dataset's two largest $\datasetDiv$ values are plotted, with each of dataset plots sharing the y\=/axis scales.
    The ILP was implemented in Gurobi with a fixed time limit of 1200 seconds.
  }%
  \label{fig:App:MoreExps:GurobiTime}
\end{figure}

\vfill
 
\newpage
\clearpage
\FloatBarrier
\subsection{\revised{Effect of Median as the Model Decision Function}}%
\label{sec:App:MoreExps:MedianDecisionFunc}

\revised{%
This section evaluates how the choice of decision function affects the accuracy of standard (i.e.,~non-robust) regression.
All experiments above exclusively used median as the decision function.
Below, we compare median's performance to the more traditional decision function, mean.
To be clear, this section's experiments exclusively consider non-robust model predictions.
This mini ablation study simply examines the decision function's effect on baseline (uncertified) prediction accuracy.

Table~\ref{tab:App:MoreExps:MedianDecisionFunc} compares the performance of mean and median on Sec.~\ref{sec:ExpRes}'s five regression datasets.
These experiments used the same threshold~($\threshold$) values as Sec.~\ref{sec:ExpRes} (see Tab.~\ref{tab:ExpRes:DatasetInfo}).
Tab.~\ref{tab:App:MoreExps:MedianDecisionFunc} includes results for two different model architectures: \knn{}-based regression and a regression ensemble that follows $\disjointSingleMethod$'s base architecture.
For each dataset, the ensemble regressor is evaluated on the same three $\datasetDiv$~values used in Sec.~\ref{sec:ExpRes}.

In summary, median and mean decision functions have comparable performance.
Median had better average accuracy than mean on 11 of 20 evaluation setups.
Median and mean had equivalent average performance on four setups.
Mean outperformed median on the remaining five setups.

\knn{} regression was the method most affected by the choice of decision function.
In particular for the Weather and Life datasets, mean outperformed median by 8.6 and 2.3 percentage points, respectively.
Only one other evaluation setup (Austin with ${\datasetDiv = 701}$) saw a similarly large performance difference (3.8pp).
}%

\vfill
\begin{table}[h]
  \centering
  \caption{%
    \revised{%
      \textbf{Effect of Median vs.\ Mean as the Decision Function}:
      Comparison of the model accuracy mean and standard deviation for two different decision functions.
      For each of Sec.~\ref{sec:ExpRes}'s five regression datasets, we evaluate the decision function's effect on both \knn{} and ensemble (\disjointSingleMethod{}) learners.
      Each dataset's \disjointSingleMethod{} non-robust accuracy is reported for three different ${\datasetDiv}$~values in line with Sec.~\ref{sec:ExpRes}'s evaluation.
      For each experimental setup, the best performing method (in terms of average accuracy) is shown in bold.
      In summary, median and mean decision functions have comparable baseline (uncertified) prediction accuracy.
      However, median is critical to achieve certified robustness guarantees.%
    }%
  }\label{tab:App:MoreExps:MedianDecisionFunc}
  \revised{%

\newcommand{\MedBaseHeader}[2][2]{\multirow{#1}{*}{#2}}

\newcommand{\DatasetHeader}{\MedBaseHeader{Dataset}}
\newcommand{\ModelHeader}{\MedBaseHeader{Model}}
\newcommand{\QHeader}{\MedBaseHeader{$q$}}

\newcommand{\Dataset}[1]{\multirow{4}{*}{#1}}
\newcommand{\Ensemble}{\multirow{3}{*}{\disjointSingleMethod}}

\newcommand{\PVal}[2]{${#1 \pm #2}$}
\newcommand{\PValB}[2]{\textBF{#1} $\pm$ \textBF{#2}}

\renewcommand{\arraystretch}{1.2}
\setlength{\dashlinedash}{0.4pt}
\setlength{\dashlinegap}{1.5pt}
\setlength{\arrayrulewidth}{0.3pt}
\setlength{\tabcolsep}{8.5pt}

\newcommand{\DatasetSep}{\midrule}
\newcommand{\QSep}{\cdashline{3-5}}
\newcommand{\ModelSep}{\cmidrule{2-5}}

\begin{tabular}{llrrr}
  \toprule
  \DatasetHeader{}
    & \ModelHeader{}
    & \QHeader{}
    & \multicolumn{2}{c}{Decision Function}
  \\\cmidrule(lr){4-5}
    &
    &
    & \multicolumn{1}{c}{Median}
    & \multicolumn{1}{c}{Mean}
  \\\midrule
\Dataset{Ames}       & \knn{}     & 1     & \PValB{53.4}{5.4}  & \PValB{53.4}{5.4}   \\\ModelSep
                     & \Ensemble  & 25    & \PVal{83.8}{2.7}   & \PValB{84.6}{2.6}   \\\QSep
                     &            & 125   & \PValB{71.2}{4.0}  & \PVal{71.0}{3.9}   \\\QSep
                     &            & 251   & \PValB{63.4}{3.5}  & \PVal{62.7}{2.2}   \\\DatasetSep
\Dataset{Austin}     & \knn{}     & 1     & \PValB{32.9}{3.8}  & \PValB{32.9}{3.8}   \\\ModelSep
                     & \Ensemble  & 51    & \PVal{61.4}{4.9}   & \PValB{61.8}{4.5}   \\\QSep
                     &            & 301   & \PVal{51.6}{4.0}   & \PValB{51.8}{3.9}   \\\QSep
                     &            & 701   & \PValB{44.6}{4.8}  & \PVal{40.8}{5.5}   \\\DatasetSep
\Dataset{Diamonds}   & \knn{}     & 1     & \PValB{16.4}{2.1}  & \PValB{16.4}{2.1}   \\\ModelSep
                     & \Ensemble  & 151   & \PValB{74.9}{3.8}  & \PVal{73.6}{4.4}   \\\QSep
                     &            & 501   & \PValB{77.7}{4.4}  & \PVal{76.6}{5.4}   \\\QSep
                     &            & 1001  & \PValB{75.2}{4.1}  & \PVal{73.2}{4.4}   \\\DatasetSep
\Dataset{Weather}    & \knn{}     & 1     & \PVal{24.6}{4.5}   & \PValB{33.2}{5.4}   \\\ModelSep
                     & \Ensemble  & 51    & \PValB{86.3}{3.2}  & \PVal{86.0}{3.2}   \\\QSep
                     &            & 1501  & \PValB{85.2}{3.9}  & \PVal{85.0}{3.8}   \\\QSep
                     &            & 3001  & \PValB{86.7}{2.7}  & \PValB{86.7}{2.7}   \\\DatasetSep
\Dataset{Life}       & \knn{}     & 1     & \PVal{35.7}{3.2}   & \PValB{38.0}{3.5}   \\\ModelSep
                     & \Ensemble  & 25    & \PValB{80.1}{4.0}  & \PVal{79.9}{3.7}   \\\QSep
                     &            & 101   & \PValB{72.0}{3.2}  & \PVal{71.5}{3.7}   \\\QSep
                     &            & 201   & \PValB{63.3}{4.1}  & \PVal{61.8}{3.3}   \\\bottomrule
\end{tabular}
  }%
\end{table}
\vfill
   \stopcontents  %

\end{document}

%% file: tables/data/ames-housing_full_results.tex
\DivOne{}        & \BoundDist{0}   & \BoundVal{90.4}{2.4}  & \BoundVal{90.4}{2.4}   & \BoundVal{90.4}{2.4}  & \BoundVal{90.4}{2.4}   & \BoundVal{54.3}{3.8}\\\DatasetDivSep
\DivHeader{25}   & \BoundDist{1}   & \BoundVal{82.3}{3.1}  & \BoundValB{86.8}{2.8}  & \BoundVal{82.3}{3.1}  & \BoundVal{85.2}{2.9}   & \BoundVal{54.3}{3.8}\\\BoundSep
                 & \BoundDist{4}   & \BoundVal{76.3}{3.7}  & \BoundValB{81.2}{2.9}  & \BoundVal{76.3}{3.7}  & \BoundVal{80.0}{2.7}   & \BoundVal{54.0}{3.6}\\\BoundSep
                 & \BoundDist{8}   & \BoundVal{57.4}{3.8}  & \BoundVal{69.6}{2.7}   & \BoundVal{57.5}{3.7}  & \BoundValB{70.1}{3.2}  & \BoundVal{53.1}{3.5}\\\BoundSep
                 & \BoundDist{12}  & \BoundVal{11.0}{3.7}  & \BoundVal{28.2}{3.8}   & \BoundVal{23.7}{5.1}  & \BoundVal{48.8}{4.4}   & \BoundValB{52.2}{3.9}\\\BoundSep
                 & \BoundDist{16}  & \BoundVal{0.0}{0.0}   & \BoundVal{0.0}{0.0}    & \BoundVal{11.3}{3.8}  & \BoundVal{15.3}{3.2}   & \BoundValB{51.5}{3.8}\\\DatasetDivSep
\DivHeader{125}  & \BoundDist{1}   & \BoundVal{70.4}{3.7}  & \BoundValB{73.9}{1.7}  & \BoundVal{70.4}{3.7}  & \BoundVal{73.1}{1.6}   & \BoundVal{54.3}{3.8}\\\BoundSep
                 & \BoundDist{10}  & \BoundVal{62.8}{3.4}  & \BoundValB{66.5}{2.0}  & \BoundVal{62.8}{3.4}  & \BoundVal{65.9}{1.9}   & \BoundVal{53.1}{3.5}\\\BoundSep
                 & \BoundDist{20}  & \BoundVal{44.9}{4.2}  & \BoundValB{52.2}{2.9}  & \BoundVal{44.9}{4.2}  & \BoundValB{52.2}{3.0}  & \BoundVal{51.2}{3.6}\\\BoundSep
                 & \BoundDist{30}  & \BoundVal{21.8}{4.2}  & \BoundVal{28.7}{3.8}   & \BoundVal{21.8}{4.2}  & \BoundVal{31.1}{3.3}   & \BoundValB{49.1}{3.7}\\\BoundSep
                 & \BoundDist{40}  & \BoundVal{2.5}{1.3}   & \BoundVal{3.9}{1.7}    & \BoundVal{2.5}{1.3}   & \BoundVal{5.9}{2.3}    & \BoundValB{48.5}{4.0}\\\DatasetDivSep
\DivHeader{251}  & \BoundDist{1}   & \BoundVal{63.1}{3.4}  & \BoundValB{66.3}{3.8}  & \BoundVal{63.1}{3.4}  & \BoundVal{66.0}{3.7}   & \BoundVal{54.3}{3.8}\\\BoundSep
                 & \BoundDist{20}  & \BoundVal{51.8}{3.2}  & \BoundVal{56.4}{2.9}   & \BoundVal{51.8}{3.2}  & \BoundValB{57.9}{2.9}  & \BoundVal{51.2}{3.6}\\\BoundSep
                 & \BoundDist{40}  & \BoundVal{37.1}{3.2}  & \BoundVal{42.5}{3.8}   & \BoundVal{37.1}{3.2}  & \BoundVal{45.3}{2.8}   & \BoundValB{48.5}{4.0}\\\BoundSep
                 & \BoundDist{60}  & \BoundVal{15.3}{3.8}  & \BoundVal{22.5}{3.4}   & \BoundVal{15.3}{3.8}  & \BoundVal{32.8}{3.7}   & \BoundValB{44.1}{4.3}\\\BoundSep
                 & \BoundDist{80}  & \BoundVal{0.2}{0.4}   & \BoundVal{0.6}{0.5}    & \BoundVal{0.2}{0.4}   & \BoundVal{10.9}{2.5}   & \BoundValB{40.1}{3.9}\\

%% file: tables/data/austin-housing_full_results.tex
\DivOne{}        & \BoundDist{0}    & \BoundVal{71.3}{4.1}   & \BoundVal{71.3}{4.1}   & \BoundVal{71.3}{4.1}   & \BoundVal{71.3}{4.1}   & \BoundVal{35.3}{4.8}\\\DatasetDivSep
\DivHeader{51}   & \BoundDist{1}    & \BoundVal{59.9}{4.5}   & \BoundValB{63.7}{4.6}  & \BoundVal{59.9}{4.5}   & \BoundVal{61.7}{4.6}   & \BoundVal{35.3}{4.8}\\\BoundSep
                 & \BoundDist{5}    & \BoundVal{49.6}{5.1}   & \BoundValB{52.9}{3.4}  & \BoundVal{49.6}{5.1}   & \BoundVal{50.8}{4.2}   & \BoundVal{35.2}{4.8}\\\BoundSep
                 & \BoundDist{10}   & \BoundVal{29.9}{3.7}   & \BoundValB{35.3}{2.8}  & \BoundVal{29.9}{3.7}   & \BoundVal{31.8}{3.2}   & \BoundVal{34.9}{4.9}\\\BoundSep
                 & \BoundDist{15}   & \BoundVal{9.2}{2.0}    & \BoundVal{12.6}{2.8}   & \BoundVal{9.2}{2.0}    & \BoundVal{10.1}{2.7}   & \BoundValB{34.7}{4.9}\\\BoundSep
                 & \BoundDist{20}   & \BoundVal{0.5}{0.5}    & \BoundVal{0.3}{0.7}    & \BoundVal{0.5}{0.5}    & \BoundVal{0.0}{0.0}    & \BoundValB{34.6}{4.6}\\\DatasetDivSep
\DivHeader{301}  & \BoundDist{1}    & \BoundValB{51.0}{3.9}  & \BoundVal{52.0}{4.1}   & \BoundValB{51.0}{3.9}  & \BoundVal{51.1}{4.1}   & \BoundVal{35.3}{4.8}\\\BoundSep
                 & \BoundDist{20}   & \BoundVal{41.4}{3.6}   & \BoundValB{43.3}{5.9}  & \BoundVal{41.4}{3.6}   & \BoundValB{43.3}{5.9}  & \BoundVal{34.6}{4.6}\\\BoundSep
                 & \BoundDist{40}   & \BoundVal{29.7}{4.1}   & \BoundVal{32.2}{5.2}   & \BoundVal{29.7}{4.1}   & \BoundVal{33.7}{5.7}   & \BoundValB{34.3}{4.7}\\\BoundSep
                 & \BoundDist{60}   & \BoundVal{15.4}{3.0}   & \BoundVal{19.1}{4.1}   & \BoundVal{15.4}{3.0}   & \BoundVal{22.7}{4.9}   & \BoundValB{34.0}{4.5}\\\BoundSep
                 & \BoundDist{80}   & \BoundVal{3.2}{1.8}    & \BoundVal{3.1}{1.2}    & \BoundVal{3.2}{1.8}    & \BoundVal{7.7}{3.4}    & \BoundValB{32.9}{4.5}\\\DatasetDivSep
\DivHeader{701}  & \BoundDist{1}    & \BoundValB{43.9}{5.0}  & \BoundVal{42.7}{5.5}   & \BoundValB{43.9}{5.0}  & \BoundVal{43.6}{5.7}   & \BoundVal{35.3}{4.8}\\\BoundSep
                 & \BoundDist{40}   & \BoundVal{34.5}{6.0}   & \BoundVal{35.0}{6.2}   & \BoundVal{34.5}{6.0}   & \BoundValB{36.9}{5.9}  & \BoundVal{34.3}{4.7}\\\BoundSep
                 & \BoundDist{80}   & \BoundVal{25.3}{4.8}   & \BoundVal{24.7}{6.1}   & \BoundVal{25.3}{4.8}   & \BoundVal{27.0}{6.1}   & \BoundValB{32.9}{4.5}\\\BoundSep
                 & \BoundDist{120}  & \BoundVal{13.1}{2.6}   & \BoundVal{14.6}{5.0}   & \BoundVal{13.1}{2.6}   & \BoundVal{18.9}{4.4}   & \BoundValB{31.6}{4.9}\\\BoundSep
                 & \BoundDist{160}  & \BoundVal{2.7}{0.9}    & \BoundVal{4.8}{3.3}    & \BoundVal{2.7}{0.9}    & \BoundVal{9.1}{2.9}    & \BoundValB{30.0}{4.7}\\

%% file: tables/data/diamonds_full_results.tex
\DivOne{}         & \BoundDist{0}    & \BoundVal{73.6}{4.0}   & \BoundVal{73.6}{4.0}   & \BoundVal{73.6}{4.0}   & \BoundVal{73.6}{4.0}   & \BoundVal{15.7}{3.5}\\\DatasetDivSep
\DivHeader{151}   & \BoundDist{1}    & \BoundVal{74.6}{3.8}   & \BoundValB{74.8}{4.5}  & \BoundVal{74.6}{3.8}   & \BoundVal{74.7}{4.4}   & \BoundVal{15.7}{3.5}\\\BoundSep
                  & \BoundDist{35}   & \BoundVal{64.4}{4.6}   & \BoundVal{67.1}{4.8}   & \BoundVal{67.2}{4.6}   & \BoundValB{69.7}{4.1}  & \BoundVal{15.6}{3.4}\\\BoundSep
                  & \BoundDist{70}   & \BoundVal{38.6}{5.7}   & \BoundVal{42.2}{4.0}   & \BoundVal{62.4}{4.3}   & \BoundValB{64.7}{5.2}  & \BoundVal{15.5}{3.4}\\\BoundSep
                  & \BoundDist{105}  & \BoundVal{0.0}{0.0}    & \BoundVal{0.0}{0.0}    & \BoundVal{54.5}{5.9}   & \BoundValB{57.4}{5.2}  & \BoundVal{15.2}{3.3}\\\BoundSep
                  & \BoundDist{140}  & \BoundVal{0.0}{0.0}    & \BoundVal{0.0}{0.0}    & \BoundValB{35.4}{5.8}  & \BoundVal{34.3}{7.1}   & \BoundVal{14.8}{3.4}\\\DatasetDivSep
\DivHeader{501}   & \BoundDist{1}    & \BoundValB{77.3}{4.2}  & \BoundVal{75.8}{3.9}   & \BoundValB{77.3}{4.2}  & \BoundVal{75.7}{4.1}   & \BoundVal{15.7}{3.5}\\\BoundSep
                  & \BoundDist{75}   & \BoundVal{66.2}{4.0}   & \BoundVal{65.0}{4.4}   & \BoundVal{66.3}{4.0}   & \BoundValB{68.7}{4.4}  & \BoundVal{15.5}{3.4}\\\BoundSep
                  & \BoundDist{150}  & \BoundVal{50.7}{4.9}   & \BoundVal{48.2}{4.5}   & \BoundVal{57.8}{4.7}   & \BoundValB{59.6}{4.9}  & \BoundVal{14.8}{3.4}\\\BoundSep
                  & \BoundDist{300}  & \BoundVal{0.0}{0.0}    & \BoundVal{0.0}{0.0}    & \BoundValB{38.0}{6.1}  & \BoundVal{36.2}{3.2}   & \BoundVal{12.3}{3.4}\\\BoundSep
                  & \BoundDist{450}  & \BoundVal{0.0}{0.0}    & \BoundVal{0.0}{0.0}    & \BoundVal{8.8}{3.3}    & \BoundVal{9.0}{2.1}    & \BoundValB{10.7}{2.9}\\\DatasetDivSep
\DivHeader{1001}  & \BoundDist{1}    & \BoundValB{75.2}{4.1}  & \BoundVal{74.9}{5.5}   & \BoundValB{75.2}{4.1}  & \BoundVal{74.9}{5.5}   & \BoundVal{15.7}{3.5}\\\BoundSep
                  & \BoundDist{150}  & \BoundVal{56.0}{4.9}   & \BoundVal{56.3}{5.8}   & \BoundVal{56.0}{4.9}   & \BoundValB{62.8}{6.1}  & \BoundVal{14.8}{3.4}\\\BoundSep
                  & \BoundDist{300}  & \BoundVal{24.7}{4.4}   & \BoundVal{25.3}{4.8}   & \BoundVal{29.5}{4.1}   & \BoundValB{42.3}{6.4}  & \BoundVal{12.3}{3.4}\\\BoundSep
                  & \BoundDist{450}  & \BoundVal{0.0}{0.0}    & \BoundVal{0.0}{0.0}    & \BoundVal{16.9}{4.0}   & \BoundValB{17.9}{5.1}  & \BoundVal{10.7}{2.9}\\\BoundSep
                  & \BoundDist{600}  & \BoundVal{0.0}{0.0}    & \BoundVal{0.0}{0.0}    & \BoundVal{4.2}{1.5}    & \BoundVal{3.3}{3.1}    & \BoundValB{9.6}{3.1}\\

%% file: tables/data/shifts-weather-medium_full_results.tex
\DivOne{}         & \BoundDist{0}     & \BoundVal{85.9}{3.4}   & \BoundVal{85.9}{3.4}   & \BoundVal{85.9}{3.4}   & \BoundVal{85.9}{3.4}   & \BoundVal{23.8}{4.5}\\\DatasetDivSep
\DivHeader{51}    & \BoundDist{1}     & \BoundVal{86.0}{3.1}   & \BoundVal{86.5}{3.8}   & \BoundVal{86.0}{3.1}   & \BoundValB{86.5}{3.8}  & \BoundVal{23.8}{4.5}\\\BoundSep
                  & \BoundDist{10}    & \BoundVal{83.9}{3.5}   & \BoundVal{84.6}{3.8}   & \BoundVal{83.9}{3.5}   & \BoundValB{85.0}{3.8}  & \BoundVal{23.8}{4.5}\\\BoundSep
                  & \BoundDist{20}    & \BoundVal{82.0}{3.5}   & \BoundVal{82.3}{4.2}   & \BoundVal{83.1}{3.4}   & \BoundValB{84.0}{3.8}  & \BoundVal{23.8}{4.5}\\\BoundSep
                  & \BoundDist{35}    & \BoundVal{0.0}{0.0}    & \BoundVal{0.0}{0.0}    & \BoundVal{81.8}{3.7}   & \BoundValB{82.0}{4.4}  & \BoundVal{23.8}{4.5}\\\BoundSep
                  & \BoundDist{50}    & \BoundVal{0.0}{0.0}    & \BoundVal{0.0}{0.0}    & \BoundValB{76.8}{5.1}  & \BoundVal{75.8}{4.9}   & \BoundVal{23.8}{4.5}\\\DatasetDivSep
\DivHeader{1501}  & \BoundDist{1}     & \BoundValB{85.2}{3.9}  & \BoundValB{85.2}{4.2}  & \BoundValB{85.2}{3.9}  & \BoundValB{85.2}{4.2}  & \BoundVal{23.8}{4.5}\\\BoundSep
                  & \BoundDist{300}   & \BoundVal{75.8}{4.3}   & \BoundVal{77.6}{4.2}   & \BoundVal{76.8}{4.2}   & \BoundValB{79.4}{4.2}  & \BoundVal{23.4}{4.6}\\\BoundSep
                  & \BoundDist{600}   & \BoundVal{54.3}{4.7}   & \BoundVal{55.1}{5.4}   & \BoundVal{71.4}{3.7}   & \BoundValB{72.2}{5.1}  & \BoundVal{22.9}{4.3}\\\BoundSep
                  & \BoundDist{1000}  & \BoundVal{0.0}{0.0}    & \BoundVal{0.0}{0.0}    & \BoundVal{56.5}{5.1}   & \BoundValB{57.4}{4.6}  & \BoundVal{22.0}{4.6}\\\BoundSep
                  & \BoundDist{1400}  & \BoundVal{0.0}{0.0}    & \BoundVal{0.0}{0.0}    & \BoundValB{22.9}{2.6}  & \BoundVal{22.5}{3.2}   & \BoundVal{21.8}{4.8}\\\DatasetDivSep
\DivHeader{3001}  & \BoundDist{1}     & \BoundValB{86.7}{2.7}  & \BoundVal{84.6}{2.9}   & \BoundValB{86.7}{2.7}  & \BoundVal{84.6}{2.9}   & \BoundVal{23.8}{4.5}\\\BoundSep
                  & \BoundDist{600}   & \BoundVal{67.7}{2.7}   & \BoundVal{66.9}{4.0}   & \BoundVal{68.1}{2.9}   & \BoundValB{71.5}{4.0}  & \BoundVal{22.9}{4.3}\\\BoundSep
                  & \BoundDist{1200}  & \BoundVal{25.7}{5.8}   & \BoundVal{25.8}{4.9}   & \BoundVal{55.0}{4.2}   & \BoundValB{56.2}{3.8}  & \BoundVal{21.9}{4.7}\\\BoundSep
                  & \BoundDist{1800}  & \BoundVal{0.0}{0.0}    & \BoundVal{0.0}{0.0}    & \BoundValB{35.8}{4.8}  & \BoundVal{34.7}{4.0}   & \BoundVal{21.5}{4.7}\\\BoundSep
                  & \BoundDist{2400}  & \BoundVal{0.0}{0.0}    & \BoundVal{0.0}{0.0}    & \BoundVal{9.3}{3.0}    & \BoundVal{9.9}{2.5}    & \BoundValB{20.5}{4.9}\\

%% file: tables/data/life_full_results.tex
\DivOne{}        & \BoundDist{0}    & \BoundVal{92.7}{3.1}  & \BoundVal{92.7}{3.1}   & \BoundVal{92.7}{3.1}  & \BoundVal{92.7}{3.1}   & \BoundVal{34.6}{3.1}\\\DatasetDivSep
\DivHeader{25}   & \BoundDist{1}    & \BoundVal{77.7}{4.4}  & \BoundValB{80.2}{4.0}  & \BoundVal{77.7}{4.4}  & \BoundVal{78.3}{5.2}   & \BoundVal{34.6}{3.1}\\\BoundSep
                 & \BoundDist{5}    & \BoundVal{69.3}{4.9}  & \BoundValB{71.5}{4.7}  & \BoundVal{69.3}{4.9}  & \BoundVal{71.4}{4.7}   & \BoundVal{33.8}{2.8}\\\BoundSep
                 & \BoundDist{10}   & \BoundVal{43.3}{5.6}  & \BoundVal{54.2}{6.0}   & \BoundVal{47.3}{5.9}  & \BoundValB{60.8}{4.9}  & \BoundVal{32.9}{2.8}\\\BoundSep
                 & \BoundDist{15}   & \BoundVal{0.0}{0.0}   & \BoundVal{0.0}{0.0}    & \BoundVal{23.8}{4.0}  & \BoundValB{33.1}{3.2}  & \BoundVal{32.1}{2.9}\\\BoundSep
                 & \BoundDist{20}   & \BoundVal{0.0}{0.0}   & \BoundVal{0.0}{0.0}    & \BoundVal{9.5}{2.9}   & \BoundVal{11.4}{3.2}   & \BoundValB{31.1}{2.3}\\\DatasetDivSep
\DivHeader{101}  & \BoundDist{1}    & \BoundVal{71.1}{3.8}  & \BoundValB{71.5}{4.3}  & \BoundVal{71.1}{3.8}  & \BoundVal{70.8}{4.1}   & \BoundVal{34.6}{3.1}\\\BoundSep
                 & \BoundDist{10}   & \BoundVal{58.9}{4.3}  & \BoundValB{61.8}{5.1}  & \BoundVal{58.9}{4.3}  & \BoundValB{61.8}{5.1}  & \BoundVal{32.9}{2.8}\\\BoundSep
                 & \BoundDist{20}   & \BoundVal{40.5}{5.8}  & \BoundVal{43.9}{4.3}   & \BoundVal{40.5}{5.8}  & \BoundValB{45.4}{4.7}  & \BoundVal{31.1}{2.3}\\\BoundSep
                 & \BoundDist{30}   & \BoundVal{20.7}{3.8}  & \BoundVal{22.8}{4.4}   & \BoundVal{20.7}{3.8}  & \BoundVal{26.4}{3.9}   & \BoundValB{28.5}{2.5}\\\BoundSep
                 & \BoundDist{40}   & \BoundVal{4.4}{2.4}   & \BoundVal{3.5}{1.4}    & \BoundVal{4.6}{2.5}   & \BoundVal{10.1}{2.7}   & \BoundValB{26.9}{2.4}\\\DatasetDivSep
\DivHeader{201}  & \BoundDist{1}    & \BoundVal{62.9}{4.1}  & \BoundValB{66.3}{3.0}  & \BoundVal{62.9}{4.1}  & \BoundVal{65.7}{2.4}   & \BoundVal{34.6}{3.1}\\\BoundSep
                 & \BoundDist{30}   & \BoundVal{46.6}{3.8}  & \BoundVal{49.0}{2.8}   & \BoundVal{46.6}{3.8}  & \BoundValB{52.9}{3.0}  & \BoundVal{28.5}{2.5}\\\BoundSep
                 & \BoundDist{60}   & \BoundVal{23.3}{2.7}  & \BoundVal{24.6}{4.1}   & \BoundVal{24.4}{2.6}  & \BoundValB{34.4}{3.9}  & \BoundVal{23.4}{2.3}\\\BoundSep
                 & \BoundDist{90}   & \BoundVal{0.1}{0.3}   & \BoundVal{0.6}{0.5}    & \BoundVal{12.8}{2.9}  & \BoundVal{18.0}{3.6}   & \BoundValB{18.1}{2.1}\\\BoundSep
                 & \BoundDist{120}  & \BoundVal{0.0}{0.0}   & \BoundVal{0.0}{0.0}    & \BoundVal{4.1}{1.6}   & \BoundVal{4.4}{2.2}    & \BoundValB{8.5}{1.4}\\

%% file: tables/data/spambase_full_results.tex
\DivOne{}        & \BoundDist{0}    & \BoundVal{87.5}{2.9}   & \BoundVal{87.5}{2.9}   & \BoundVal{87.5}{2.9}   & \BoundVal{87.5}{2.9}   & \BoundVal{64.0}{4.3}\\\DatasetDivSep
\DivHeader{25}   & \BoundDist{1}    & \BoundValB{87.6}{3.5}  & \BoundVal{87.1}{3.8}   & \BoundValB{87.6}{3.5}  & \BoundVal{85.8}{3.6}   & \BoundVal{64.0}{4.3}\\\BoundSep
                 & \BoundDist{5}    & \BoundVal{80.3}{3.8}   & \BoundVal{81.0}{3.4}   & \BoundVal{81.1}{3.6}   & \BoundValB{83.5}{3.7}  & \BoundVal{63.6}{4.1}\\\BoundSep
                 & \BoundDist{10}   & \BoundVal{57.3}{4.5}   & \BoundVal{65.0}{3.1}   & \BoundVal{73.2}{3.8}   & \BoundValB{76.4}{3.8}  & \BoundVal{63.4}{4.3}\\\BoundSep
                 & \BoundDist{15}   & \BoundVal{0.0}{0.0}    & \BoundVal{0.0}{0.0}    & \BoundVal{61.5}{4.8}   & \BoundVal{63.1}{2.4}   & \BoundValB{63.2}{4.4}\\\BoundSep
                 & \BoundDist{20}   & \BoundVal{0.0}{0.0}    & \BoundVal{0.0}{0.0}    & \BoundVal{42.5}{4.4}   & \BoundVal{38.7}{4.2}   & \BoundValB{63.0}{4.4}\\\DatasetDivSep
\DivHeader{151}  & \BoundDist{1}    & \BoundValB{87.4}{2.9}  & \BoundVal{87.2}{2.2}   & \BoundValB{87.4}{2.9}  & \BoundVal{86.7}{2.6}   & \BoundVal{64.0}{4.3}\\\BoundSep
                 & \BoundDist{25}   & \BoundVal{69.1}{4.3}   & \BoundVal{70.2}{5.5}   & \BoundVal{69.1}{4.3}   & \BoundValB{75.7}{4.9}  & \BoundVal{63.0}{4.4}\\\BoundSep
                 & \BoundDist{50}   & \BoundVal{22.8}{5.8}   & \BoundVal{24.9}{4.0}   & \BoundVal{35.4}{6.3}   & \BoundVal{52.8}{4.0}   & \BoundValB{62.0}{4.8}\\\BoundSep
                 & \BoundDist{75}   & \BoundVal{0.0}{0.0}    & \BoundVal{0.0}{0.0}    & \BoundVal{14.8}{3.0}   & \BoundVal{23.0}{3.9}   & \BoundValB{61.8}{4.7}\\\BoundSep
                 & \BoundDist{100}  & \BoundVal{0.0}{0.0}    & \BoundVal{0.0}{0.0}    & \BoundVal{3.4}{2.2}    & \BoundVal{5.2}{2.4}    & \BoundValB{61.3}{4.2}\\\DatasetDivSep
\DivHeader{301}  & \BoundDist{1}    & \BoundVal{83.1}{2.8}   & \BoundValB{86.2}{3.3}  & \BoundVal{83.1}{2.8}   & \BoundVal{86.0}{3.2}   & \BoundVal{64.0}{4.3}\\\BoundSep
                 & \BoundDist{45}   & \BoundVal{65.1}{4.7}   & \BoundVal{68.6}{3.9}   & \BoundVal{65.1}{4.7}   & \BoundValB{72.1}{3.9}  & \BoundVal{62.3}{4.3}\\\BoundSep
                 & \BoundDist{90}   & \BoundVal{30.4}{3.5}   & \BoundVal{34.6}{4.9}   & \BoundVal{33.6}{3.2}   & \BoundVal{53.7}{2.7}   & \BoundValB{61.7}{4.6}\\\BoundSep
                 & \BoundDist{135}  & \BoundVal{0.5}{0.7}    & \BoundVal{0.1}{0.3}    & \BoundVal{23.7}{3.5}   & \BoundVal{31.1}{4.6}   & \BoundValB{60.2}{4.2}\\\BoundSep
                 & \BoundDist{180}  & \BoundVal{0.0}{0.0}    & \BoundVal{0.0}{0.0}    & \BoundVal{7.2}{2.5}    & \BoundVal{11.3}{2.5}   & \BoundValB{58.3}{4.3}\\